\documentclass{article}

\usepackage{arxiv}
\usepackage[dvipsnames]{xcolor}
\usepackage[utf8]{inputenc} 
\usepackage[T1]{fontenc}    
\usepackage{hyperref}       
\usepackage{url}            
\usepackage{booktabs}       
\usepackage{amsfonts}       
\usepackage{nicefrac}       
\usepackage{microtype}      
\usepackage{fancyhdr}       
\usepackage{graphicx}       

\usepackage{natbib}

\usepackage{microtype}
\usepackage{graphicx,fancybox}
\usepackage{subfigure}
\usepackage{booktabs} 
\usepackage{algorithm,algorithmic}
\usepackage{amssymb, multirow, paralist, color,amsmath,amsthm}
\newtheorem{thm}{Theorem}

\newtheorem{lemma}{Lemma}

\newtheorem{ass}{Assumption}
\usepackage{enumitem}
\setlist[itemize]{align=parleft,left=1em}
\usepackage{textcase}
\usepackage[T1]{fontenc}
\usepackage{booktabs} 
\usepackage{makecell}
\usepackage{comment}
\usepackage{hyperref}

\def \S {\mathbf{S}}

\def \O {\mathcal{O}}

\def \R {\mathbb{R}}

\def \w {\mathbf{w}}

\def \x {\mathbf{x}}

\def \E {\mathrm{E}}

\def \x {\mathbf{x}}

\def \d {\mathbf{d}}

\def \1 {\mathbf{1}}

\def \I {\mathbb{I}}

\def \Q {\mathcal{Q}}

\def \B {\mathcalB}
\def \C {\mathbf C}

\def \m {\mathbf{m}}

\def \E {\mathrm{E}}
\def \x {\mathbf{x}}

\def \w {\mathbf{w}}

\def \R {\mathbb{R}}
\def \S {\mathcal{S}}

\def \Q {\mathcal{Q}}

\def \d {\mathbf{d}}

\def \I {\mathbb{I}}

\def \B {\mathcal{B}}

\def \C {\mathcal{C}}

\def \E {\mathbb{E}}

\title{Large-scale Stochastic Optimization of\\ NDCG Surrogates for Deep Learning\\ with Provable Convergence 
}

\author{%
    Zi-Hao Qiu$^1$\thanks{Contribute Equally. Correspondence to tianbao-yang@uiowa.edu}, Quanqi Hu$^2$\footnotemark[1], Yongjian Zhong$^3$, Lijun Zhang$^1$, Tianbao Yang$^3$\\
    $^1$ National Key Laboratory for Novel Software Technology, Nanjing University, Nanjing 210023, China\\
    $^2$ Department of Mathematics, the University of Iowa, Iowa City, IA 52242, USA\\
    $^3$ Department of Computer Science, the University of Iowa, Iowa City, IA 52242, USA\\
    \texttt{qiuzh@lamda.nju.edu.cn},\texttt{\{quanqi-hu,yongjian-zhong\}@uiowa.edu},\\\texttt{zlj@nju.edu.cn},\texttt{tianbao-yang@uiowa.edu}
}

\begin{document}

\abovedisplayskip=2pt
\abovedisplayshortskip=0pt
\belowdisplayskip=2pt
\belowdisplayshortskip=0pt

\maketitle

\begin{abstract}

NDCG, namely Normalized Discounted Cumulative Gain, is a widely used ranking metric in information retrieval and machine learning. However, efficient and provable stochastic methods for maximizing NDCG are still lacking, especially for deep models. In this paper, we propose a principled approach to optimize NDCG and its top-$K$ variant. First, we formulate a novel compositional optimization problem for optimizing the NDCG surrogate, and a novel bilevel compositional optimization problem for optimizing the top-$K$ NDCG surrogate. Then, we develop efficient stochastic algorithms with provable convergence guarantees for the non-convex objectives. Different from existing NDCG optimization methods, the per-iteration complexity of our algorithms scales with the mini-batch size instead of the number of total items. To improve the effectiveness for deep learning, we further propose practical strategies by using initial warm-up and stop gradient operator. Experimental results on multiple datasets demonstrate that our methods outperform prior ranking approaches in terms of NDCG. To the best of our knowledge, this is the first time that stochastic algorithms are proposed to optimize NDCG with a provable convergence guarantee. Our proposed methods are implemented in the LibAUC library at \url{https://libauc.org/}.
\end{abstract}

\section{Introduction}
\label{sec:introduction}

NDCG is a performance metric of primary interest for learning to rank in information retrieval~\citep{liu2011learning}, and is also adopted in many machine learning tasks where ranking is of foremost importance~\citep{liu2008eigenrank,bhatia2015sparse}. In the following, we use the terminologies from information retrieval to describe NDCG and our methods. The goal is to rank the relevant items higher than irrelevant items for any given query. For a query $q$ and a list of $n$ items, the ranking model assigns a score for each item, and then we obtain an ordered list by sorting these scores in descending order. The NDCG score for $q$ can be computed by:
\begin{align}\label{eqn:ndcg_def}
    \text{NDCG}_q = \frac{1}{Z_q}\sum_{i=1}^n\frac{2^{y_i}-1}{\log_2(1+\text{r}(i))},
\end{align}
where $y_i$ denotes the relevance score of the $i$-th item, 
$\text{r}(i)$ denotes the rank of the $i$-th item in the ordered list, and $Z_q$ is a normalization factor that is the Discounted Cumulative Gain (DCG) score~\citep{jarvelin2002cumulated} of the optimal ranking for $q$. 
The top-$K$ variant of NDCG can be defined similarly by summing over items whose ranks are in the top $K$ positions of the ordered list. In many real-world applications, e.g., recommender systems, we want to recommend a small set of $K$ items from a large collection of items~\citep{cremonesi2010performance}, thus top-$K$ NDCG is a popular metric in these applications.

There are several challenges for optimizing NDCG and its top-$K$ variant. 
First, computing the rank of each item among all $n$ items is expensive. Second, the rank operator is non-differentiable in terms of model parameters. 
To tackle non-differentiability, surrogate functions have been proposed in the literature for approximating NDCG and its top-$K$ variant~\citep{SoftRank,ApproxNDCG,PiRank,NeuralNDCG}. 
However, to the best of our knowledge, \emph{the computational challenge} of computing the gradient of~(\ref{eqn:ndcg_def}) that involves sorting $n$ items has never been addressed. All existing gradient-based methods have a complexity of $O(nd)$ per-iteration, where $d$ is the number of model parameters, which is prohibitive for deep learning tasks with big $n$ and big $d$. A naive approach is to update the model parameters by the gradient of the NDCG surrogate over a mini-batch of samples, however, since the surrogate for NDCG is complicated and non-convex, an unbiased stochastic gradient is not readily computed, which makes existing methods lack theoretical guarantee.




In this paper, we propose the first stochastic algorithms with a per-iteration complexity of $O(Bd)$, where $B$ is the mini-batch size, for optimizing the surrogates for NDCG and its top-$K$ variant, and establishing their convergence guarantees. For optimizing the NDCG surrogate, we first formulate a novel \emph{finite-sum coupled compositional optimization (FCCO)} problem. Then, we develop an efficient stochastic algorithm inspired by a recent work on average precision maximization~\citep{DBLP:journals/corr/abs-2104-08736}. We establish an iteration complexity of $O(\frac{1}{\epsilon^4})$ for finding an $\epsilon$-level stationary solution, which is better than that proved by~\citet{DBLP:journals/corr/abs-2104-08736}, i.e., $O(\frac{1}{\epsilon^5})$. 
To tackle the challenge of optimizing the top-$K$ NDCG surrogate that involves a selection operator, we propose a novel \emph{bilevel optimization} problem, which contains many lower level problems for top-$K$ selection of all queries. Then we smooth the non-smooth functions in the selection operator, and propose an efficient algorithm with the iteration complexity of $O(\frac{1}{\epsilon^4})$. The algorithm is based on recent advances of stochastic bilevel optimization~\citep{guo2021randomized}, but with unique features to tackle the compositional upper level problem and a mini-batch of randomly sampled lower level problems per iteration for optimizing the top-$K$ NDCG surrogate.


To improve the effectiveness of optimizing the NDCG surrogates, we also study two practical strategies. First, we propose initial warm-up to find a good initial solution. 
Second, we use stop gradient operator to simplify the optimization of the top-$K$ NDCG surrogate. We conduct comprehensive experiments on two tasks, learning to rank and recommender systems. Empirical results demonstrate that the proposed algorithms can consistently outperform prior approaches in terms of NDCG, and show the effectiveness of two proposed strategies. 

We summarize our contributions below: 
\begin{itemize}
\vspace{-3.5mm}
    \setlength\itemsep{-0.0em}
    \item We formulate the optimization of the NDCG surrogate as a finite-sum coupled compositional optimization problem, and propose a novel stochastic algorithm with provable convergence guarantees.
    
    \item We propose a novel bilevel compositional optimization formulation for optimizing the top-$K$ NDCG surrogate. Then we develop a novel stochastic algorithm and establish its convergence rate.
    
    \item To improve the effectiveness for deep learning, we also study practical strategies by using initial warm-up and stop gradient operator. Experimental results on multiple datasets demonstrate the effectiveness of our algorithms and strategies.
\end{itemize}

\section{Related Work}
\label{sec:related-work}

{\bf Listwise LTR approaches.} Learning to rank (LTR) is an extensively studied area~\citep{liu2011learning}, and we only review the listwise LTR approaches that are closely related to this work. The listwise methods can be classified into three groups. The first group uses ranking metrics to dynamically re-weight instances during training. For example, LambdaRank algorithms~\citep{LambdaRank,burges2010ranknet} define a weight $\Delta$NDCG, which is the NDCG difference when a pair of items is swapped in the current list, and use it to re-weight the pair during training. Although algorithms in this group take NDCG into account, the underlying losses of them remain unknown and their theoretical relations to NDCG are difficult to analyze. The second group defines loss functions over the entire item lists to optimize the agreement between predictions and ground truth rankings. For example, ListNet~\citep{ListNet} minimizes cross-entropy between predicted and ground truth top-one probability distributions. ListMLE~\citep{ListMLE} aims to maximize the likelihood of the ground truth list given the predicted results. However, optimizing these loss functions might not necessarily maximize NDCG. 
In addition, efficient stochastic algorithms for optimizing these losses are still lacking. The third group directly optimizes ranking metrics, and most of works focus on the widely used NDCG, as reviewed below.

{\bf NDCG Optimization.}  Some earlier works employ traditional optimization techniques, e.g., genetic algorithm~\citep{RankGP}, boosting~\citep{AdaRank,NIPS2009_b3967a0e}, and SVM framework~\citep{SVM-NDCG}. However, these methods are not scalable to big data. A popular class of approaches is to approximate ranks in NDCG with smooth functions and then optimize the resulting surrogates. For example, SoftRank~\citep{SoftRank} tries to use rank distributions to smooth NDCG, however, it suffers from a high computational complexity of $O(n^3)$. ApproxNDCG~\citep{ApproxNDCG} approximates the indicator function in the computation of ranks, and the top-$K$ selector in the computation of top-$K$ variant by a generalized sigmoid function. Recently, PiRank~\citep{PiRank} and NeuralNDCG~\citep{NeuralNDCG} are proposed to smooth NDCG by approximating non-continuous sorting operator based on NeuralSort~\citep{NeuralSort}. However, these methods mainly focus on how to approximate NDCG with differentiable functions, and remain computationally expensive as their per-iteration complexity is $O(nd)$. 
Moreover, little attention has been paid to the convergence guarantee for the stochastic optimization of these surrogates. In contrast, this is the first work to develop stochastic algorithms with provable convergence guarantee for optimizing the surrogates for NDCG and its top-$K$ variant.

{\bf Stochastic Compositional Optimization.} Optimization of a two-level compositional function in the form of $\mathbb{E}_{\xi}[f(\mathbb{E}_{\zeta}[g(\textbf{w};\zeta)];\xi)]$ where $\xi$ and $\zeta$ are independent random variables, or its finite-sum variant has been studied extensively~\citep{wang2017stochastic,balasubramanian2020stochastic,chen2021solving}. In this paper, we formulate the surrogate function of NDCG into a similar but more complicated two-level compositional function of the form $\mathbb{E}_{\xi}[f(\mathbb{E}_{\zeta}[g(\textbf{w};\zeta,\xi))]$ where $\xi$ and $\zeta$ are independent and $\xi$ has a finite support inspired by~\citep{DBLP:journals/corr/abs-2104-08736}. The key difference between our compositional function and the ones considered in previous work is that the inner function $g(\textbf{w};\zeta,\xi)$ also depends on the random variable $\xi$ of the outer level. Our algorithm is developed based on that of~\citet{DBLP:journals/corr/abs-2104-08736} for average precision maximization, but establishes an improved complexity of $O(\frac{1}{\epsilon^4})$ for finding an $\epsilon$-stationary solution. It is also notable that our algorithm and convergence result for optimizing NDCG is similar to that in a concurrent work~\cite{FCCO} dedicated to FCCO. However, the key difference is that our convergence analysis for optimizing NDCD follows that for optimizing top-$K$ NDCG in a novel bi-level optimization framework.

{\bf Stochastic Bilevel Optimization.} Stochastic bilevel optimization (SBO) has a long history in the literature~\citep{colson2007overview,kunisch2013bilevel,liu2020generic}. Recent works on SBO focus on algorithms with provable convergence rates~\citep{ghadimi2018approximation,ji2020provably,hong2020two,chen2021single}. However, most of these studies do not explicitly consider the challenge for dealing with SBO with many lower level problems. \citet{guo2021randomized} consider SBO with many lower level problems and develop a stochastic algorithm with convergence guarantee. However, their algorithm is not applicable to our problem for optimizing the compositional top-$K$ NDCG surrogate and a mini-batch of randomly sampled lower level problems in each iteration, and is not practical as it requires evaluating the stochastic gradients twice per-iteration at two different points. In this paper, we propose a novel stochastic algorithm for optimizing the top-$K$ NDCG surrogate, which contains many lower level problems, and establish its iteration complexity of $O(\frac{1}{\epsilon^4})$.

\section{Preliminaries}
\label{sec:preliminaries}

In this section, we provide some preliminaries and notations. Let $\Q$ denote the query set of size $N$, and $q\in\Q$ denote a query. $\S_q$ denotes a set of $N_q$ items (e.g., documents, movies) to be ranked for $q$. For each $\x^q_i\in\S_q$, let $y^q_i\in\R^+$ denote its relevance score, which measures the relevance between query $q$ and item $x^q_i$. Let $\S^+_q\subseteq\S_q$ denote a set of $N^+_q$ items \emph{relevant} to $q$, whose relevance scores are \emph{non-zero}. Denoted by $\S=\{(q, \x^q_i),q\in\Q, \x^q_i\in\S^+_q\}$ all relevant query-item (Q-I) pairs. 
Let $h_q(\x; \w)$ denote the predictive function for $\x$ with respect to the query $q$, whose parameters are denoted by $\w\in\R^d$ (e.g., a deep neural network). Let $\I(\cdot)$ denote the indicator function, which outputs 1 if its input is true and 0 otherwise. Let
\begin{align*}
r(\w; \x, \S_q) = \sum_{\x'\in\S_q}\I(h_q(\x'; \w) - h_q(\x; \w)\geq 0)
\end{align*}
denote the rank of $\x$ with respect to the set $\S_q$, where we simply ignore the tie. 

According to the definition in~(\ref{eqn:ndcg_def}), the averaged NDCG over all queries can be expressed by  
\begin{align*} 
\text{NDCG:}\quad \frac{1}{N}\sum_{q=1}^N\frac{1}{Z_q}\sum_{\x_i^q\in S^+_q} \frac{2^{y^q_i}-1}{\log_2(r(\w; \x^q_i, \S_q)+1) },
\end{align*}
where $Z_q$ is the maximum DCG of a perfect ranking of items in $\S_q$, which can be pre-computed. 
Note that $\x^q_i$ are summed over $\S_q^+$ instead of $\S_q$, because only relevant items have non-zero relevance scores and contribute to NDCG.

An important variant of NDCG is its top-$K$ variant, which is defined over the items $\x^q_i\in\S_q$ whose prediction scores are in the top-$K$ positions, i.e.,
\begin{align*} 
\frac{1}{N}\sum_{q=1}^N\frac{1}{Z^K_q}\sum_{\x_i^q\in\S^+_q}\I(\x_i^q\in\S_q[K]) \frac{2^{y^q_i}-1}{\log_2(r(\w; \x^q_i, \S_q)+1) },
\end{align*} 
where $\S_q[K]$ denotes the top-$K$ items whose prediction scores are in the top-$K$ positions among all items in $\S_q$, and $Z^K_q$ denotes the top-$K$ DCG score of the perfect ranking. 

\section{Optimizing a Smooth NDCG Surrogate}
\label{sec:opt-ndcg}

To address the non-differentiability of the rank function $r(\w; \x, \S_q)$, we approximate it by a continuous and differentiable surrogate function
\begin{align*}
    \bar g(\w; \x, \S_q) = \sum_{\x'\in\S_q}\ell(h_q(\x'; \w) - h_q(\x; \w)),
\end{align*}
where $\ell(\cdot)$ is a surrogate loss function of $\I(\cdot\geq 0)$. In this paper, we use a convex and non-decreasing smooth surrogate loss, e.g., squared hinge loss 
$\ell(x)= \max(0,x+c)^2$, where $c$ is a margin parameter. Other choices are possible with pros and cons discussed in the literature~\citep{10.1145/1645953.1646266,ApproxNDCG}.  Below, we abuse the notation $\ell(\w; \x', \x, q)=\ell(h_q(\x'; \w) - h_q(\x; \w))$. 

Using the surrogate function, we cast NDCG maximization into: 
\begin{align} \label{eqn:NDCG}
\max_{\w\in\R^d} L(\w):=\frac{1}{|\S|}\sum_{q=1}^N\sum_{\x_i^q\in S^+_q} \frac{2^{y^q_i}-1}{Z_q\log_2(\bar g(\w; \x^q_i, \S_q)+1)}.
\end{align} 
The following lemma justifies the maximization over $L(\w)$ for NDCG maximization:
\begin{lemma}\label{lemma:1}
When $\ell(\w; \x', \x, q)\geq \I(h_q(\x'; \w)- h_q(\x; \w)\geq 0)$, then $L(\w)$ is a lower bound of NDCG. 
\end{lemma}
\vspace{-1mm}
\begin{algorithm}[t]
\caption{\underline{S}tochastic \underline{O}ptimization of \underline{N}DC\underline{G}: SONG}\label{alg:1}
\begin{algorithmic}[1]
\REQUIRE $\eta,\gamma_0,\beta_1, u^{(1)}=0$
\ENSURE $\w_T$
\FOR{$t=1,...T$}
\STATE Draw some relevant Q-I pairs $\B=\{(q, \x^q_i)\}\subset\S$ 
\STATE For each sampled $q$ draw a batch of items  $\B_q\subset\S_q$
\FOR{each sampled Q-I pair $(q,\x^q_i)\in \B$}
\STATE Let $\hat g_{q, i}(\w_t) = \frac{1}{|\B_q|}\sum_{\x'\in\B_q} \ell(\w_t; \x',\x^q_i, q)$
\STATE Compute $u^{(t+1)}_{q, i}=(1-\gamma_0)u^{(t)}_{q, i}+\gamma_0\hat g_{q,i}(\w_t)$
\STATE Compute $p_{q, i} = \nabla f_{q,i}(u^{(t)}_{q,i})$
\ENDFOR
\STATE Compute the stochastic gradient estimator $G(\w_t)$ by
\vspace{-1mm}
$$G(\w_t) = \frac{1}{|\B|}\sum_{(q,\x^q_i)\in\B}p_{q, i}\nabla\hat g_{q,i}(\w_t)$$
\vspace{-1mm}
\STATE Compute $\m_{t+1} = \beta_1\m_t + (1-\beta_1) G(\w_t)$
\STATE update $\w_{t+1}=\w_t - \eta \m_{t+1}$
\ENDFOR
\end{algorithmic}
\end{algorithm}

The key challenge in designing an efficient algorithm for solving the above problem lies at (i) computing $\bar g(\w; \x^q_i, \S_q)$ and its gradient is expensive when $N_q=|\S_q|$ is very large; and (ii) an unbiased stochastic gradient of the objective function is not readily available. To highlight the second challenge, let us consider the gradient of the function
$\phi(\w)=\frac{1}{\log_2(\bar g(\w; \x^q_i, \S_q)+1)}$, which is given by 
\begin{align*}
\nabla\phi(\w)=\frac{-\log_2(e)\cdot\nabla \bar g(\w; \x^q_i, \S_q)}{\log^2_2(\bar g(\w; \x^q_i, \S_q)+1)\cdot (\bar g(\w; \x^q_i, \S_q)+1)}.\notag
\end{align*}
We can estimate $\bar g(\w; \x_q^i, \S_q)$ by its unbiased estimator using a mini-batch of $B_q$ items $\x'\in\B_q\subset\S_q$, i.e., $\frac{N_q}{B_q}\sum_{\x'\in\B_q}\ell(h_q(\x'; \w) - h_q(\x_q^i; \w))$. 
However, directly plug this unbiased estimator of $\bar g(\w; \x_q^i, \S_q)$ into the above expression will produce a biased estimator of $\nabla\phi(\w)$ due to the non-linear function of $\bar g$. The optimization error will be large if the mini-batch size $B_q$ is small~\citep{hu2020biased}.

To address this challenge, we cast the problem into the following equivalent minimization form: 
\begin{align}\label{eqn:NDCG} 
\min_{\w\in\R^d} F(\w):=\frac{1}{|\S|}\sum_{(q,\x_i^q)\in\S} f_{q,i}(g(\w; \x^q_i, \S_q)),
\end{align} 
where $g(\w; \x^q_i, \S_q) = \frac{1}{N_q}\bar g(\w; \x^q_i, \S_q)$ and $f_{q,i}(g) = \frac{1}{Z_q}\frac{1-2^{y^q_i}}{\log_2(N_q g+ 1)}$. It is a special case of a family of {\bf finite-sum coupled compositional stochastic optimization} problems, which was first studied by~\citet{DBLP:journals/corr/abs-2104-08736} for maximizing average precision.
Inspired by their method, we develop a stochastic algorithm for solving~(\ref{eqn:NDCG}). 
The complete procedure is provided in Algorithm~\ref{alg:1}, which is named as \underline{S}tochastic \underline{O}ptimization of \underline{N}DC\underline{G} (SONG). 

To motivate the proposed method, we first derive the gradient of $F(\w)$ by the chain rule, which is given by \newline
\vspace*{-0.15in}\begin{align*}
    \nabla F(\w)=\frac{1}{|\S|}\sum_{(q,\x_i^q)\in\S}\nabla f_{q,i}(g(\w;\x_i^q,\S_q))\nabla g(\w;\x_i^q,\S_q).
\end{align*}
The major cost for computing $\nabla F(\w)$ lies at computing $g(\w; \x^q_i, \S_q)$ and its gradient, which involves all items in $\S_q$. To this end, we approximate these quantities by stochastic samples. The gradient $\nabla g(\w; \x^q_i, \S_q)$ can be simply approximated by the stochastic gradient $\nabla \hat g_{q,i}(\w_t)=\frac{1}{|\B_q|}\sum_{\x'\in\B_q}\nabla \ell(\w_t; \x',\x^q_i, q)$, where $\B_q$ is sampled from $\S_q$. 
Note that $\nabla f_{q,i}(g(\w; \x^q_i, \S_q))$ is non-linear with $g(\w; \x^q_i, \S_q)$, thus we need a better way to estimate $g(\w; \x^q_i, \S_q)$ to control the approximation error and provide convergence guarantee. We borrow a technique from~\citet{DBLP:journals/corr/abs-2104-08736} by \emph{using a moving average estimator to keep track of} $g(\w_t; \x^q_i, \S_q)$ for each $\x^q_i\in\S_q^+$. To this end, we maintain a scalar $u_{q,i}$ for each \emph{relevant} query-item pair $(q, \x_i^q)$ and update it by a linear combination of historical one $u^{(t)}_{q,i}$ and an unbiased estimator of $g(\w_t; \x^q_i, \S_q)$ denoted by $\hat g_{q,i}(\w_t)$ in Step 5 and 6, where $\gamma_0\in(0,1)$ is a parameter. Intuitively, when $t$ increases, $\w_{t-1}$ is getting closer to $\w_t$, hence the previous value of the estimator, i.e., $u^{(t)}_{q,i}$ is useful for estimating $g_{q,i}(\w_t)$. With these stochastic estimators, we can compute the gradient of the objective in~(\ref{eqn:NDCG}) with controllable approximation error in Step 9. We implement the momentum update for $\w_{t+1}$ in Step 10 and 11, where $\beta_1\in(0,1)$ is the momentum parameter. The momentum update can be also replaced by the Adam-style update~\citep{guo2022stochastic}, where the step size $\eta$ is replaced by an adaptive step size. We can establish the same convergence rate for the Adam-style update.

We also have several remarks about SONG: (i) the total per-iteration complexity of SONG is $O(Bd+B^2)$. The details can be found in Appendix~\ref{appendix:per-iteration-complexity}. For a large model size $d\gg B$, we have the per-iteration complexity of $O(Bd)$, which is similar to the standard cost of deep learning and is independent of the length of $\S_q$ for each query; and (ii) the additional memory cost is the size of $u_{q,i}$, i.e., the number of all relevant Q-I pairs. It is worth to mention that in many real-world datasets the number of relevant Q-I pairs are much fewer than all Q-I pairs (i.e., data  is sparse)~\citep{yuan2014recommendation,yin2020overcoming,singh2020scalability}. Thus the additional memory cost is acceptable in most cases.

Next, we establish the convergence guarantee of SONG in the following theorem.
\vspace{-0mm}
\begin{thm}\label{thm:2}
Under appropriate conditions and proper settings of parameters $\gamma_0, \gamma_1, \eta=O(\epsilon^2)$, $\beta_1=1-\gamma_1$, Algorithm~\ref{alg:1} ensures that after $T=O(\frac{1}{\epsilon^4})$ iterations we can find an $\epsilon$-stationary solution of $F(\w)$, i.e., $\E[\|\nabla F(\w_\tau)\|^2]\leq \epsilon^2$ for a randomly selected $\tau\in\{1,\ldots,T\}$.
\vspace{-0mm}
\end{thm}
{\bf Remark:} The above theorem indicates that SONG has the same $O(\frac{1}{\epsilon^4})$ iteration complexity as the standard SGD for solving standard non-convex losses~\citep{ghadimi2013stochastic}. We refer the interested readers to Appendix~\ref{appendix:convergence-analysis} for the proof, where we also exhibit the settings for $\gamma_0, \beta_1, \eta$ and the conditions. The conditions are imposed mainly for ensuring $f_{q,i}(g)$ and $g(\w; \x^q_i, \S_q)$ are smooth and Lipchitz continuous.  It is worth mentioning that the above complexity is better than that proved by~\citet{DBLP:journals/corr/abs-2104-08736}, i.e., $O(1/\epsilon^5)$. In addition, we do not have any requirement on the batch size, i.e., $|\B|, |\B_q|$, which can be as small as 1. However, we can enjoy parallel speed-up for a large batch size.

\section{Optimizing a Smooth Top-$K$ NDCG Surrogate}
\label{sec:opt-ndcg-at-k}

In this section, we propose an efficient stochastic algorithm to optimize the top-$K$ variant of NDCG. By using the smooth surrogate loss $\ell(\cdot)$ for approximating the rank function, we have the following objective for top-$K$ NDCG:
\begin{align*} 
\frac{1}{N}\sum_{q=1}^N\frac{1}{Z^K_q}\sum_{\x_i^q\in\S^+_q}\I(\x_i^q\in\S_q[K]) \frac{2^{y^q_i}-1}{\log_2(\bar g(\w; \x^q_i, \S_q)+1) },
\end{align*} 
where $\S_q[K]$ denotes the set of top-$K$ items in $\S_q$ whose prediction scores are in the top-$K$ positions. Compared with optimizing the NDCG surrogate in~(\ref{eqn:NDCG}), there is another level of complexity, i.e., the selection of top-$K$ items from $\S_q$, which is non-differentiable. In the literature, \citet{ApproxNDCG} and~\citet{10.1145/1645953.1646266} use the relationship  $\I(\x_i^q\in\S_q[K])=\I(K-r(\w; \x_i^q, \S_q )\geq 0)$ and  approximate it by $\psi(K-\bar g(\w; \x_i^q, \S_q))$, where $\psi$ is a continuous surrogate of the indicator function. However, there are two levels of approximation error, one lies at approximating $r(\w; \x_i^q, \S_q )$ by $\bar g(\w; \x_i^q, \S_q)$ and the other one lies at approximating $\I(\cdot\geq 0)$ by $\psi(\cdot)$. 
To reduce the error for selecting $\x_i^q\in\S_q[K]$, we propose a more effective method, which relies on the following lemma:
\begin{lemma}\label{lemma:2}
Let $\lambda_q(\w) = \arg\min_{\lambda}(K+\varepsilon)\lambda +\sum_{\x'\in\S_q}(h_q(\x'; \w) -\lambda)_+$, where $\varepsilon\in(0,1)$, then $\lambda_q(\w)$ is the $(K+1)$-th largest value among $h_q(\x', \w), \forall\x'\in\S_q$, and hence $\x^q_i\in\S_q[K]$ is equivalent to $h_q(\x^q_i; \w)> \lambda_q(\w)$. 
\end{lemma}
\vspace{-0mm}
{\bf Remark:} We can show that the optimal solution $\lambda_q(\w)$ can be served as the threshold for selecting top-$K$ items in $\S_q$.

As a result, the problem can be converted  into 
\begin{align*} 
&\min \frac{1}{|\S|}\sum_{q=1}^N\sum_{\x_i^q\in\S^+_q}\frac{\I(h_q(\x^q_i; \w)- \lambda_q(\w)>0)(1-2^{y^q_i})}{Z_q^K\log_2(g(\w; \x^q_i, \S_q)+1) }\\
& s.t., \lambda_q(\w) = \arg\min_{\lambda}\frac{K+\varepsilon}{N_q}\lambda +\frac{1}{N_q}\sum_{\x'\in\S_q}(h_q(\x'; \w) -\lambda)_+.
\end{align*}

However, there are still several challenges that prevent us developing a provable algorithm. In particular, the selection operator $\I(h_q(\x^q_i; \w)- \lambda_q(\w)>0)$ is a non-smooth function of $\w$ due to (i) the indicator function $\I(\cdot)$ is non-continuous and non-differentiable; and (ii) $\lambda_q(\w)$ is a non-smooth function of $\w$ because the lower optimization problem is non-smooth and non-strongly convex. 

To address the above challenges, we first approximate $\I(\cdot>0)$ by a smooth and Lipschtiz continuous function $\psi(\cdot)$. The choice of $\psi$ can be justified by the following lemma: 
\vspace{-0mm}
\begin{lemma}\label{lemma:3}
If $\psi(h_q(\x^q_i; \w)- \lambda_q(\w))\leq C\I(h_q(\x^q_i; \w)- \lambda_q(\w)>0)$ holds for some constant $C>0$ and  $\ell(\w; \x', \x, q)\geq \I(h_q(\x'; \w)- h_q(\x; \w)> 0)$, then the function $ \frac{1}{N}\sum_{q=1}^N\sum_{\x_i^q\in S^+_q}\frac{\psi(h_q(\x^q_i; \w)- \lambda_q(\w))(2^{y^q_i}-1)}{CZ_q^K\log_2(\bar{g}(\w; \x^q_i, \S_q)+1) }$ is a lower bound of the  top-$K$ NDCG.
\vspace{-0mm}
\end{lemma}
{\bf Remark:} When $h_q(\x; \w)$ is bounded, it is not hard to find a smooth and Lipschtiz continuous function $\psi(\cdot)$ satisfying the above condition. A simple choice for $\psi(\cdot)$ is sigmoid function.

Next, we smooth $\lambda(\w)$. The idea is to make the objective function in the lower level problem smooth and strongly convex, while not affecting the optimal solution $\lambda(\w)$ too much. To this end, we replace the lower level problem by 
\begin{align*}
    &\hat\lambda_q(\w)  = \arg\min_{\lambda}L_q(\lambda; \w) : = \frac{K+\varepsilon}{N_q}\lambda+ \frac{\tau_2}{2}\lambda^2 +\frac{1}{N_q}\sum_{\x_i\in\S_q}\tau_1\ln(1+\exp((h_q(\x_i;\w)-\lambda)/\tau_1)).
\end{align*}
The following lemma justifies the above smoothing.
\vspace{-0mm}
\begin{lemma}\label{lemma:4}
Assuming $h_q(\x,\w)\in(0,c_h]$ , if $\tau_1=\tau_2=\varepsilon$ for some $\varepsilon\ll 1$ , then we have $|\hat\lambda_q(\w) - \lambda_q(\w)|\leq O(\varepsilon)$ for any $\w$. In addition,  $L_q(\lambda; \w)$ is a smooth and strongly convex function in terms of $\lambda$ for any $\w$. 
\vspace{-0mm}
\end{lemma}
As a result, we propose to solve the following optimization problem for top-$K$ NDCG maximization:
\begin{align} \label{eqn:KNDCG}
&\min \frac{1}{|\S|}\sum_{(q, \x^q_i)\in\S}\psi(h_q(\x^q_i; \w)- \hat\lambda_q(\w))f_{q,i}(g(\w; \x^q_i, \S_q))\notag \\
& s.t., \hat\lambda_q(\w) = \arg\min_{\lambda} L_q(\lambda; \w), \forall q\in\Q,
\end{align}
where we employ $f_{q,i}(g)$ to denote $\frac{1}{Z^K_q}\frac{1-2^{y^q_i}}{\log_2(N_q g+ 1)}$. 

Our bilevel formulation is more advantageous than previous NDCG@$K$ formulation. First, our formulation only approximates $r(\w; \x_i^q, \S_q )$ by $\bar g(\w; \x_i^q, \S_q)$ \textbf{once} in the denominator, while previous one approximates $r(\w; \x_i^q, \S_q )$ \textbf{twice} (one in $\psi(K-r(\w; \x_i^q, \S_q ))$ and one in the denominator). In addition, $\psi(h_q(\x_i^q;\w)-\lambda_q(\w))$ is arguably better than $\psi(K-\bar g(\w; \x_i^q, \S_q))$ for approximating $\I(K-r(\w; \x_i^q, \S_q )\geq 0)$ due to Lemma~\ref{lemma:2}.

Although~(\ref{eqn:KNDCG}) is a bilevel optimization problem, existing stochastic algorithms for bilevel optimization are not applicable to solving the above problem. That is because there are several differences from the standard bilevel optimization problem studied in the literature. First, an unbiased stochastic gradient of the objective function is not readily computed as we explained before. Second, there are multiple lower level problems in~(\ref{eqn:KNDCG}), whose solutions cannot be updated at the same time for all $q\in\Q$ when $N$ is large. To address these challenges, we develop a tailored stochastic algorithm for solving~(\ref{eqn:KNDCG}).

The proposed algorithm is presented in Algorithm~\ref{alg:2}, to which we refer as K-SONG. To motivate K-SONG, we first consider the gradient of the objective function denoted by $F_K(\w)$, which can be computed as
\begin{align}
    \nabla F_K(\w)&=\frac{1}{|\S|}\sum_{(q,\x^q_i)\in\S} \bigg(\psi'(h_q(\x^q_i; \w)- \hat\lambda_q(\w))  
    \cdot \notag (\nabla h_q(\x^q_i; \w) -\nabla_\w\hat\lambda_q(\w)) \bigg)f_{q,i}(g(\w; \x^q_i, \S_q))\label{eqn:nabla_F}\\
    &+\psi(h_q(\x^q_i; \w)-\hat\lambda_q(\w))\nabla g(\w;\x^q_i, \S_q)f_{q,i}'(g(\w; \x^q_i, \S_q)\notag.
\end{align}
Similar to SONG, we can estimate $g(\w_t; \x^q_i, \S_q)$ by $u^{(t)}_{q,i}$. An inherent challenge of bilevel optimization is to estimate the implicit gradient $\nabla_\w\hat\lambda(\w)$. According to the optimality condition of $\hat\lambda(\w)$~\citep{ghadimi2018approximation}, we can derive 
\[
\nabla_\w\hat\lambda_q(\w)=-\nabla_{\lambda,\w}^2 L_q(\hat\lambda_q(\w); \w) (\nabla^2_{\lambda}L_{q}( \hat\lambda_q(\w); \w))^{-1}.
\]
To estimate $\nabla_{\lambda,\w}^2 L_q(\hat\lambda(\w); \w)$ at the $t$-th iteration, we use the current estimate $\lambda_{q,t}$ in place of $\hat \lambda_q(\w_t)$ and use $L_q(\hat\lambda, \w; \B_q)$ that is defined by a mini-batch samples of $\B_q$ in place of $L_q(\hat\lambda; \w)$, i.e., 
\begin{align}
 & L_q(\lambda, \w; \B_q) = \frac{K}{N_q}\lambda+ \frac{\tau_2}{2}\lambda^2+\frac{1}{|\B_q|}\sum_{\x_i\in\B_q}\tau_1\ln(1+\exp((h_q(\x_i;\w)-\lambda)/\tau_1))\notag.
\end{align}
The issue of estimating $(\nabla^2_{\lambda}L_{q}( \hat\lambda_q(\w); \w))^{-1}$ is more tricky. In the literature~\citep{ghadimi2018approximation}, a common method is to use von Neuman series with stochastic samples to estimate it. However, such method requires multiple samples in the order of $O(1/\tau_2)$, which is a large number when $\tau_2$ is small. To address this issue, we follow a similar strategy of~\citet{guo2021randomized} to estimate $\nabla^2_{\lambda}L_{q}( \hat\lambda_q(\w); \w)$ directly by using mini-batch samples. In the proposed algorithm, we use a moving average estimator denoted by $s_{q}$ as shown in Step 10. Finally, we have the following stochastic gradient estimator:
\begin{equation}\label{eqn:update_grad}
    \begin{aligned}
    G(\w_t)&=\frac{1}{|\B|}\sum_{(q,\x_i^q)\in\B}p_{q,i}\nabla\hat g_{q,i}(\w_t) \\
   &+\psi'
    (h_q(\x^q_i; \w_t)- \lambda_{q,t})\bigg[\nabla_\w h_q(\x^q_i; \w_t)+\nabla_{\lambda,\w}^2 L_q(\w_t,\lambda_i^t;\B_t) s^{-1}_{q,t}\bigg]f(u^{(t)}_{q,i})
    \end{aligned}
\end{equation}
where $p_{q,i}$ is computed in Step 7 in  K-SONG. 

\begin{algorithm}[t]
\caption{Stochastic Optimization of top-$K$ NDCG: K-SONG}\label{alg:2}
\begin{algorithmic}[1]
\REQUIRE $\eta_0,\eta_1,\gamma_0,\gamma_0',\beta_1,  u^{(1)}=0, \lambda=0$
\ENSURE $\w_T$
\FOR{$t=1,...T$}
\STATE Draw some relevant Q-I pairs $\B=\{(q, \x^q_i)\}\subset\S$ 
\STATE For each sampled $q$ draw a batch of items  $\B_q\subset\S_q$
\FOR{each sampled Q-I pair $(q,\x^q_i)\in \B$}
\STATE Let $\hat g_{q, i}(\w_t) = \frac{1}{|\B_q|}\sum_{\x'\in\B_q} \ell(\w_t; \x',\x^q_i, q)$
\STATE Let $u^{(t+1)}_{q, i}=(1-\gamma_0)u^{(t)}_{q, i}+\gamma_0\hat g_{q,i}(\w_t)$
\STATE Let $p_{q, i} = \psi(h_q(\x^q_i; \w_t)- \lambda_{q,t})\nabla f_{q,i}(u^{t}_{q,i})$
\ENDFOR
\FOR{each sampled query $q\in \B$}
\STATE Let $s_{q,t+1} = (1-\gamma_0')s_{q,t} + \gamma_0'\nabla^2_{\lambda} L_q(\lambda_{q,t}; \w_t; \B_q)$
\STATE Let $\lambda_{q,t+1} = \lambda_{q, t} - \eta_0 \nabla_\lambda L_q(\lambda_{q,t}; \w_t; \B_q)$
\ENDFOR
\STATE Compute a stochastic gradient $G(\w_t)$ according to~(\ref{eqn:update_grad}) or ~~(\ref{eqn:update_grad2})
\STATE Compute $\m_{t+1} = \beta_1\m_t + (1-\beta_1) G(\w_t)$
\STATE Update $\w_{t+1}=\w_t - \eta_1 \m_{t+1}$
\ENDFOR
\end{algorithmic}
\end{algorithm}

We follow a similar strategy as~\citet{guo2022stochastic} to update $\lambda_{q,t+1}$ by a simple stochastic gradient update, shown in Step 11. 
It is notable that different from~\citet{guo2021randomized}, we update $\lambda_{q,t+1}$ for a mini-batch of randomly sampled queries $q$, which makes the analysis more challenging. 

Finally, we present the convergence guarantee of K-SONG. 
\begin{thm}\label{thm:3}
Under appropriate conditions and proper settings of parameters $\gamma_0,\gamma_0',\eta_0 = \O(|\B_q|\epsilon^2)$, $\gamma_1 = \O(\min\{|\B|,|\B_q|\}\epsilon^2)$, $\beta_1=1-\gamma_1$, $\eta_1 = \O\left(\min\left \{\frac{|\B||\B_q|\epsilon^2}{|\S|},\min\{|\B|,|\B_q|\}\epsilon^2\right\}\right)$, Algorithm~\ref{alg:2} ensures that after $T=\O\left( \max\left\{\frac{|\S|}{|\B||\B_q|\epsilon^4},\frac{1}{\min\{|\B|,|\B_q|\}\epsilon^4} \right\}\right)$ iterations we can find an $\epsilon$-stationary solution of $F_K(\w)$, i.e., $\E[\|\nabla F_K(\w_\tau)\|^2]\leq \epsilon^2$ for a randomly selected $\tau\in\{1,\ldots,T\}$.
\end{thm}
{\bf Remark:} The above theorem indicates that K-SONG also has the iteration complexity of $O(\frac{1}{\epsilon^4})$ in terms of $\epsilon$. We refer the interested readers to Appendix~\ref{appendix:convergence-analysis} for details.

\section{Practical Strategies}
\label{sec:practical-strategies}

In this section, we present two practical strategies for improving the effectiveness of SONG/K-SONG. 

{\bf Initial Warm-up.} A potential problem of optimizing NDCG is that it may not lead to a good local minimum if a bad initial solution is given.
To address this issue, we use warm-up to find a good initial solution by solving a well-behaved objective. 
Similar strategies have been used in the literature~\citep{yuan2020robust,DBLP:journals/corr/abs-2104-08736}, however, their objectives are not suitable for ranking. Here we choose the listwise cross-entropy loss~\citep{ListNet}, i.e.,
\newline\vspace*{-0.15in}
\begin{align*} 
&\min_{\w}\quad \frac{1}{N}\sum_{q=1}^N\frac{1}{N_q}\sum_{\x_i^q\in \S^+_q} -\ln\left(\frac{\exp(h_q(\x^q_i; \w)}{ \sum_{\x_j^q\in\S_q} h_q(\x^q_j; \w))}\right),
\end{align*} 
which is the cross-entropy between predicted and ground truth top-one probability distributions. The objective can be formulated as a similar finite-sum coupled compositional problem as NDCG, and  a similar algorithm to SONG can be used to solve it. We present the formulation and detailed algorithm in Appendix~\ref{appendix:initial-warm-up}.

{\bf Stop Gradient for the top-$K$ Selector.}
Given a good initial solution, we justify that the second term in~(\ref{eqn:update_grad}) is close to 0 under a reasonable condition, and present the details in Appendix~\ref{appendix:stop-gradient-operator}. Thus, the gradient of the top-$K$ selector $\psi(h(\x^q_i,\w)- \hat\lambda_q(\w))$ is not essential. We can apply the stop gradient operator on the top-$K$ selector, and compute the gradient estimator by 
\begin{equation}\label{eqn:update_grad2}
    \begin{aligned}
    &G(\w_t)=\frac{1}{|\B|}\sum_{(q,\x^q_i)\in\B}p_{q,i}\nabla\hat g_{q,i}(\w_t),
    \end{aligned}
\end{equation}
which \textbf{simplifies} K-SONG by avoiding maintaining and updating $s_{q,t}$. We refer to the K-SONG using the gradient in~(\ref{eqn:update_grad}) as theoretical K-SONG, and the K-SONG using the gradient in~(\ref{eqn:update_grad2}) as practical K-SONG.

\section{Experiments}
\label{sec:experiments}

In this section, we evaluate our algorithms through comprehensive experiments on two different domains: learning to rank and recommender systems. Experimental results show that our algorithms can outperform prior ranking methods in terms of NDCG. We also conduct experiments to demonstrate the convergence speed of training and verify our algorithmic designs, including the moving average estimator and the bilevel formulation for K-SONG. In addition, we examine the effectiveness of initial warm-up and stop gradient operator. We implement our proposed methods in the LibAUC\footnote{https://libauc.org/} library. To show the advantages of our library, we compare SONG and K-SONG in LibAUC with several listwise ranking approaches implemented in TensorFlow Ranking\footnote{https://www.tensorflow.org/ranking} library.  The code to reproduce the results in this paper is available at \url{https://github.com/zhqiu/NDCG-Optimization}.

We compare our algorithms, SONG and K-SONG, against the following methods that optimize different loss functions. \textbf{RankNet}~\citep{RankNet} is a commonly used pairwise loss.
\textbf{ListNet}~\citep{ListNet} and \textbf{ListMLE}~\citep{ListMLE} are two listwise losses that optimize the agreement between predictions and ground truth rankings. \textbf{LambdaRank}~\citep{LambdaRank} is a listwise loss that takes NDCG into account, but not directly optimizes NDCG. \textbf{ApproxNDCG}~\citep{ApproxNDCG} and \textbf{NeuralNDCG}~\citep{NeuralNDCG} are two losses that optimize the NDCG surrogates directly. Similar to NeuralNDCG, PiRank~\citep{PiRank} also employs NeuralSort~\citep{NeuralSort} to approximate NDCG, so we do not compare with it. We do not compare with SoftRank~\citep{SoftRank}, as its $O(n^3)$ complexity is prohibitive. 

For all methods, we sample a batch of queries, and a few (e.g., 10) relevant items and some irrelevant items for each query per iteration. For K-SONG, we report its \emph{theoretical} version results unless specified otherwise. We use the Adam-style update for all methods and set the momentum parameters to their default values~\citep{Adam}. The hyper-parameters of all losses are fine-tuned using grid search with training/validation splits mentioned below. Due to the limited space, we present the detailed implementation and datasets information in Appendix~\ref{appendix:implemention} and ~\ref{appendix:data}, respectively. To further show the effectiveness of our methods, we conduct more experiments on multi-label classification and provide the results in Appendix~\ref{sec:multi-label-cls}.

\subsection{Learning to Rank}

{\bf Data.} Learning to rank (LTR) algorithms aim to rank a set of candidate items for a given search query. 
We consider two datasets: MSLR-WEB30K~\cite{mslr} and Yahoo! LTR dataset~\cite{yahoo}, which are the largest public LTR datasets from commercial search engines. Both datasets contain query-document pairs represented by real-valued feature vectors, and have associated relevance scores on the scale from 0 to 4. Following~\citet{ai2019learning}, we use the training/validation/test sets in the Fold1 of MSLR-WEB30K dataset for evaluation. The Yahoo! LTR dataset splits the queries arbitrarily and uses 19,944 for training, 2,994 for validation and 6,983 for testing. 

{\bf Setup.} For the backbone network, we adopt the Context-Aware Ranker~\cite{ltr-model}, a ranking model based on the Transformer. 
For all methods, we first pre-train a model by initial warm-up. Then we re-initialize the last layer and train the model by different methods as mentioned before. In both stages, we set the initial learning rate and batch size to 0.001 and 64, respectively. We train the networks for 100 epochs, decaying the learning rate by 0.1 after 50 epochs. We tune $\gamma_0$ and $K$ in our algorithms from \{0.1, 0.2, 0.3, 0.4, 0.5\} and \{10, 20, 50\}, respectively.

{\bf Results.} We evaluate all methods and calculate NDCG@$k$ ($k\in[1,3,5]$) on the test data. We provide partial results in Table~\ref{tab:part-testing-results}, and full results in Table~\ref{tab:ltr-testing-results} in Appendix~\ref{sec:additional-exp-results}. We notice that, in general, methods that directly optimize the NDCG surrogates achieve higher performance. Similar conclusions have been reached in other studies~\cite{ApproxNDCG,NeuralNDCG}. We also observe that our SONG and K-SONG can consistently outperform all baselines on both datasets. These results clearly show that our methods are effective for LTR tasks.

\begin{table*}[t]
\vspace{-4mm}
\caption{The test NDCG on four datasets. We report the average NDCG@3 for two LTR datasets, the average NDCG@20 for two RS datasets, and standard deviation over 3 runs with different random seeds. Full results are in Appendix~\ref{sec:additional-exp-results}.}
\vspace{-4mm}
\label{tab:part-testing-results}
\vskip 0.2in
\begin{center}
\begin{small}
\begin{sc}
\begin{tabular}{p{2.0cm}p{2.5cm}p{2.5cm}p{2.5cm}p{2.5cm}}
\toprule
\multirow{1}{*}{\thead{Method}} &
\multicolumn{2}{c}{\thead{NDCG@3}} &
\multicolumn{2}{c}{\thead{NDCG@20}} \\
\cmidrule(lr){2-3}
\cmidrule(lr){4-5}
& MSLE WEB30K & Yahoo! LTR & MovieLens20M & Netflix Prize  \\
\midrule
RankNet   & 0.5105$\pm$0.0004 & 0.7150$\pm$0.0004 &  0.0744$\pm$0.0013  &  0.0489$\pm$0.0003 \\
ListNet   & 0.5058$\pm$0.0001 & 0.7151$\pm$0.0004 &  0.0875$\pm$0.0004  &  0.0700$\pm$0.0002 \\
ListMLE   & 0.5074$\pm$0.0002 & 0.7146$\pm$0.0006 &  0.0799$\pm$0.0001  &  0.0508$\pm$0.0004 \\
LambdaRank & 0.5118$\pm$0.0003 & 0.7155$\pm$0.0002 &  0.0913$\pm$0.0002  &  0.0693$\pm$0.0002 \\
ApproxNDCG & 0.5114$\pm$0.0005 & 0.7152$\pm$0.0007 &  0.0938$\pm$0.0003  &  0.0592$\pm$0.0009   \\
NeuralNDCG & 0.5101$\pm$0.0005 & 0.7139$\pm$0.0001 &  0.0901$\pm$0.0003  &  0.0718$\pm$0.0003  \\
SONG  & 0.5136$\pm$0.0006 & 0.7187$\pm$0.0004 &  0.0969$\pm$0.0002  &  \textbf{0.0749}$\pm$0.0002 \\
K-SONG & \textbf{0.5147}$\pm$0.0006 & \textbf{0.7191}$\pm$0.0004&   \textbf{0.0973}$\pm$0.0003  &  0.0743$\pm$0.0003   \\
\bottomrule
\end{tabular}
\end{sc}
\end{small}
\end{center}
\end{table*}

\subsection{Recommender Systems}

{\bf Data.} Recommender systems (RS) are widely used in IT industry~\citep{lu2015recommender}. We use two large-scale movie recommendation datasets: MovieLens20M~\citep{ml-20m} and Netflix Prize dataset~\citep{netflix}. Both datasets contain large numbers of users and movies, which are represented with integer IDs. All users have rated several movies, with ratings range from 1 to 5. To create training/validation/test sets, we use the most recent rated item of each user for testing, the second recent item for validation, and the remaining items for training, which is widely-used in the literature~\citep{loo-1,loo-2}. When evaluating models, we need to collect irrelevant (unrated) items and rank them with the relevant (rated) item to compute NDCG metrics. During training, inspired by~\citet{wang2019modeling}, we randomly sample 1000 unrated items to save time. When testing, however, we adopt the all ranking protocol~\cite{wang2019neural,he2020lightgcn} --- all unrated items are used for evaluation. 

{\bf Setup.} 
We choose NeuMF~\cite{NCF} as the backbone network, which is commonly used in RS tasks. For all methods, models are first pre-trained by our initial warm-up method for 20 epochs with the learning rate 0.001 and a batch size of 256. Then the last layer is randomly re-initialized and the network is fine-tuned by different methods. At the fine-tuning stage, the initial learning rate and weight decay are set to 0.0004 and 1e-7, respectively. We train the models for 120 epochs with the learning rate multiplied by 0.25 at 60 epochs. The hyper-parameters of all methods are individually tuned for fair comparison, e.g., we tune $\gamma_0$ in SONG and K-SONG from \{0.1, 0.2, 0.3, 0.4, 0.5\}, and $K$ in K-SONG in a range \{50, 100, 300, 500\}.

{\bf Results.} We evaluate all methods and calculate NDCG@$k$ ($k\in[10, 20, 50]$) on the test data. Part of the results are reported in Table~\ref{tab:part-testing-results}, and full results are in Table~\ref{tab:rs-testing-results} in Appendix~\ref{sec:additional-exp-results}. First, SONG outperforms all baselines on both datasets. Specifically, SONG achieves 3.30\% and 4.32\% improvements on NDCG@20 over the best baseline on MovieLens20M and Netflix Prize, respectively. Besides, K-SONG performs better than SONG in most cases. These results clearly demonstrate that our algorithms are effective for optimizing NDCG and its top-$K$ variant. It is worth to mention that the improvements from our methods on RS datasets are higher than that on LTR datasets. The reason is that RS datasets have about 20,000 items per query, while most queries in LTR datasets have less than 1,000 items (detailed statistics in Appendix~\ref{appendix:data}). These results validate that our methods are more advantageous for large-scale data.

\subsection{More Studies}

{\bf Convergence Speed.} We plot the convergence curves for optimizing NDCG on two RS datasets in Figure~\ref{fig:part-training-curves}. All convergence curves for four datasets are shown in Figure~\ref{fig:training-curves} in Appendix~\ref{sec:additional-exp-results}. We can observe that our proposed SONG and K-SONG converge much faster than other methods.

\begin{figure}
\centering
\begin{minipage}[c]{0.4\textwidth}
\centering\includegraphics[width=1\textwidth]{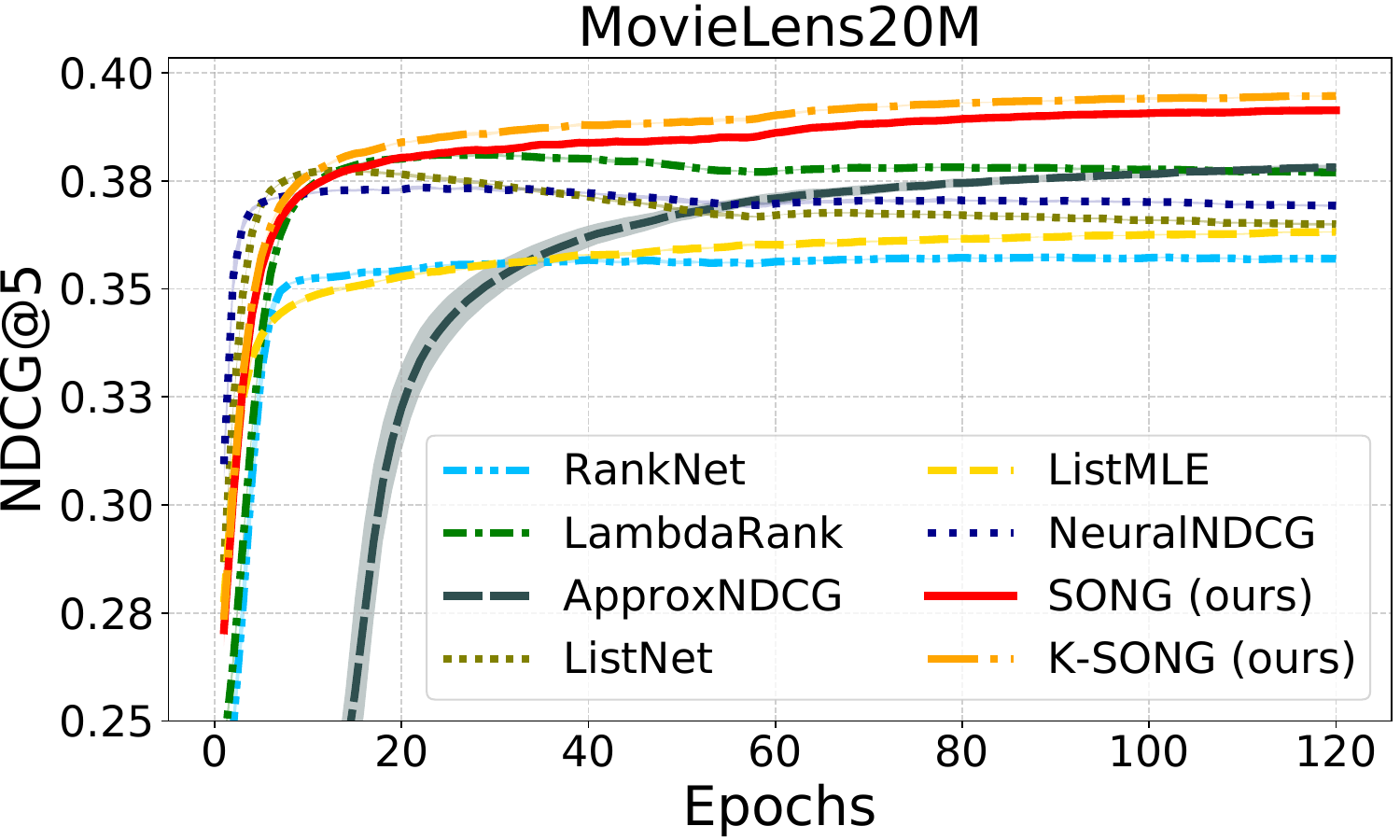}
\end{minipage}
\begin{minipage}[c]{0.4\textwidth}
\centering\includegraphics[width=1\textwidth]{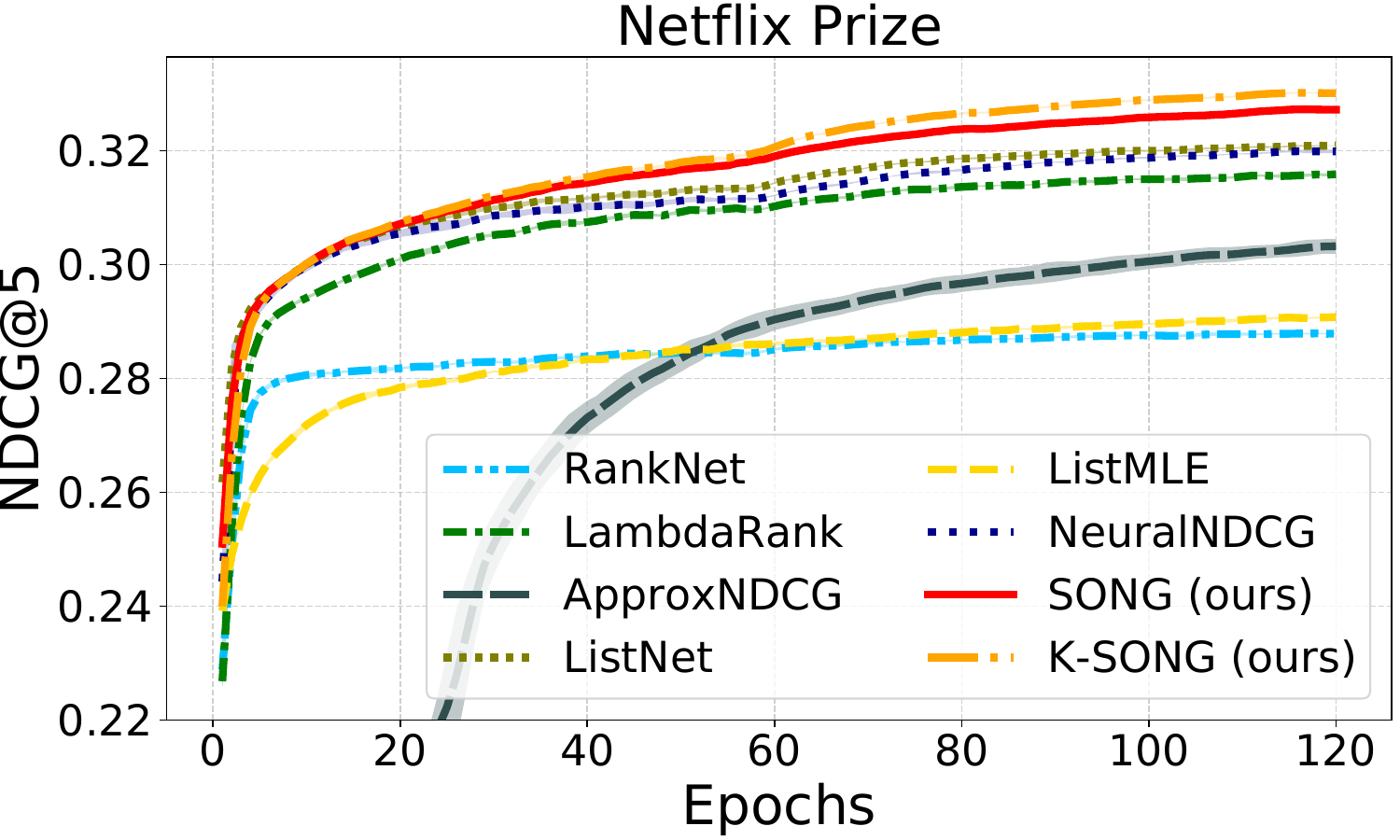}
\end{minipage}
\caption{Comparison of convergence of different methods in terms of validation NDCG@5 scores on two RS datasets.}
\label{fig:part-training-curves}
\vspace{-3mm}
\end{figure}

{\bf Ablation Studies.} We now study the effects of the moving average estimators in our methods and initial warm-up. We present the experimental results of two RS datasets in Figure~\ref{fig:part-ablation} and more results in Figure~\ref{fig:ablation} in Appendix~\ref{sec:additional-exp-results}. First, we can observe that maintaining the moving average estimators enables our algorithm perform better. To further study the effect of $\gamma_0$, we provide more results and analysis in Appendix~\ref{sec:additional-exp-results}. Second, we consistently observe that initial warm-up can bring the model to a good initialization state and improve the final performance of the model.

\begin{figure}
\centering
\begin{minipage}[c]{0.4\textwidth}
\centering\includegraphics[width=1\textwidth]{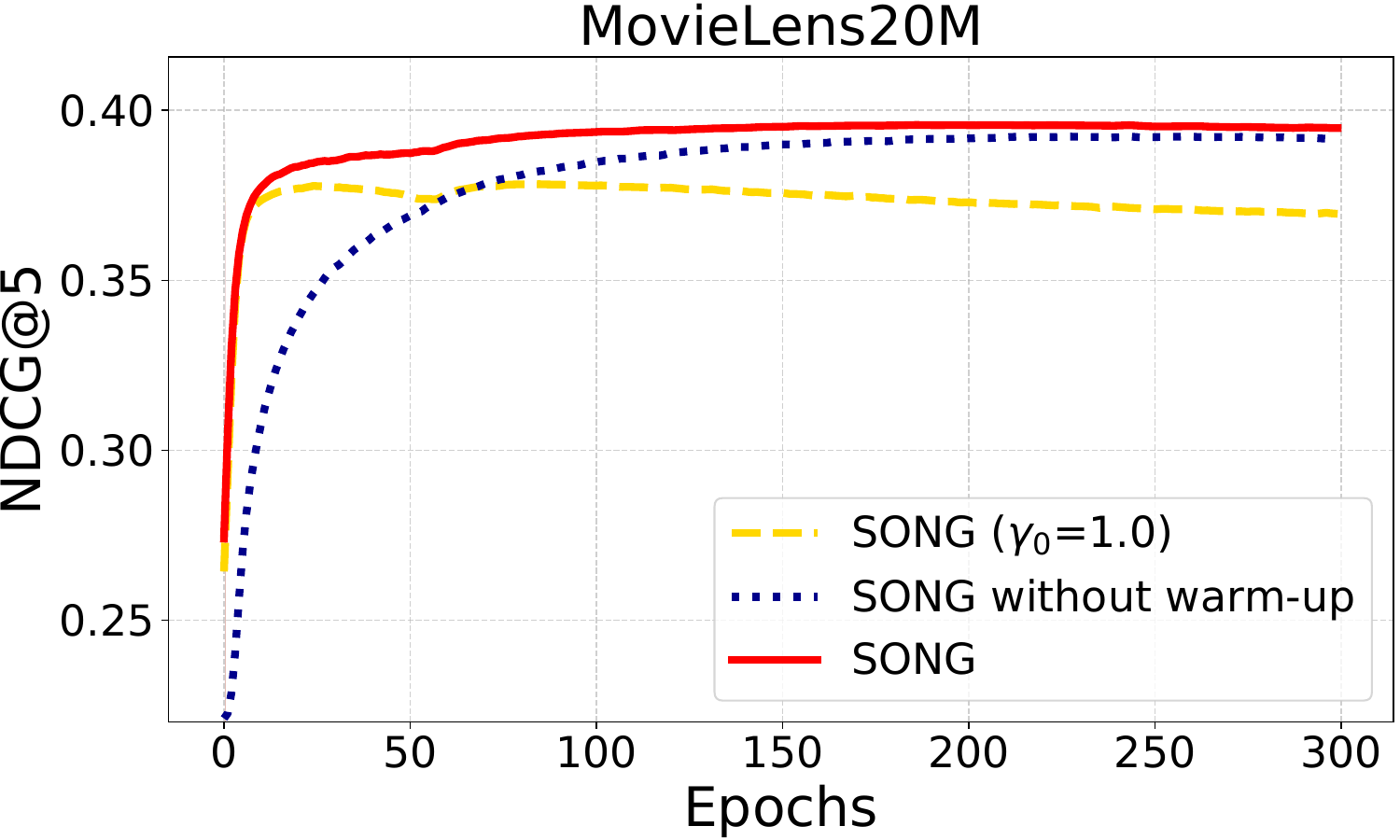}
\end{minipage}
\begin{minipage}[c]{0.4\textwidth}
\centering\includegraphics[width=1\textwidth]{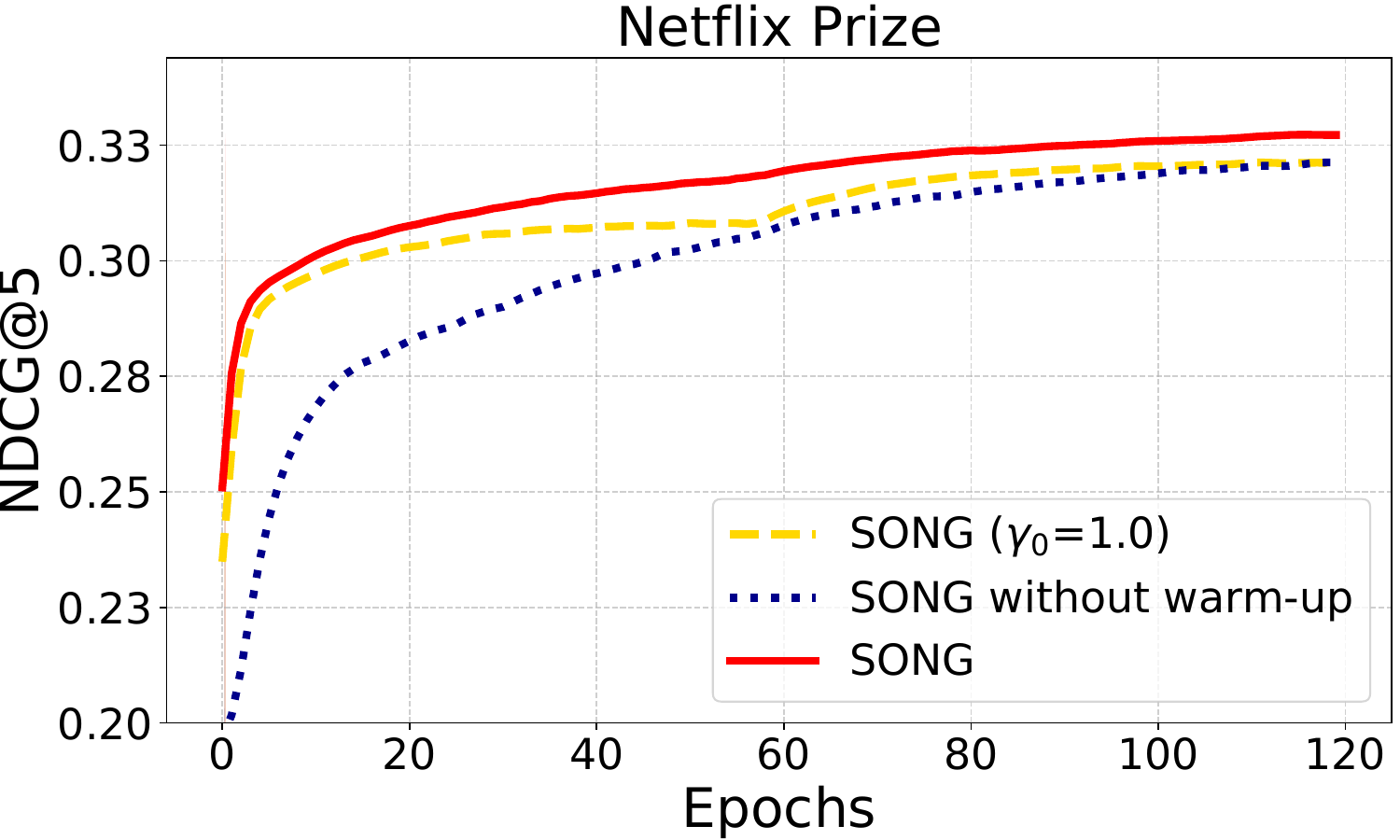}
\end{minipage}
\caption{Ablation study on two variants of SONG.}
\label{fig:part-ablation}
\end{figure}

\begin{figure}
\centering
\begin{minipage}[c]{0.4\textwidth}
\centering\includegraphics[width=1\textwidth]{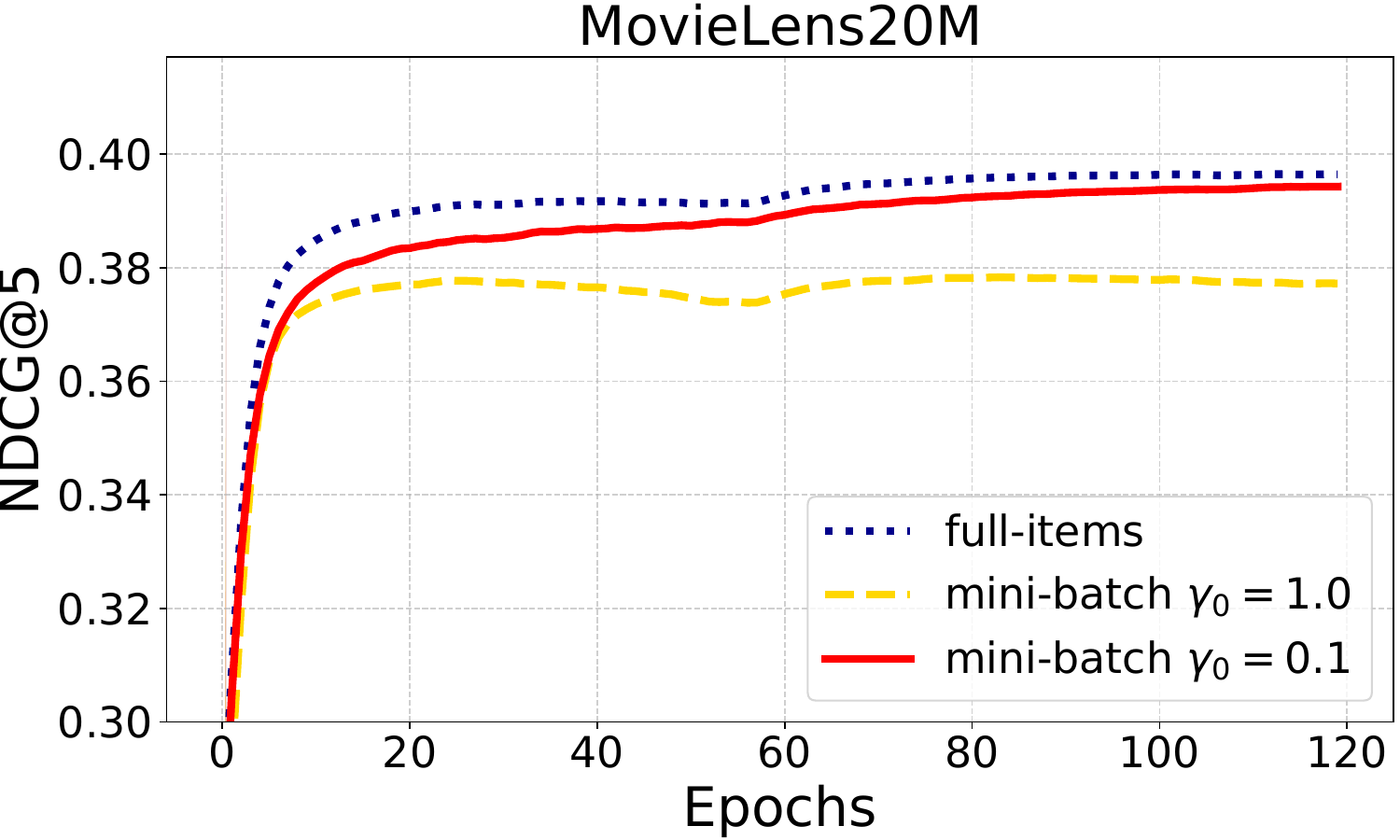}
\end{minipage}
\begin{minipage}[c]{0.4\textwidth}
\centering\includegraphics[width=1\textwidth]{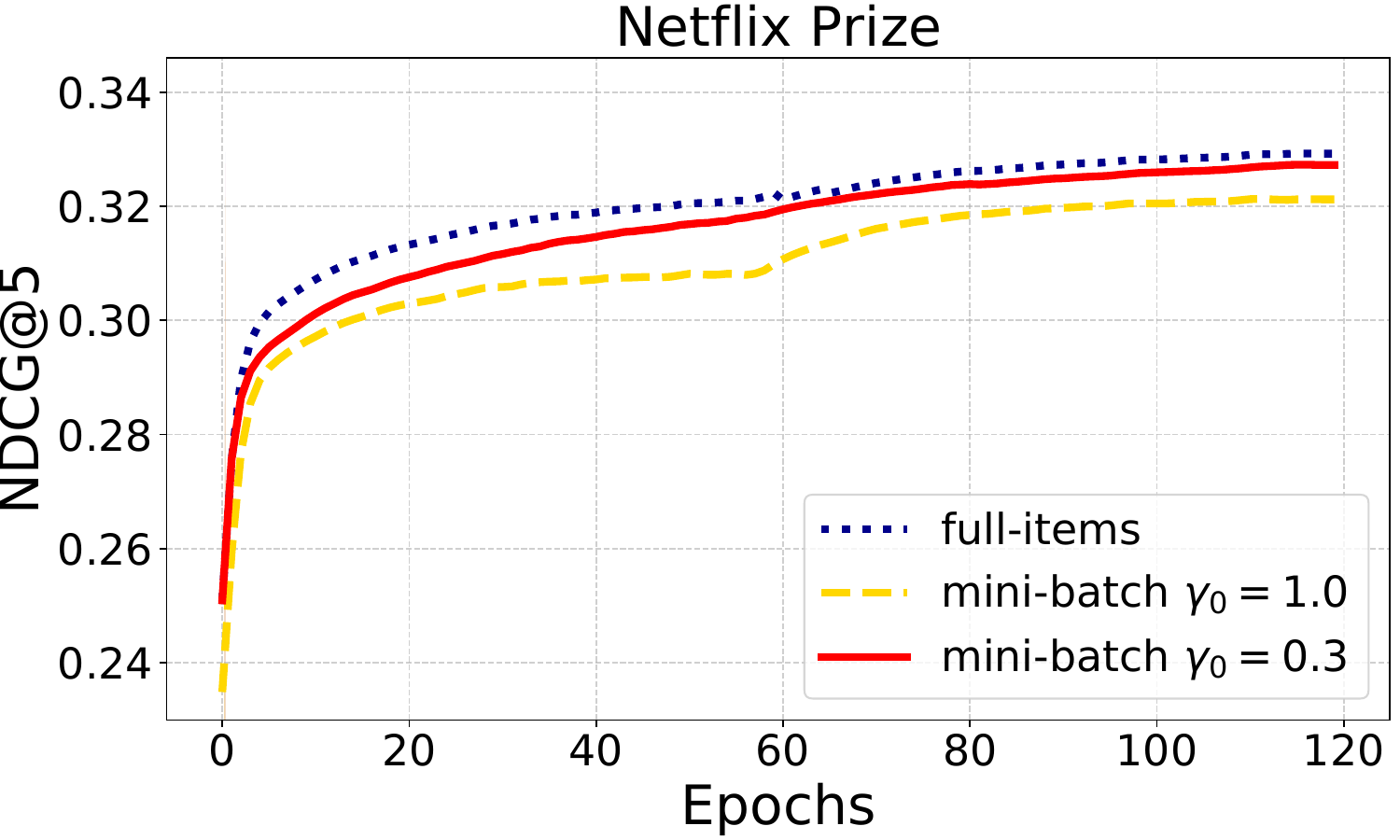}
\end{minipage}
\caption{Comparison of full-items and mini-batch training.}
\label{fig:part-full_batch_comp}
\vspace{-2mm}
\end{figure}

\begin{figure}
\centering
\begin{minipage}[c]{0.4\textwidth}
\centering\includegraphics[width=1\textwidth]{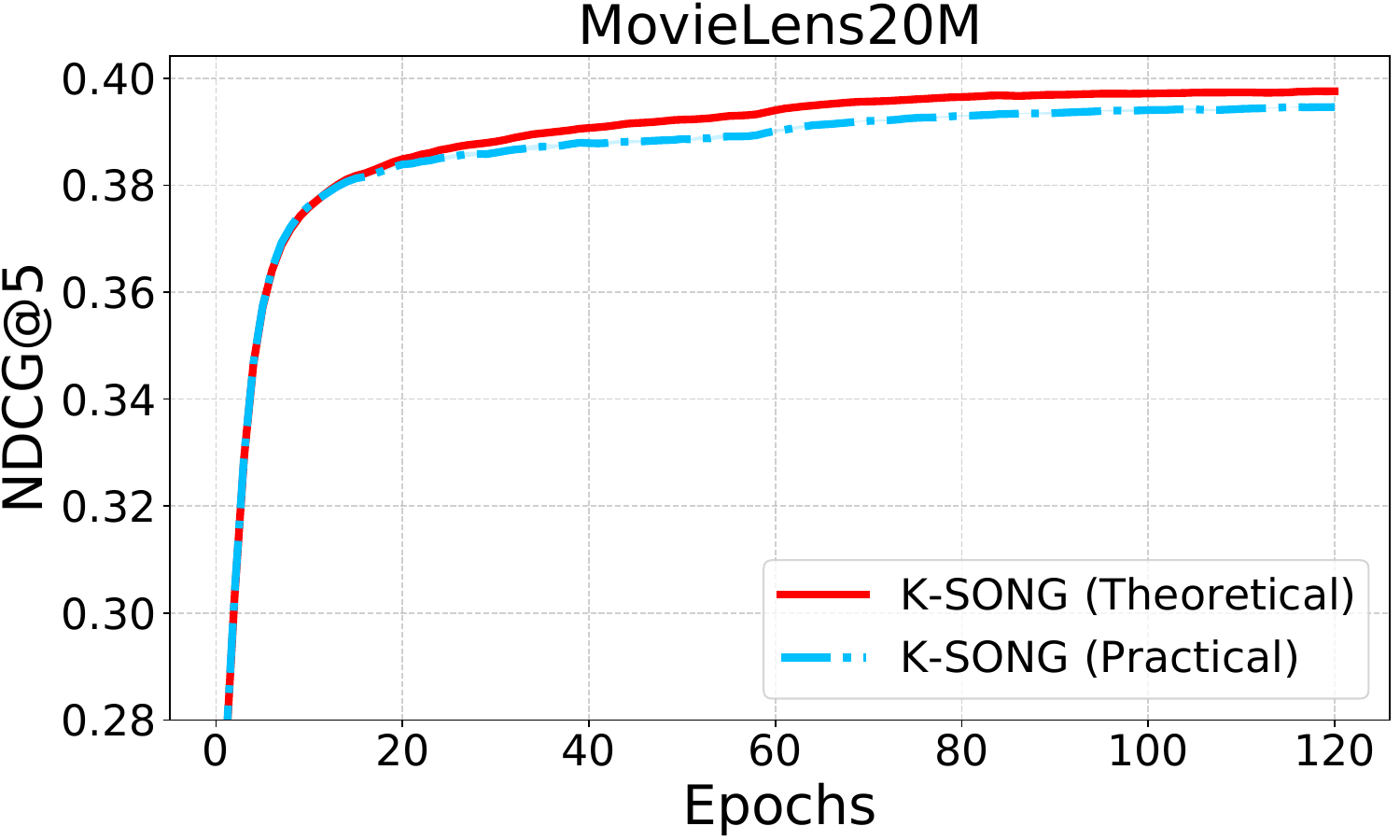}
\end{minipage}
\begin{minipage}[c]{0.4\textwidth}
\centering\includegraphics[width=1\textwidth]{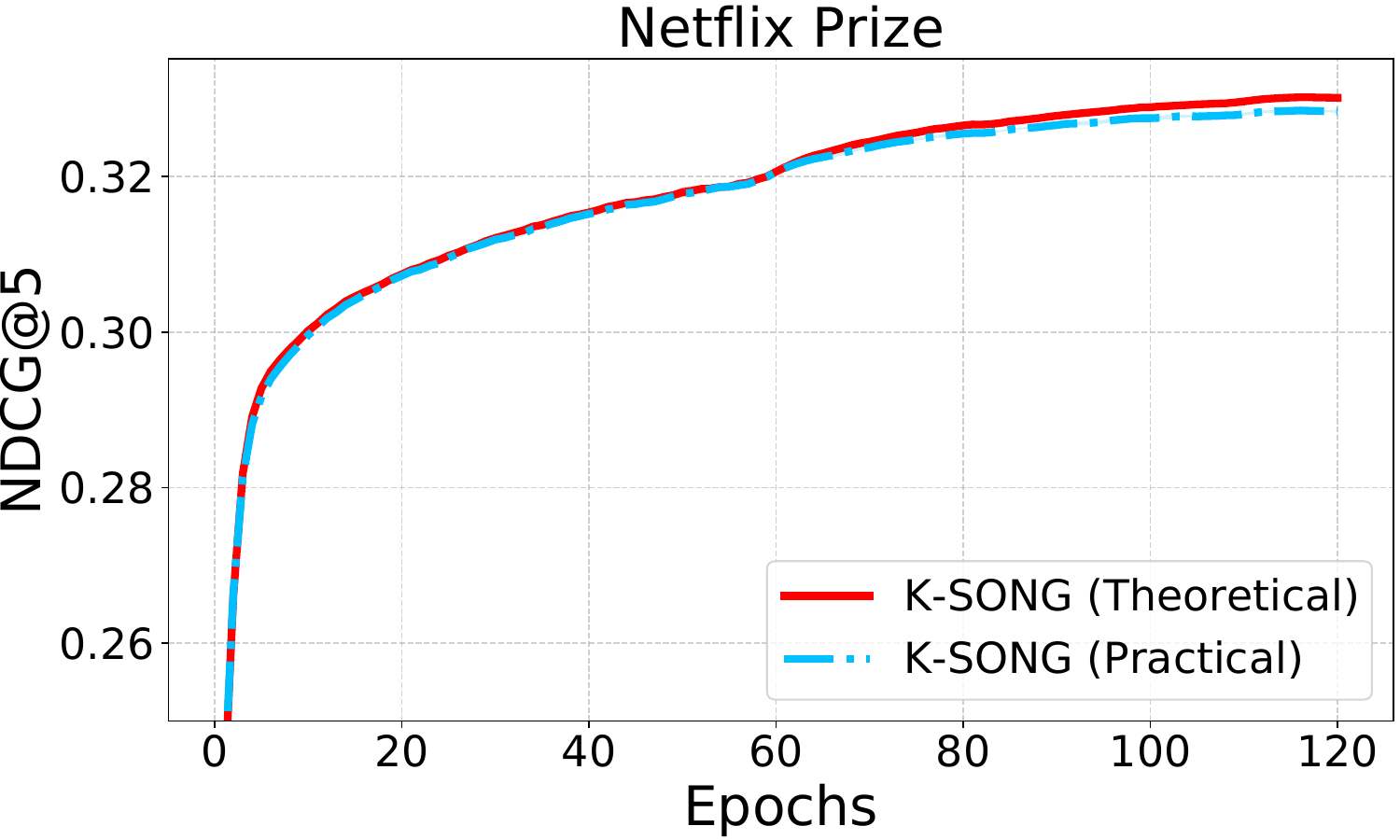}
\end{minipage}
\caption{Comparison of theoretical and practical K-SONG.}
\label{fig:K-SONG-comp}
\end{figure}

{\bf Comparison with Full-Items Training.} We compare three different training methods: full-items gradient descent that uses \emph{all items} in $\S_q$ to computing $g(\w; \x_i^q, \S_q)$ and its gradient, biased mini-batch gradient descent (i.e., set $\gamma_0=1.0$ in our algorithms), and our algorithms (i.e., with $\gamma_0$ tuned). We compare these methods for NDCG maximization and present the results in Figure~\ref{fig:part-full_batch_comp}. We can see that our methods converge to that of full-items gradient descent, which proves the effectiveness of our algorithms. We also provide the negative loglikelihood loss curves of three different training methods for warm-up in Figure~\ref{fig:full_batch_comp} in Appendix~\ref{sec:additional-exp-results}, and similar conclusions can be reached.

{\bf Theoretical and Practical K-SONG.} To verify the effectiveness of stop gradient operator, we present the comparison of theoretical K-SONG and practical K-SONG in Figure~\ref{fig:K-SONG-comp}. We observe that practical K-SONG and theoretical K-SONG achieve similar performance on both datasets, which indicates that the proposed stop gradient operator is effective in simplifying theoretical K-SONG.

{\bf The advantage of the bi-level formulation.} To demonstrate the advantage of our bi-level formulation for optimizing the top-$K$ NDCG surrogate, we implement previous NDCG@$K$ formulation by modifying our Algorithm 1 for optimizing the NDCG@$K$ objective with $\psi(K-\bar{g}(\w,\x))$ in place of $\mathbb I(K\geq r(\w;\x))$. We compare these two formulations and present the results in Figure~\ref{fig:ndcg-topk-comp}, and we can see that our bi-level formulation is more advantageous.

{\bf Comparison with TensorFlow Ranking.} We implement our SONG and K-SONG into the LibAUC library  and compare it with TensorFlow Ranking library~\citep{TensorflowRanking} (TFR), which is an open-source library for neural learning to rank implemented by Google. Specifically, we compare our implementations of SONG and K-SONG with four listwise ranking methods implemented in TFR, including ListNet, ListMLE, ApproxNDCG, and Gumbel-ApproxNDCG. For all methods, models are trained for 120 epochs on MovieLens20M with the learning rate 0.001 and a batch size of 256. For SONG and K-SONG, we first train the models by initial warm-up for the first 20 epochs, and then keep training the models by SONG or K-SONG for 100 epochs. We present the comparison of convergence and training time per epoch in Figure~\ref{fig:tfr-comp}. We notice that our implementation of SONG and K-SONG in the LibAUC library with initial warm-up converge faster than the algorithms implemented in the TFR library by Google, and the training time our methods is competitive if not better than that of TFR library, which indicates the advantages of our implementations in LibAUC. In addition, our algorithm for optimizing ListNet in LibAUC is better than that implemented in TFR due to that our algorithm has rigorous convergence guarantee and that in TFR is a mini-batch based heuristic method. 

\begin{figure}
\centering
\begin{minipage}[c]{0.4\textwidth}
\centering\includegraphics[width=1\textwidth]{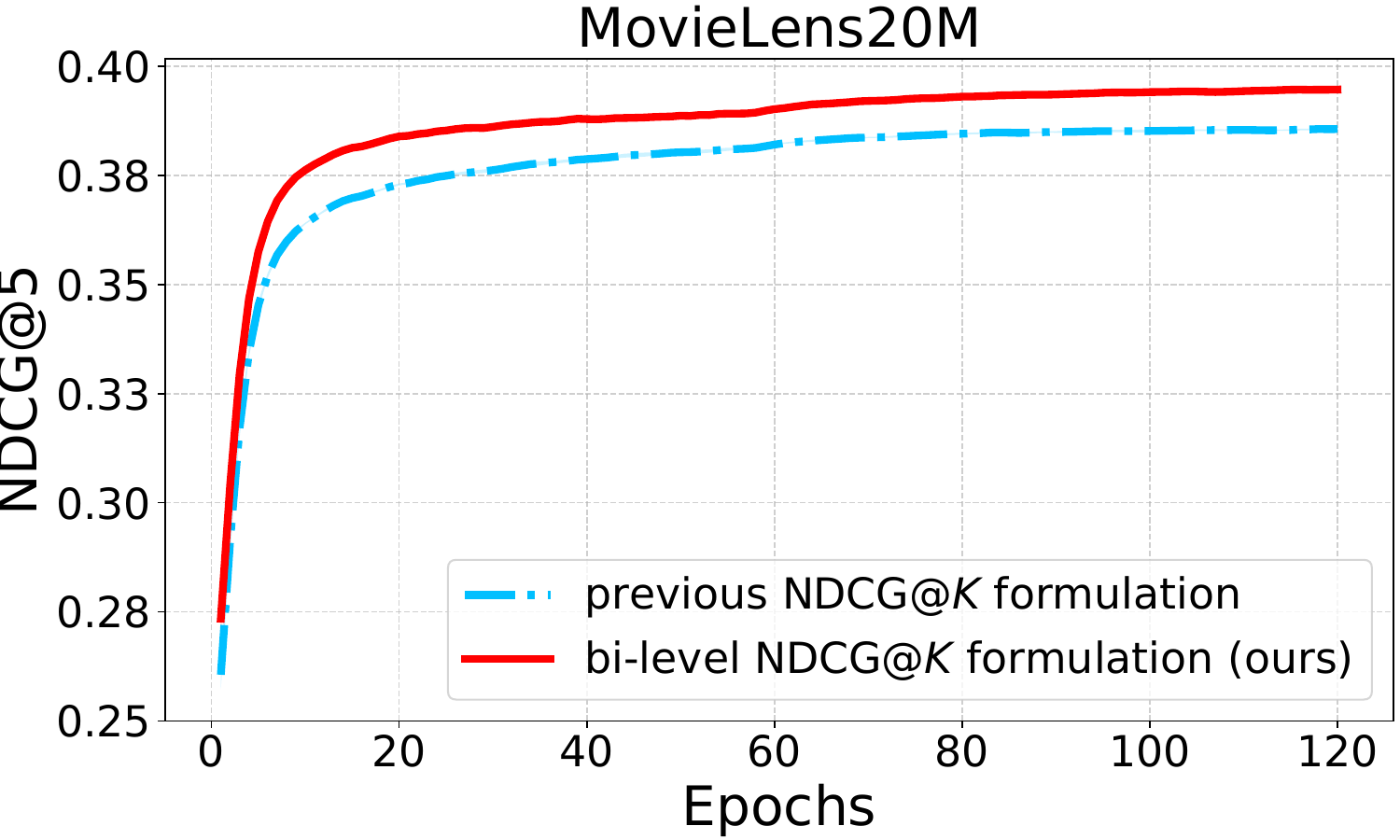}
\end{minipage}
\begin{minipage}[c]{0.4\textwidth}
\centering\includegraphics[width=1\textwidth]{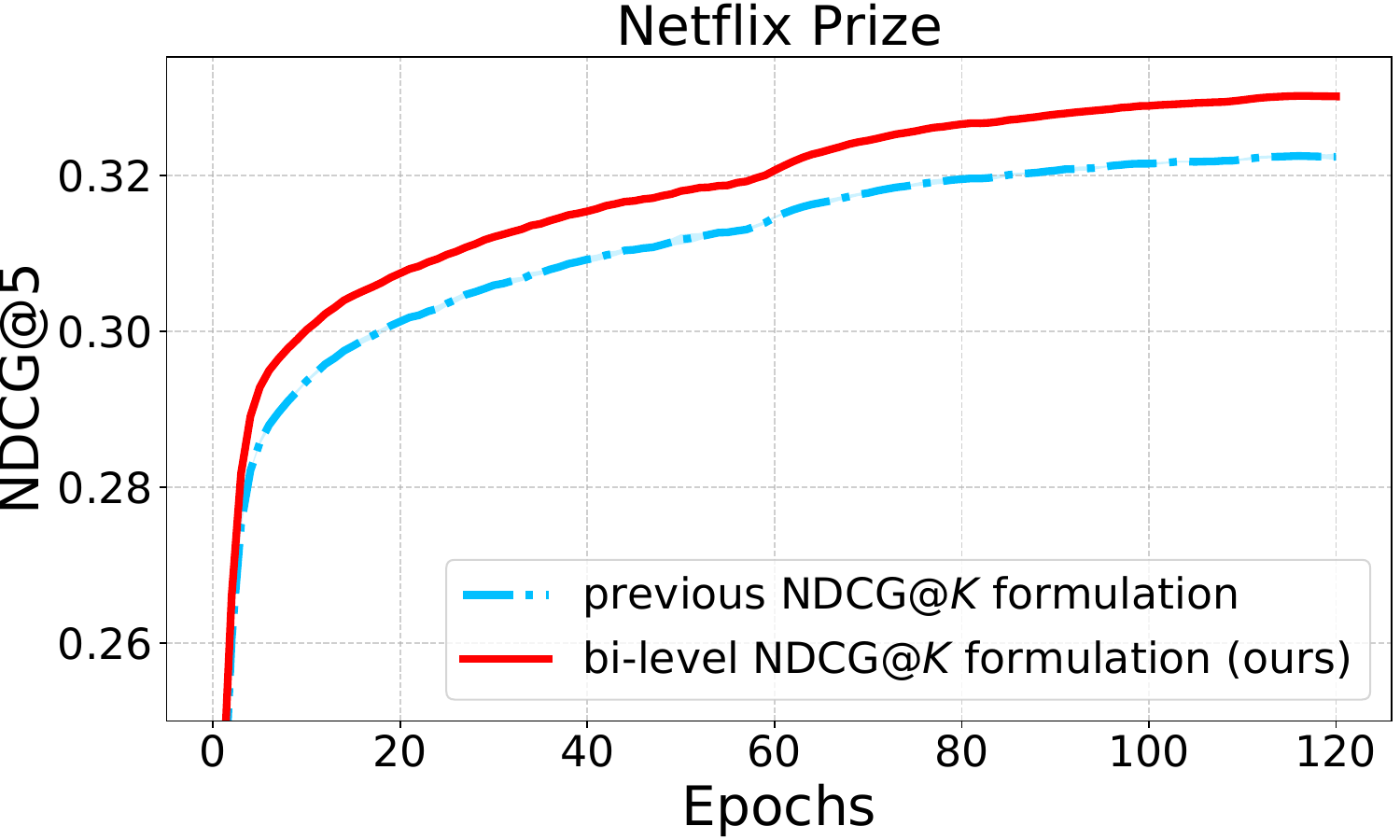}
\end{minipage}
\caption{Comparison of our bilevel NDCG@$K$ formulation and previous NDCG@$K$ formulation.}
\label{fig:ndcg-topk-comp}
\end{figure}

\begin{figure}
\centering
\begin{minipage}[c]{0.4\textwidth}
\centering\includegraphics[width=1\textwidth]{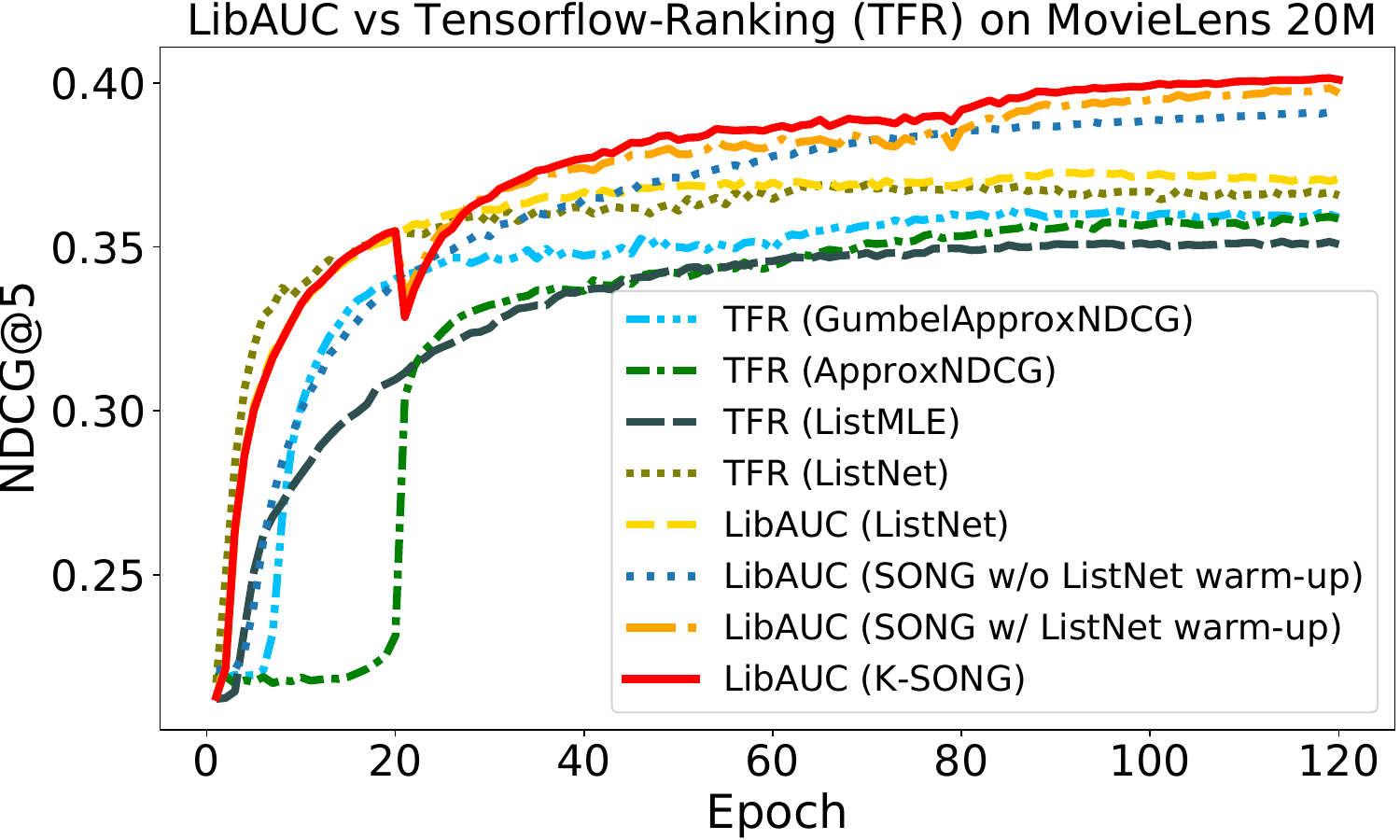}
\end{minipage}
\begin{minipage}[c]{0.4\textwidth}
\centering\includegraphics[width=1\textwidth]{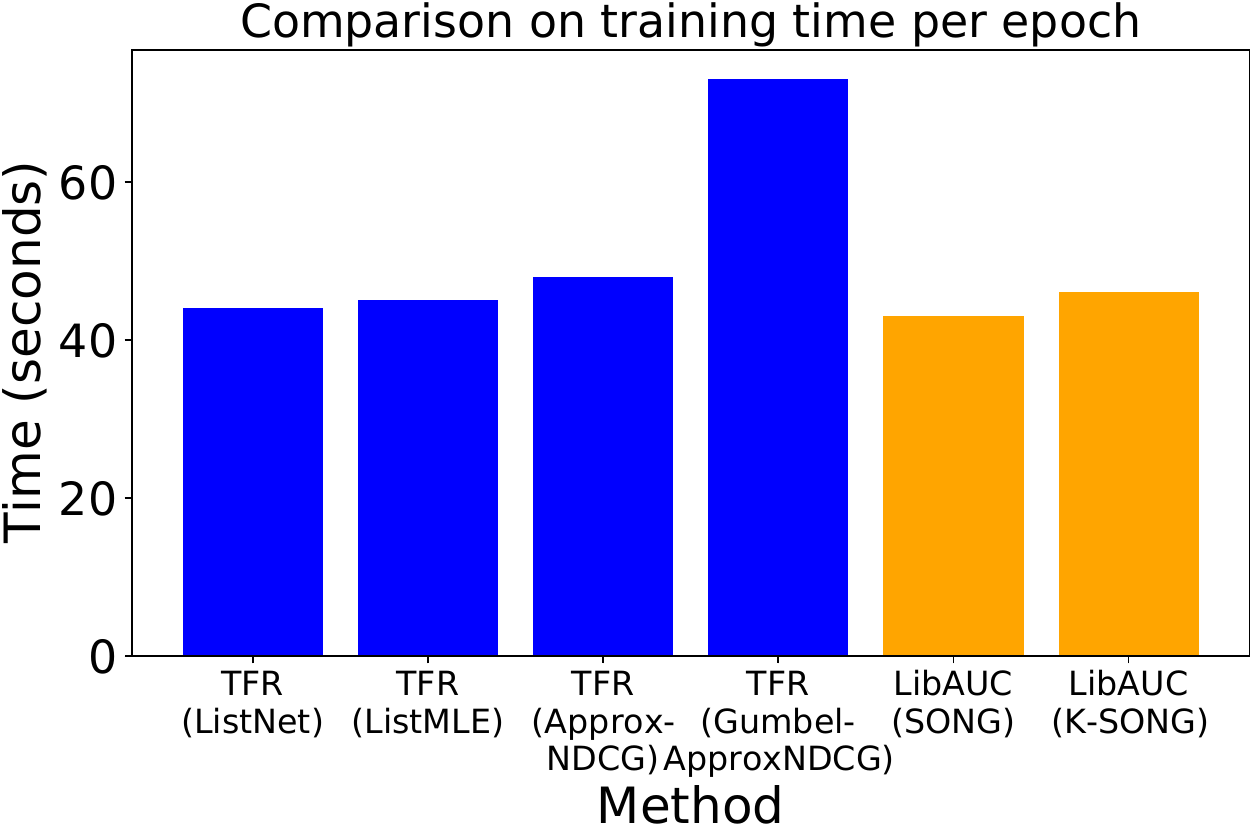}
\end{minipage}
\caption{Comparison of convergence (left) and training time (right) between  LibAUC (ours) and TensorFlow Ranking library.}
\label{fig:tfr-comp}
\end{figure}

\section{Conclusion}
\label{sec:conclusion}

In this work, we propose stochastic methods to optimize NDCG and its top-$K$ variant that have been widely used in various ranking tasks. The optimization problems of NDCG and top-$K$ NDCG are casted into a novel compositional optimization problem and a novel bilevel optimization problem, respectively. We design efficient stochastic algorithms with provable convergence guarantee to compute the solutions. We also study initial warm-up and stop gradient operator to improve the effectiveness for deep learning. Extensive experimental results on multiple domains demonstrate that our methods can achieve promising results.

\section*{Acknowledgements}
Q. Hu, Y. Zhong and T. Yang were partially supported by NSF Grant 2110545 and NSF Career Award 1844403.  Z. Qiu and L. Zhang were partially supported by NSFC (62122037, 61921006). Part work of Z. Qiu was done when he was visiting T. Yang's lab virtually. 

\newpage
\bibliography{arxiv-ndcg}
\bibliographystyle{icml2022}

\newpage
\appendix
\onecolumn

\section{Per-iteration Complexity} \label{appendix:per-iteration-complexity}

For complexity analysis, let $\Q_t$ denote the sampled queries at the $t$-th iteration and $\B^+_q$ denote the sampled relevant items for each sampled query.  
In terms of the per-iteration complexity of SONG, we need to conduct forward propagation for computing $h_q(\x^q_i,\w), \forall \x_q^i\in\B^+_q\cup \B_q$ and back-propagation for computing  $\nabla h_q(\x^q_i,\w), \forall \x_q^i\in\B^+_q\cup \B_q$. The complexity for these forward propagations and back-propagations is $\sum_{q\in\Q_t}(|\B^q_+|+|\B_q|)d \leq O(Bd)$, where $B=\sum_{q\in\Q_t}(|\B^+_q|+|\B_q|)$ is the total mini-batch size. With these computed,  the cost for computing $\hat g_{q,i}(\w_t)$ and $\nabla \hat g_{q,i}(\w_t)$ for all $q, \x^q_i\in\B^+_q$ is $ \sum_{q\in\Q_t}|\B_q^+||\B_q| \leq O(B^2)$. Hence, the total complexity per iteration is $O(Bd+B^2)$. For a large model size $d\gg B$, we have the per-iteration complexity of $O(Bd)$, which is similar to the standard cost of deep learning per-iteration and is independent of the length of $\S_q$ for each query.

\section{Initial Warm-up}
\label{appendix:initial-warm-up}

The listwise cross-entropy loss can be reformulated as follows:
\begin{align*} 
&\min_{\w}\quad \frac{1}{N}\sum_{q=1}^N\frac{1}{N_q}\sum_{\x_i^q\in \S^+_q} -\ln\left(\frac{\exp(h_q(\x^q_i; \w)}{ \sum_{\x_j^q\in\S_q} h_q(\x^q_j; \w))}\right)\\
&= \frac{1}{N}\sum_{q=1}^N\frac{1}{N_q}\sum_{\x_i^q\in\S^+_q}\ln\left(\sum_{\x^q_j\in\S_q}\exp(h_q(\x^q_j) - h_q(\x^q_i))\right).
\end{align*} 
The above objective has the same structure of the NDCG surrogate, i.e., it is an instance of finite-sum coupled compositional stochastic optimization problem. Hence, we can use a similar algorithm to SONG to solve the above problem. We present the details in Algorithm~\ref{alg:3}. 

\begin{algorithm}[h]
\caption{Stochastic Optimization of Listwise CE loss: SOLC}\label{alg:3}
\begin{algorithmic}
\REQUIRE $\eta,\beta_0,\beta_1, u^{(1)}=0$
\ENSURE $\w_T$
\FOR{$t=1,...T$}
\STATE draw a set of queries denoted by  $\Q_t$
\STATE For each query draw a batches of examples  $\{\B^+_q, \B_q\}$, where $\B^+_q$ denote a set of sampled relevant documents for $q$ and $\B_q$ denote a set of sampled documents from $\S_q$ 
\FOR{$\x^q_i\in \B^+_q$ for each $q\in\Q_t$}
\STATE $u^{(t+1)}_{q, i}=(1-\gamma_0)u^{(t)}_{q, i}+\gamma_0\frac{1}{|\B_q|}\sum_{\x'\in\B_q}  \exp(h_q(\x'; \w) - h_q(\x; \w))$
\STATE Compute $p_{q, i} = 1/u^{t+1}_{q,i}$
\ENDFOR
\STATE Compute gradient $$G(\w_t)= \frac{1}{|\Q_t|}\frac{1}{|\B_q^+|}\frac{1}{|\B_q|}\sum_{q\in\Q_t}\sum_{\x^q_i\in\B^+_q}\sum_{\x^q_j\in\B_q}p_{q, i}\nabla_\w(h_q(\x^q_j; \w_t) - h_q(\x^q_i; \w_t))$$
\STATE Compute $\m_{t+1} = \beta_1\m_t + (1-\beta_1) G(\w_t)$
\STATE Update $\w_{t+1}=\w_t - \eta \m_{t+1}$
\ENDFOR
\end{algorithmic}
\end{algorithm}

\section{Justification of Stop Gradient Operator}
\label{appendix:stop-gradient-operator}

Below, we provide a justification by showing that the second term in~(\ref{eqn:update_grad}) is close to 0 under a reasonable condition. For simplicity of notation, we let $\psi_i(\w, \hat\lambda_q(\w)) = \psi(h(\x^q_i,\w)- \hat\lambda_q(\w))$.  Its gradient is given by 
\begin{align*}
\nabla_\w\psi_i=\psi'_i(\w,\hat\lambda_q(\w))\bigg(\nabla_\w h(\x^q_i, \w) -\nabla_{\w \lambda}^2 L_q(\w,\hat\lambda_q(\w)) [\nabla_{\lambda}^2 L_q(\w,\hat\lambda_q(\w))]^{-1}\bigg).
\end{align*}
For the purpose of justification, we can approximate $\phi(h_q(\x_i;\w)-\lambda)= \tau_1\log(1+\exp((h_q(\x_i;\w)-\lambda)/\tau_1))$ by a smoothed hinge loss function, $\kappa(h_q(\x_i;\w)-\lambda)=\max_{\alpha}\alpha(h_q(\x_i;\w)-\lambda)-\tau_1\alpha^2/2$, which is equivalent to
\begin{align*}
    \kappa(h_q(\x_i;\w)-\lambda)=
    \begin{cases}
    0, & h_q(\x_i;\w)-\lambda\leq 0 \\
    \frac{(h_q(\x_i;\w)-\lambda)^2}{2\tau_1}, & 0<h_q(\x_i;\w)-\lambda\leq\tau_1 \\
    h_q(\x_i;\w)-\lambda - \frac{\tau_1}{2}, & h_q(\x_i;\w)-\lambda> \tau_1
    \end{cases}
\end{align*}
Please refer to Figure~\ref{fig:three_funcs} for the curves of $[\cdot]_+$ and $\phi(\cdot)$ and $\kappa(\cdot)$. Below, we assume $L_q(\w, \lambda)$ is defined by using $\kappa(h_q(\x_i; \w) -\lambda)$ in place of $\phi(h_q(\x_i;\w)-\lambda)$.

For any $\w$, let us consider a subset $\C_q=\{\x^q_i\in\S_q^+: h_\w(\x^q_i) - \hat\lambda_q(\w)\in(0, \tau_1)\}$. It is not difficult to show that 
\begin{align*}
&\nabla_{\w \lambda}^2 L_q(\w,\hat\lambda_q(\w))= \frac{1}{N_q}\sum_{\x^q_i\in\C_q}\frac{-\partial_\w h(\x^q_i; \w)}{\tau_1}\\
&\nabla_{\lambda}^2 L_q(\w,\hat\lambda_q(\w))= \frac{1}{N_q}\sum_{\x^q_i\in\C_q}\frac{1}{\tau_1}+\tau_2\approx \frac{1}{N_q}\sum_{\x^q_i\in\C_q}\frac{1}{\tau_1}
\end{align*}
for sufficiently small $\tau_1,\tau_2$. Then we have $$\frac{\nabla_{\w \lambda}^2 L_q(\w,\lambda(\w))}{ \nabla_{\lambda}^2 L_q(\w,\lambda_q(\w))} = \frac{1}{|\C_q|}\sum_{\x^q_i\in\C_q}-\partial h_\w(\x^q_i).$$
Assume that $\psi$ is chosen such that $\psi'_i(\w,\lambda_q(\w))\approx 0$ if $h_\w(\x^q_j)-\lambda_q(\w)\not\in[0,\tau_1]$, and $\psi'_i(\w,\lambda_q(\w))\approx c_1$ and $f_{q,i}(g(\w; \x^q_i, \S_q))\approx c_2$ if $h_\w(\x^q_j)-\lambda_q(\w)\in[0,\tau_1]$, then we have

\begin{align*}
\sum_{\x^q_i\in\S_q}\nabla_\w\psi_i f_{q,i}(g(\w; \x^q_i, \S_q))&\approx \sum_{\x_i^q\in\C_q}\psi'_i(\w,\lambda_q(\w))\cdot
    \bigg(\nabla_\w h(\x^q_i; \w)- \frac{1}{|\C_q|}\sum_{\x_j^q\in\C_q}\nabla_\w h(\x^q_j; \w)\bigg)f_{q,i}(g(\w; \x^q_i, \S_q))\\
    &\approx c_1c_2\sum_{\x_i^q\in\C_q}\left(\nabla_\w h_\w(\x^q_i; \w) + \frac{1}{|\C_q|}\sum_{\x_j^q\in\C_q}-\nabla_\w h(\x^q_j; \w)\right)\\
    &=0
\end{align*}
As a result, when $\tau_1$ is small enough the condition $\psi'_i(\w,\lambda_q(\w))\approx 0$ if $h_\w(\x^q_j)-\lambda_q(\w)\not\in[0,\tau_1]$,  and $\psi'_i(\w,\lambda_q(\w))\approx c$ if $h_\w(\x^q_j)-\lambda_q(\w)\in[0,\tau_1]$ is well justified. An example of such $\psi(\cdot)$ is provided in the Figure~\ref{fig:smoothed_ind}.  As a result, with initial warm-up, we can compute the gradient estimator by 

\begin{equation*}
    \begin{aligned}
    &G(\w_t)=\frac{1}{|\B|}\sum_{(q,\x^q_i)\in\B}p_{q,i}\nabla\hat g_{q,i}(\w_t),
    \end{aligned}
\end{equation*}
which simplifies K-SONG by avoiding maintaining and updating $s_{q,t}$.

\begin{figure}
  \begin{minipage}[t]{0.5\linewidth}
    \centering
    \includegraphics[scale=0.3]{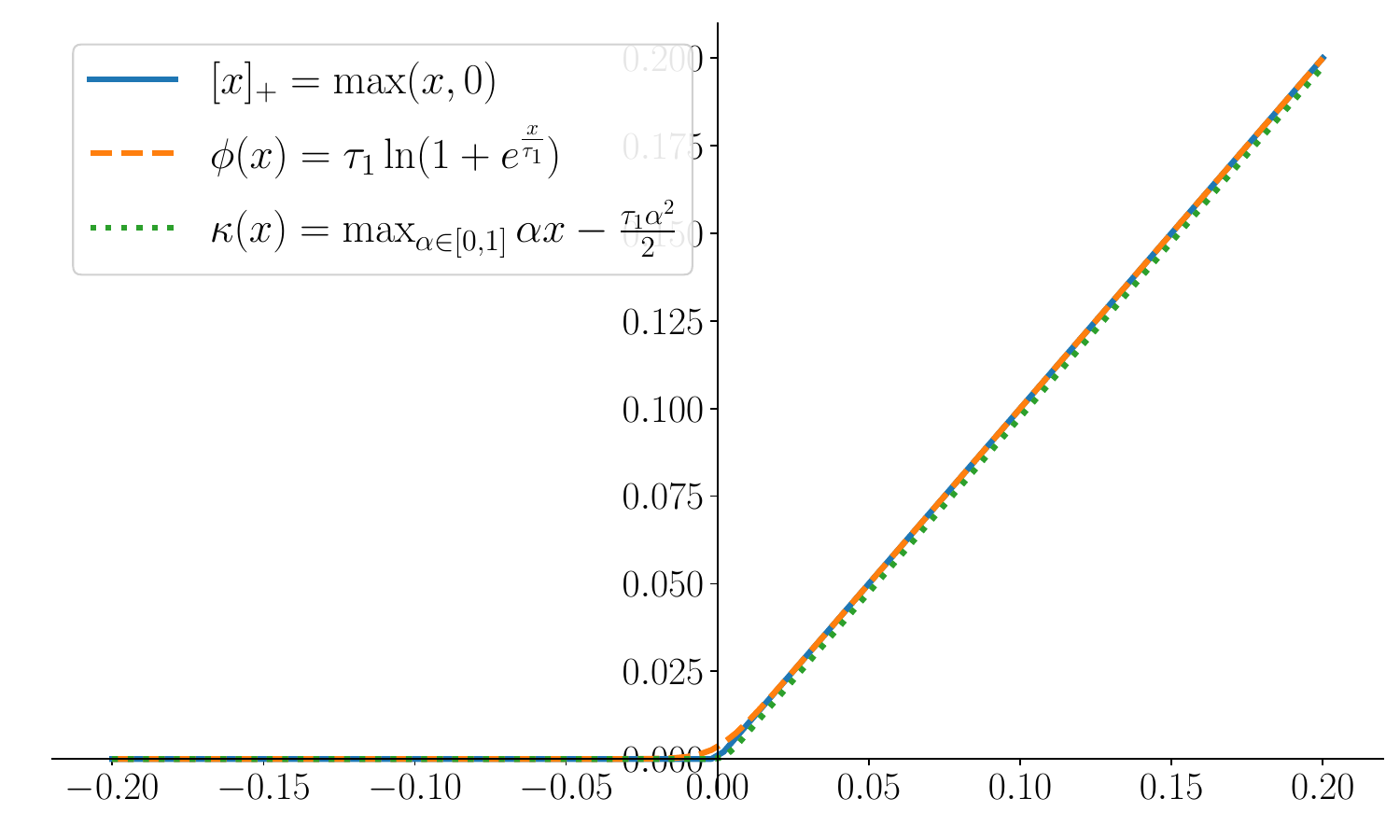}
    \caption{Curves of $[\cdot]_{+}$, $\phi(\cdot)$, and $\kappa(\cdot)$.}
    \label{fig:three_funcs}
  \end{minipage}%
  \begin{minipage}[t]{0.5\linewidth}
    \centering
    \includegraphics[scale=0.3]{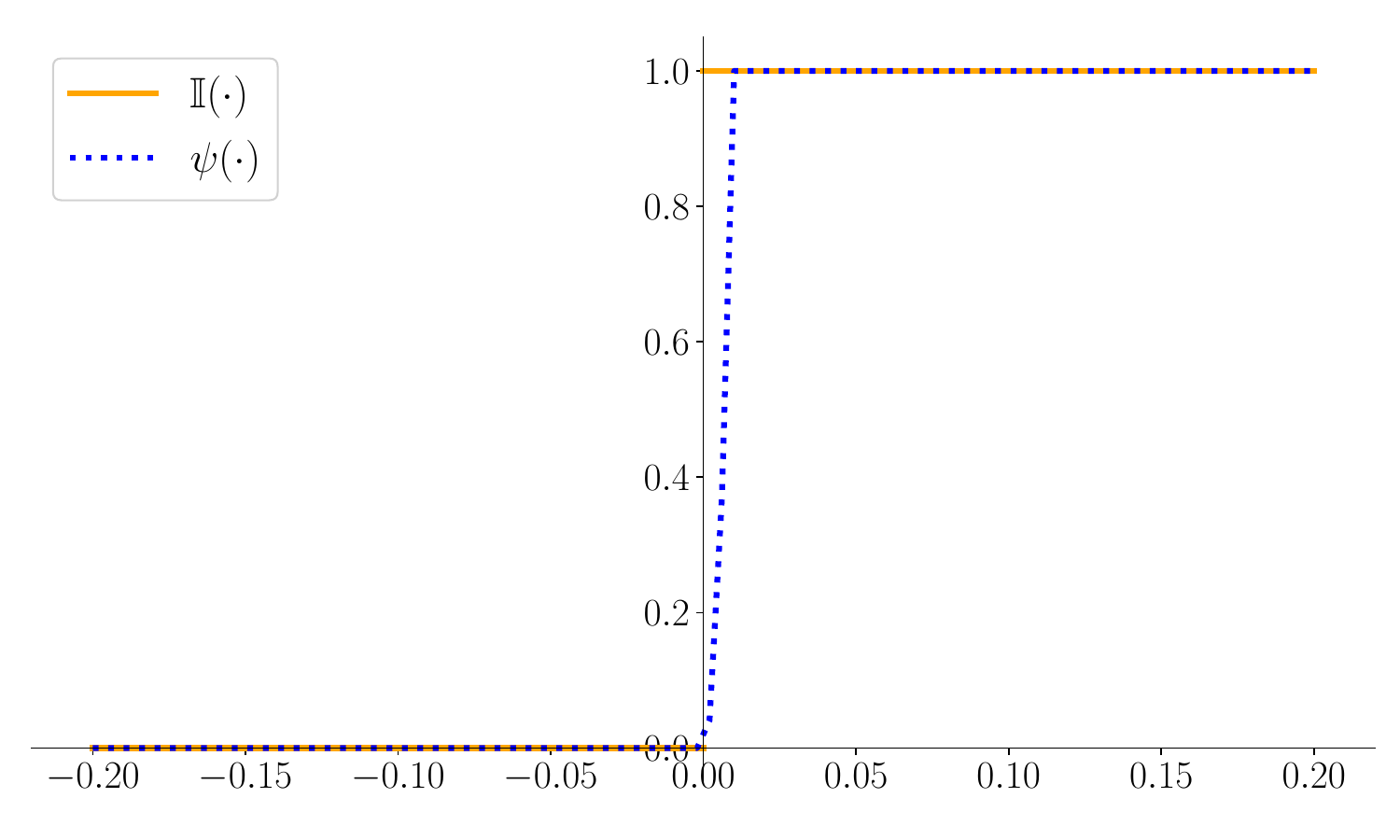}
    \caption{An example of $\psi(\cdot)$}
    \label{fig:smoothed_ind}
  \end{minipage}
\end{figure}

\section{Experiments}

\subsection{Details of Implementation}\label{appendix:implemention}

For the experiments on two LTR datasets, we adopt allRank framework\footnote{https://github.com/allegro/allRank}~\cite{ltr-model}. We implement some baseline methods based on their code. For the recommender systems experiments, we use ReChorus framework\footnote{https://github.com/THUwangcy/ReChorus}~\cite{wang2020make}, which is a general PyTorch framework for Top-K recommendation. We also follow the scripts in ReChorus to preprocess the datasets. The hyper-parameters for SONG and K-SONG are presented in Table~\ref{hyper-parameters}.

We train our models on one Tesla V100 GPU with 32GB memory. The training on the Context-Aware Ranker model takes about 2\textasciitilde3 hours for convergence, while the training of the NeuMF model takes about 8\textasciitilde12 hours for convergence.

\subsection{Details of Data}
\label{appendix:data}

MSLR-WEB30K\footnote{https://www.microsoft.com/en-us/research/project/mslr/} and Yahoo! LTR dataset\footnote{https://webscope.sandbox.yahoo.com} are the largest public LTR datasets from commercial English search engines. We provide the statistics of these two datasets in Table~\ref{ltr-datasets-statistics}. In MSLR-WEB30K dataset, there are 5 folds containing the same data, and each fold randomly splits to training, validation, and test sets. Due to privacy concerns, these datasets do not disclose any text information and only provide feature vectors for each query-document pair. For these two LTR datasets, we standarize the features, log-transforming selected ones, before feeding them to the learning algorithms. Since the lengths of search results lists in the datasets are unequal, we truncate or pad samples to the length of 40 and 100 for Yahoo! LTR dataset and MSLR-WEB30K when training, respectively, but use the full list for evaluation.

MovieLens20M\footnote{https://grouplens.org/datasets/movielens/20m/} contains 20 million ratings applied to 27,000 movies by 138,000 users, and all users have rated at least 20 movies. Netflix Prize dataset\footnote{https://www.kaggle.com/netflix-inc/netflix-prize-data} consists of about 100,000,000 ratings for 17,770 movies given by 480,189 users. We filter the Netflix Prize dataset by retaining users with at least $100$ interactions to cater sufficient information for modeling. In both datasets, users and movies are represented with integer IDs, while ratings range from 1 to 5. The statistics of these two datasets are shown in Table~\ref{rs-datasets-statistics}.

\begin{table}[t]
\caption{Statistics of Learning to Rank Datasets.}
\label{ltr-datasets-statistics}
\vskip 0.1in
\begin{center}
\begin{small}
\begin{sc}
\begin{tabular}{lcc}
\toprule
Dataset & MSLR-WEB30K & Yahoo! LTR dataset \\
\midrule
Query    &  30,000  & 29,921 \\
Q-D pair &  3,771,125  & 709,877 \\
max Q-D pair per query  & 1,245 & 135 \\
min Q-D pair per query & 1 & 1 \\
\bottomrule
\end{tabular}
\end{sc}
\end{small}
\end{center}
\end{table}

\begin{table}[t]
\caption{Statistics of Recommender Systems Datasets.}
\label{rs-datasets-statistics}
\vskip 0.1in
\begin{center}
\begin{small}
\begin{sc}
\begin{tabular}{lcccc}
\toprule
Dataset & \# users & \# items & \# interactions & sparsity \\
\midrule
MovieLens20M & 138,493  & 26,744  & 20,000,263 & 99.46\% \\
Netflix Prize dataset & 236,117  & 17,770 & 89,973,534 & 97.86\% \\
\bottomrule
\end{tabular}
\end{sc}
\end{small}
\end{center}
\end{table}

\begin{table}[t]
\caption{Hyper-parameters for SONG and K-SONG.}
\label{hyper-parameters}
\vskip 0.1in
\begin{center}
\begin{small}
\begin{sc}
\begin{tabular}{lcccc}
\toprule
 & MovieLens20M & Netflix Prize & MSLR-WEB30K & Yahoo! LTR\\
\midrule
$\gamma_0$ & 0.1 & 0.3 & 0.3 & 0.3 \\
$K$ & 300 & 300 & 10 & 10 \\
\bottomrule
\end{tabular}
\end{sc}
\end{small}
\end{center}
\end{table}

\subsection{Additional Experimental Results}
\label{sec:additional-exp-results}

\begin{table*}[t]
\caption{The test NDCG on two Learning to Rank datasets. We report the average NDCG@$k$ ($k\in[1,3,5]$) and standard deviation over 3 runs with different random seeds.}
\label{tab:ltr-testing-results}
\vskip -1in
\begin{center}
\begin{small}
\begin{sc}
\begin{tabular}{p{2.0cm}p{2.0cm}p{2.0cm}p{2.0cm}p{2.0cm}p{2.0cm}p{2.0cm}}
\toprule
\multirow{2}{*}{\thead{Method}} &
\multicolumn{3}{c}{\thead{MSLR WEB30K}} &
\multicolumn{3}{c}{\thead{Yahoo! LTR Dataset}} \\
\cmidrule(lr){2-4}
\cmidrule(lr){5-7}
& NDCG@1 & NDCG@3 & NDCG@5 & NDCG@1 & NDCG@3 & NDCG@5 \\
\midrule
RankNet  & 0.5138$\pm$0.0008 & 0.5105$\pm$0.0004 & 0.5159$\pm$0.0003
 & 0.7066$\pm$0.0006 & 0.7150$\pm$0.0004 & 0.7368$\pm$0.0005 \\
ListNet  & 0.5105$\pm$0.0001 & 0.5058$\pm$0.0001 & 0.5146$\pm$0.0002
 & 0.7066$\pm$0.0002 & 0.7151$\pm$0.0004 & 0.7352$\pm$0.0004 \\
ListMLE  & 0.5153$\pm$0.0012 & 0.5074$\pm$0.0002 & 0.5136$\pm$0.0005
 & 0.7067$\pm$0.0008 & 0.7146$\pm$0.0006 & 0.7353$\pm$0.0007 \\
LambdaRank  & 0.5173$\pm$0.0014 & 0.5118$\pm$0.0003 & 0.5187$\pm$0.0003
 & 0.7084$\pm$0.0003 & 0.7155$\pm$0.0002 & 0.7352$\pm$0.0004 \\
ApproxNDCG & 0.5204$\pm$0.0007 & 0.5114$\pm$0.0005 & 0.5179$\pm$0.0006
 & 0.7085$\pm$0.0009 & 0.7152$\pm$0.0007 & 0.7350$\pm$0.0006 \\
NeuralNDCG & 0.5160$\pm$0.0006 & 0.5101$\pm$0.0005 & 0.5155$\pm$0.0002
 & 0.7076$\pm$0.0003 & 0.7139$\pm$0.0001 & 0.7349$\pm$0.0003 \\
SONG & 0.5265$\pm$0.0005 & 0.5136$\pm$0.0006 & \textbf{0.5206}$\pm$0.0003
 & \textbf{0.7131}$\pm$0.0002 & 0.7187$\pm$0.0004 & 0.7390$\pm$0.0002 \\
K-SONG & \textbf{0.5271}$\pm$0.0006 & \textbf{0.5147}$\pm$0.0006 & 0.5204$\pm$0.0003
 & 0.7128$\pm$0.0004 & \textbf{0.7191}$\pm$0.0004 & \textbf{0.7394}$\pm$0.0008\\
\bottomrule
\end{tabular}
\end{sc}
\end{small}
\end{center}
\end{table*}

\begin{table*}[t]
\caption{The test NDCG on two movie recommendation datasets. We report the average NDCG@$k$ ($k\in[10,20,50]$) and standard deviation over 3 runs with different random seeds.}
\label{tab:rs-testing-results}
\vskip 0.1in
\begin{center}
\begin{small}
\begin{sc}
\begin{tabular}{p{2.0cm}p{2.0cm}p{2.0cm}p{2.0cm}p{2.0cm}p{2.0cm}p{2.0cm}}
\toprule
\multirow{2}{*}{\thead{Method}} &
\multicolumn{3}{c}{\thead{MovieLens20M}} &
\multicolumn{3}{c}{\thead{Netflix Prize Dataset}} \\
\cmidrule(lr){2-4}
\cmidrule(lr){5-7}
& NDCG@10 & NDCG@20 & NDCG@50 & NDCG@10 & NDCG@20 & NDCG@50 \\
\midrule
RankNet  & 0.0538$\pm$0.0011 & 0.0744$\pm$0.0013 & 0.1086$\pm$0.0013 & 0.0362$\pm$0.0002 & 0.0489$\pm$0.0003  & 0.0730$\pm$0.0003 \\
ListNet  & 0.0660$\pm$0.0003 & 0.0875$\pm$0.0004 & 0.1227$\pm$0.0003 & 0.0532$\pm$0.0002 & 0.0700$\pm$0.0002  & 0.0992$\pm$0.0002 \\
ListMLE  & 0.0588$\pm$0.0001 & 0.0799$\pm$0.0001 & 0.1137$\pm$0.0001 & 0.0376$\pm$0.0003 & 0.0508$\pm$0.0004 & 0.0753$\pm$0.0001 \\
LambdaRank  & 0.0697$\pm$0.0001 & 0.0913$\pm$0.0002 & 0.1259$\pm$0.0001 & 0.0531$\pm$0.0002 & 0.0693$\pm$0.0002 & 0.0976$\pm$0.0003 \\
ApproxNDCG & 0.0735$\pm$0.0005 & 0.0938$\pm$0.0003 & 0.1284$\pm$0.0002 & 0.0434$\pm$0.0005 & 0.0592$\pm$0.0009 & 0.0873$\pm$0.0012 \\
NeuralNDCG & 0.0692$\pm$0.0003 & 0.0901$\pm$0.0003 & 0.1232$\pm$0.0007 & 0.0554$\pm$0.0002  & 0.0718$\pm$0.0003 & 0.1003$\pm$0.0002 \\
SONG & \textbf{0.0748}$\pm$0.0002 & 0.0969$\pm$0.0002 & 0.1326$\pm$0.0001 & 0.0571$\pm$0.0002 & \textbf{0.0749}$\pm$0.0002  & \textbf{0.1050}$\pm$0.0003 \\
K-SONG & 0.0747$\pm$0.0002 & \textbf{0.0973}$\pm$0.0003 & \textbf{0.1340}$\pm$0.0001 & \textbf{0.0573}$\pm$0.0003 & 0.0743$\pm$0.0003  & 0.1042$\pm$0.0001 \\
\bottomrule
\end{tabular}
\end{sc}
\end{small}
\end{center}
\vskip -0.1in
\end{table*}

\begin{figure*}[!h]
\centering
\begin{minipage}[c]{0.24\textwidth}
\centering\includegraphics[width=1\textwidth]{training_curves/ml-20m-training-curves.pdf}
\end{minipage}
\begin{minipage}[c]{0.24\textwidth}
\centering\includegraphics[width=1\textwidth]{training_curves/netflix-training-curves.pdf}
\end{minipage}
\begin{minipage}[c]{0.24\textwidth}
\centering\includegraphics[width=1\textwidth]{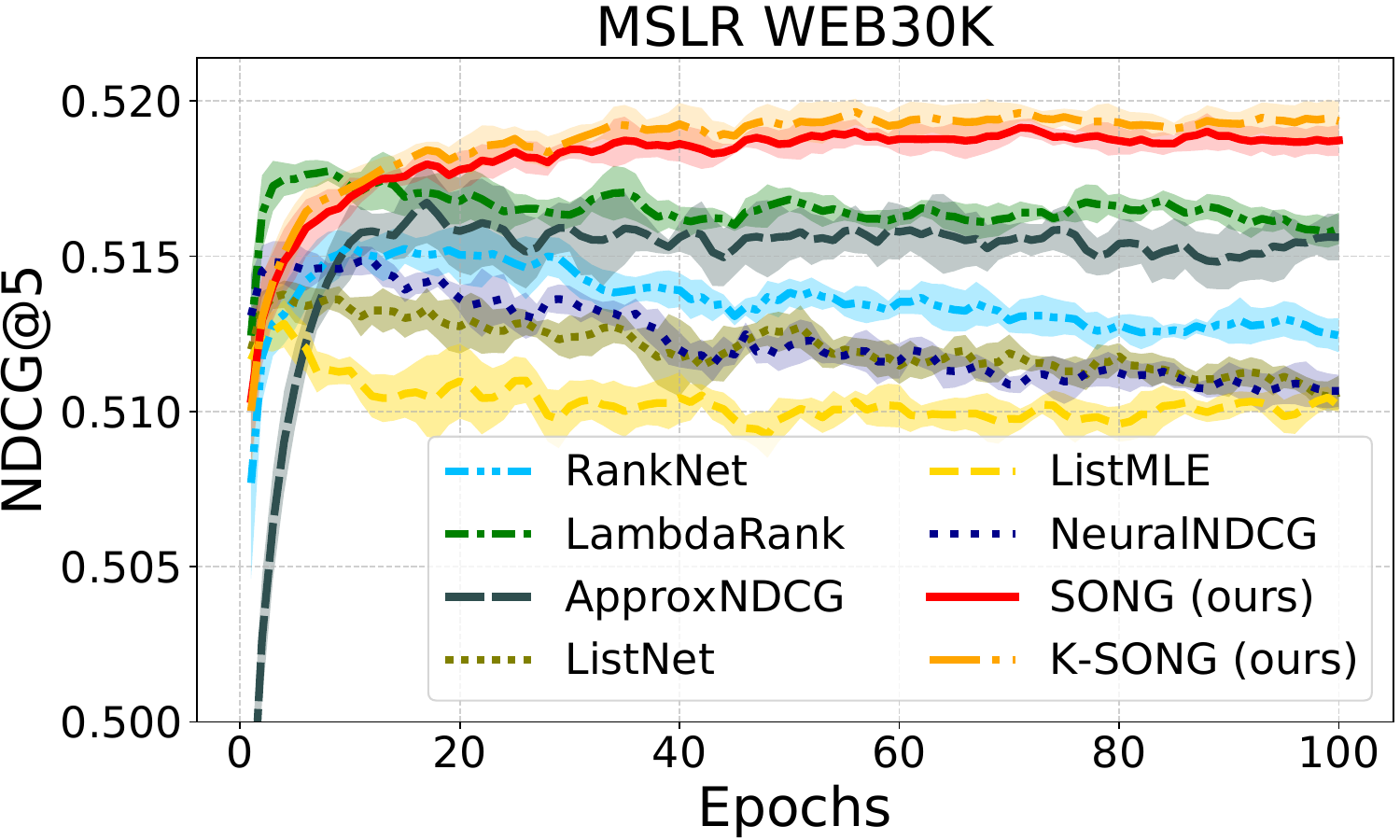}
\end{minipage}
\begin{minipage}[c]{0.24\textwidth}
\centering\includegraphics[width=1\textwidth]{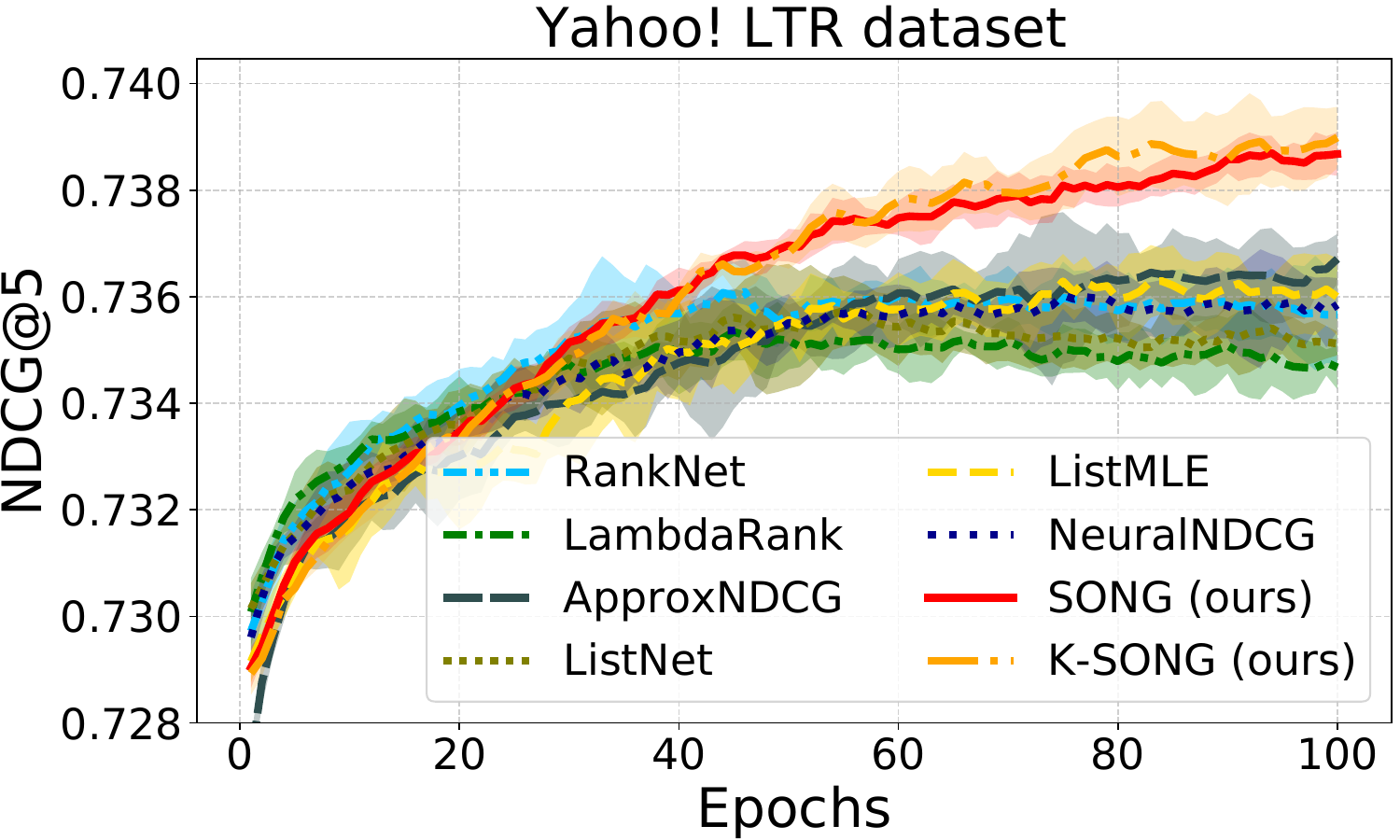}
\end{minipage}
\caption{Comparison of convergence of different methods in terms of validation NDCG@5 scores.}
\label{fig:training-curves}
\end{figure*}

\begin{figure*}[!h]
\centering
\begin{minipage}[c]{0.24\textwidth}
\centering\includegraphics[width=1\textwidth]{ablation_study/ml-20m-ablations.pdf}
\end{minipage}
\begin{minipage}[c]{0.24\textwidth}
\centering\includegraphics[width=1\textwidth]{ablation_study/netflix-ablations.pdf}
\end{minipage}
\begin{minipage}[c]{0.24\textwidth}
\centering\includegraphics[width=1\textwidth]{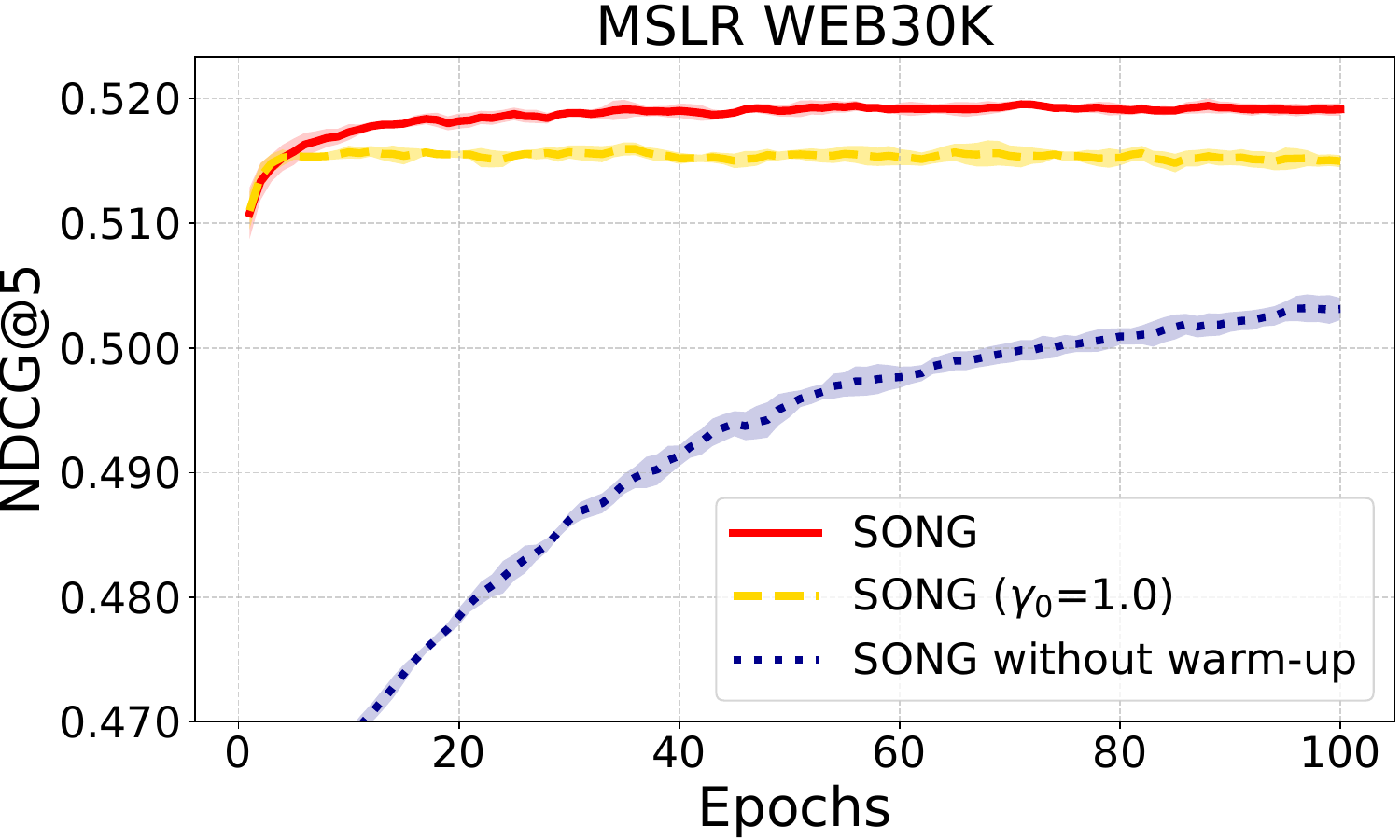}
\end{minipage}
\begin{minipage}[c]{0.24\textwidth}
\centering\includegraphics[width=1\textwidth]{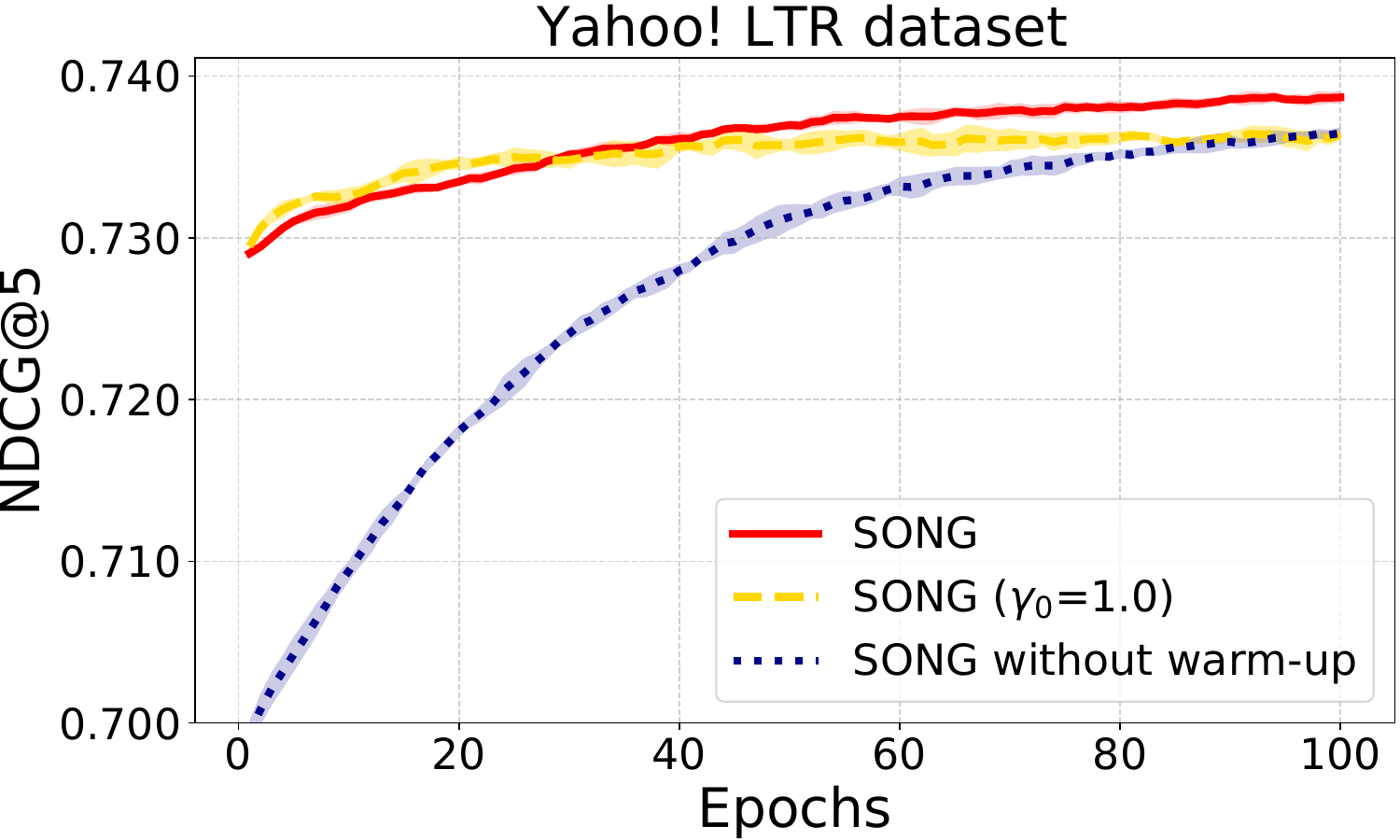}
\end{minipage}
\caption{Ablation study on two variants of SONG on four different datasets.}
\label{fig:ablation}
\end{figure*}

{\bf Convergence Speed.} We present the training curves on four different datasets (MovieLens20M, Netflix Prize dataset, MSLR WEB30K, and Yahoo! LTR dataset) in Figure~\ref{fig:training-curves}.

{\bf Ablation Studies.} We provide the full ablation studies on four datasets in Figure~\ref{fig:ablation}.

{\bf The Effect of Varying $\gamma_0$.} We adjust $\gamma_0$ in our algorithms from $\{0.1,0.3,0.5,0.7,1.0\}$, and report the training curves of warm-up and SONG in Figure~\ref{fig:varying_gamma}. We observe that $\gamma_0=0.1$ achieves the best performance in most cases. Setting $\gamma_0=1.0$ is equivalent to update the model with a biased stochastic gradient, which leads to the worst performance. These results signify the importance of moving average estimators in our methods.

\begin{figure*}[!h]
\centering
\begin{minipage}[c]{0.24\textwidth}
\centering\includegraphics[width=1\textwidth]{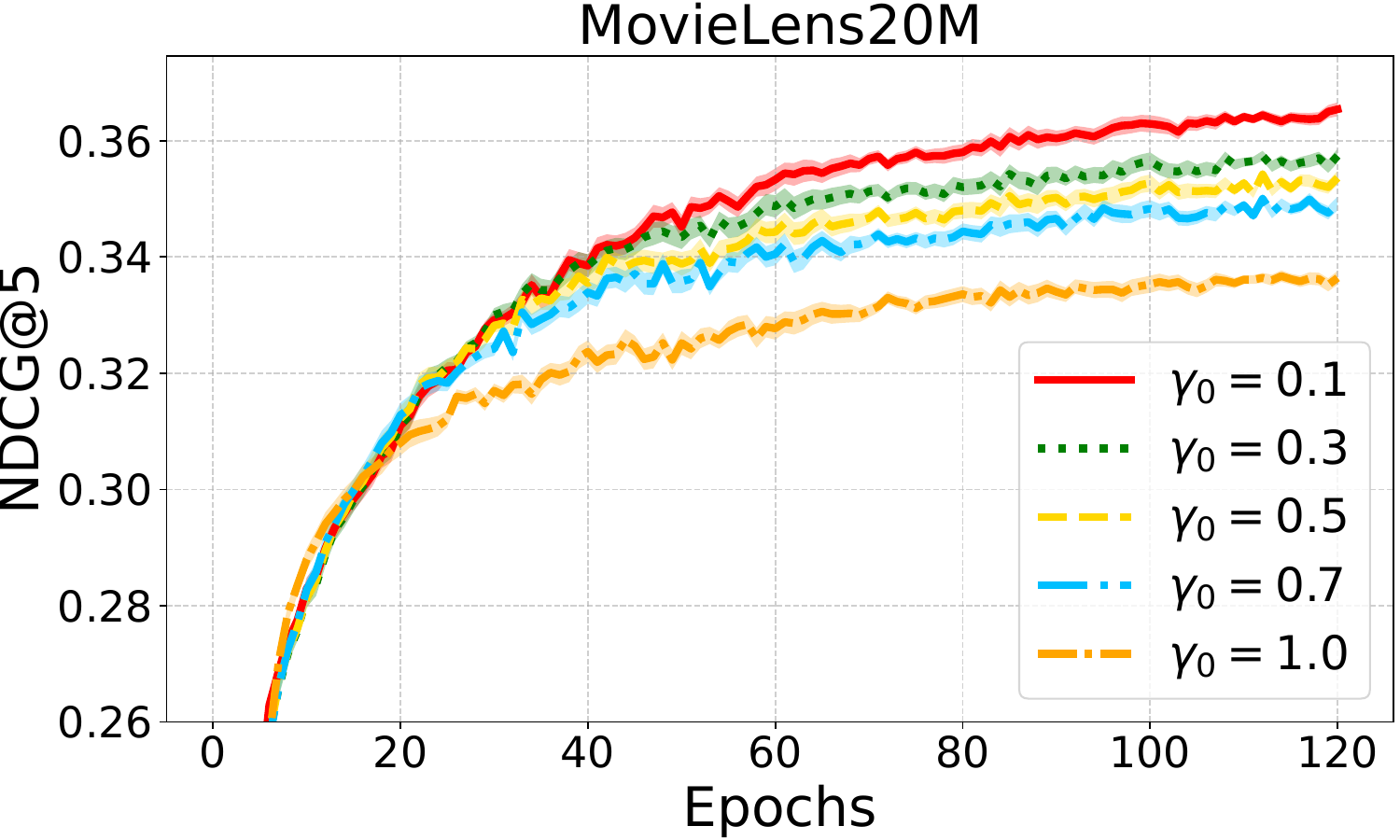}
\end{minipage}
\begin{minipage}[c]{0.24\textwidth}
\centering\includegraphics[width=1\textwidth]{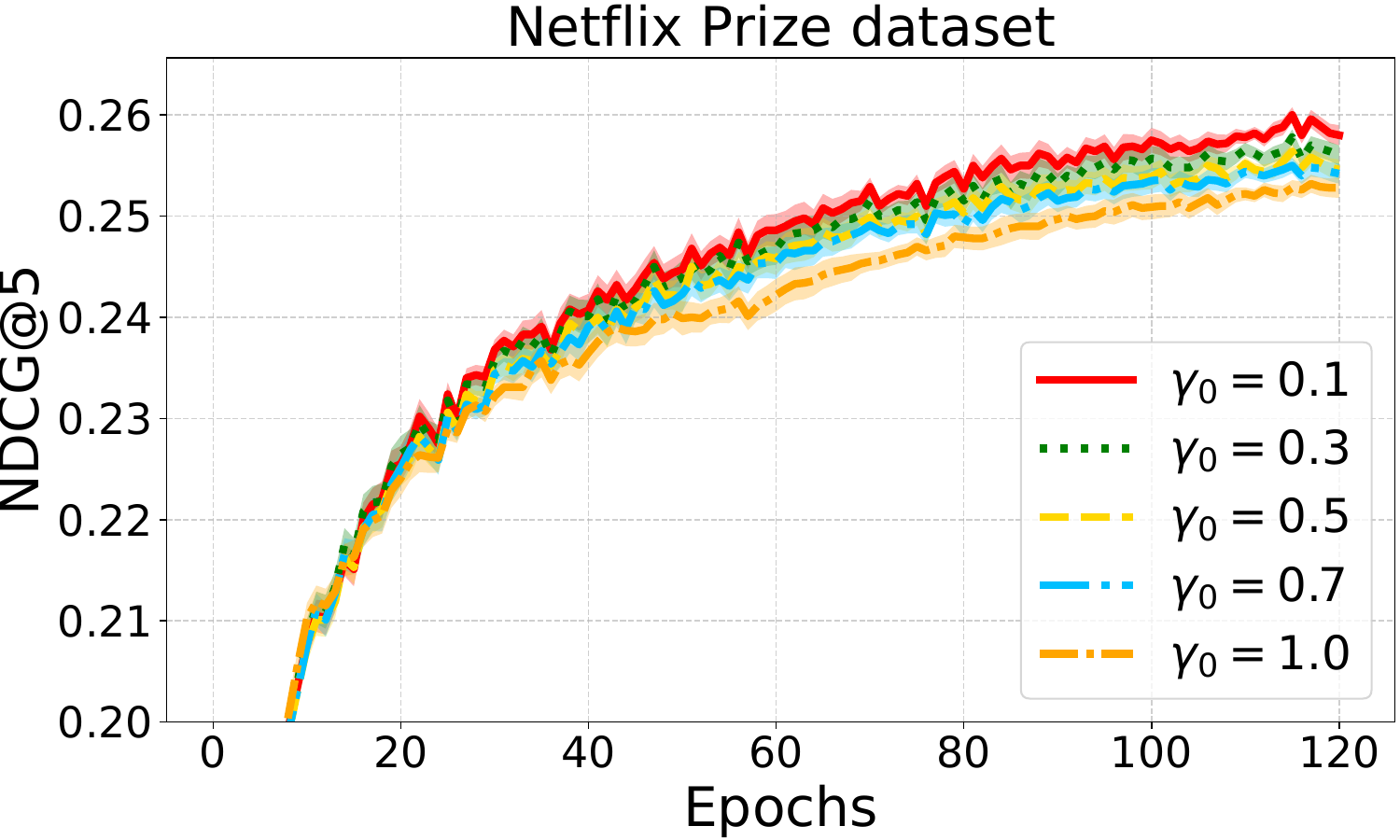}
\end{minipage}
\begin{minipage}[c]{0.24\textwidth}
\centering\includegraphics[width=1\textwidth]{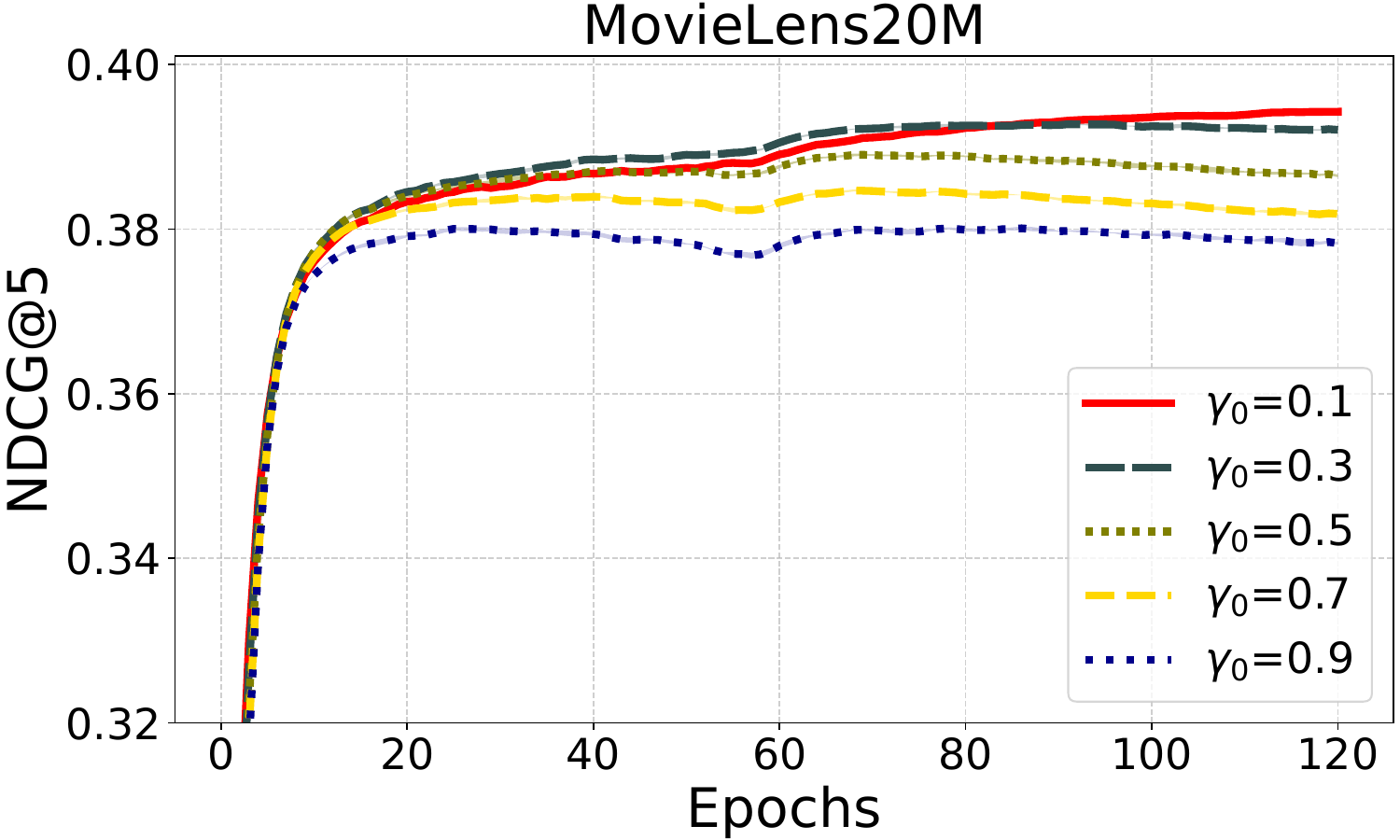}
\end{minipage}
\begin{minipage}[c]{0.24\textwidth}
\centering\includegraphics[width=1\textwidth]{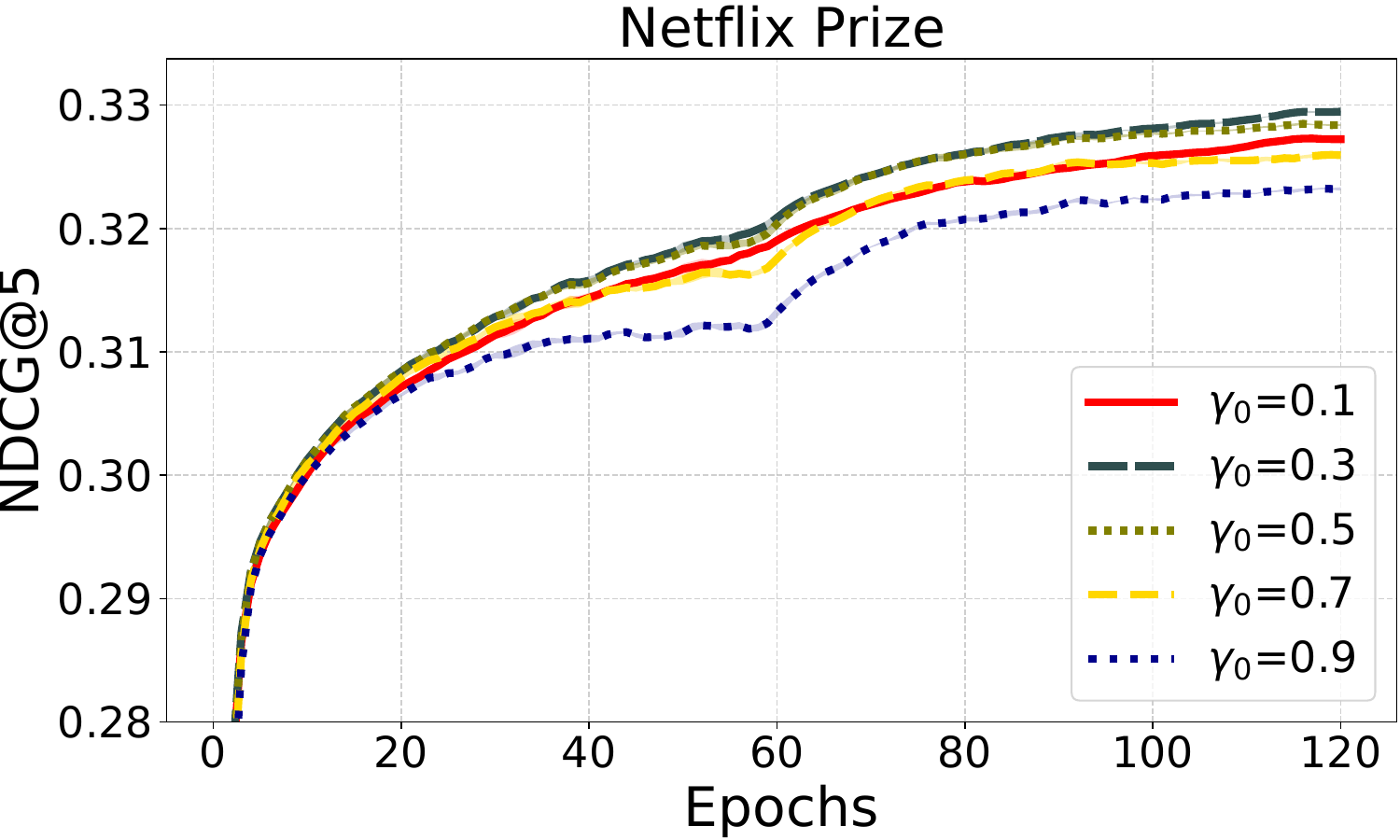}
\end{minipage}
\caption{The effect of varying $\gamma_0$ for warm-up (left two) and SONG (right two).}
\label{fig:varying_gamma}
\vskip -0.0in
\end{figure*}

{\bf Comparison with Full-Items Training.} We provide the negative loglikelihood loss curves of three different training methods for warm-up in Figure~\ref{fig:full_batch_comp}.

\begin{figure}[!h]
\centering
\begin{minipage}[c]{0.24\textwidth}
\centering\includegraphics[width=1\textwidth]{full_batch_comp/ml-20m-ndcg.pdf}
\end{minipage}
\begin{minipage}[c]{0.24\textwidth}
\centering\includegraphics[width=1\textwidth]{full_batch_comp/netflix-ndcg.pdf}
\end{minipage}
\begin{minipage}[c]{0.24\textwidth}
\centering\includegraphics[width=1\textwidth]{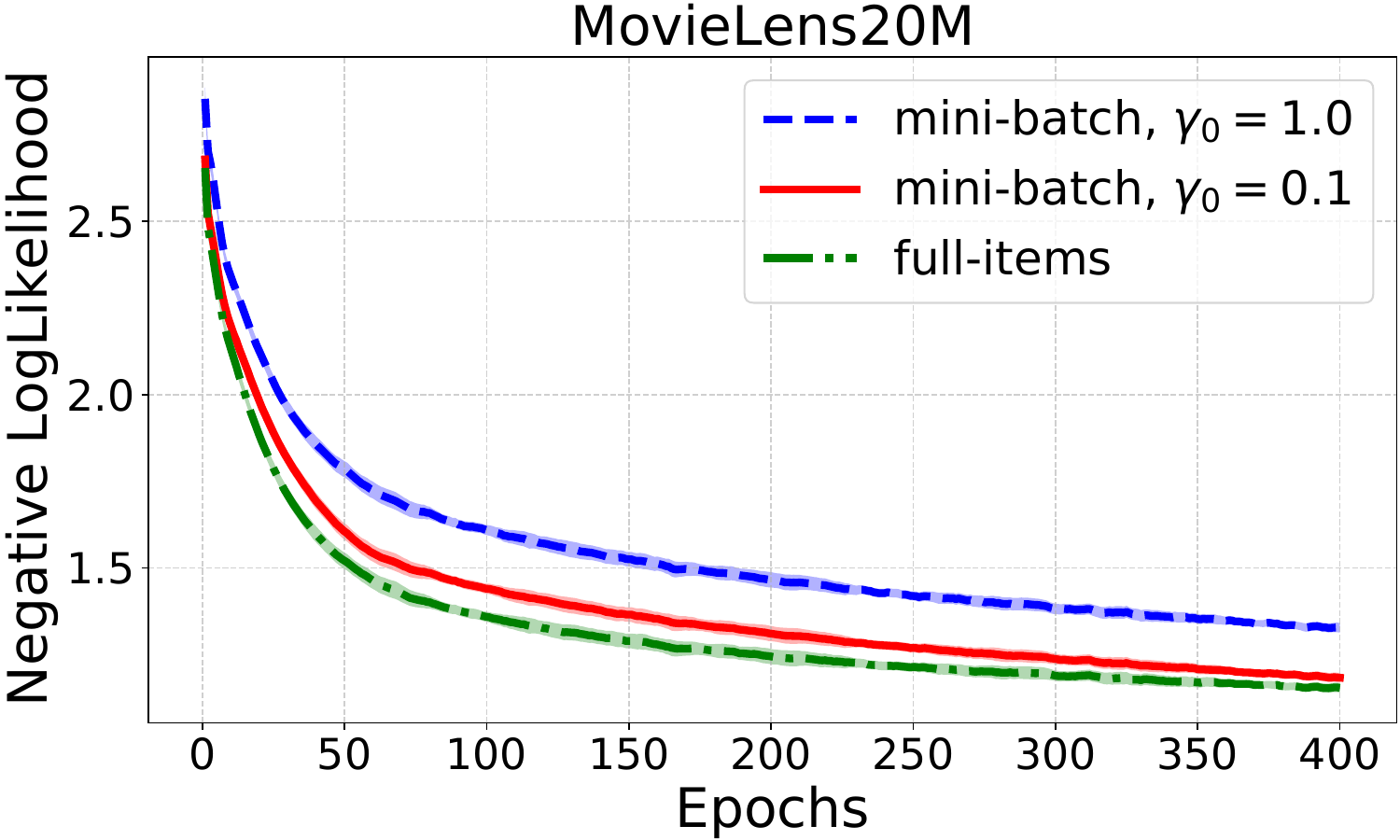}
\end{minipage}
\begin{minipage}[c]{0.24\textwidth}
\centering\includegraphics[width=1\textwidth]{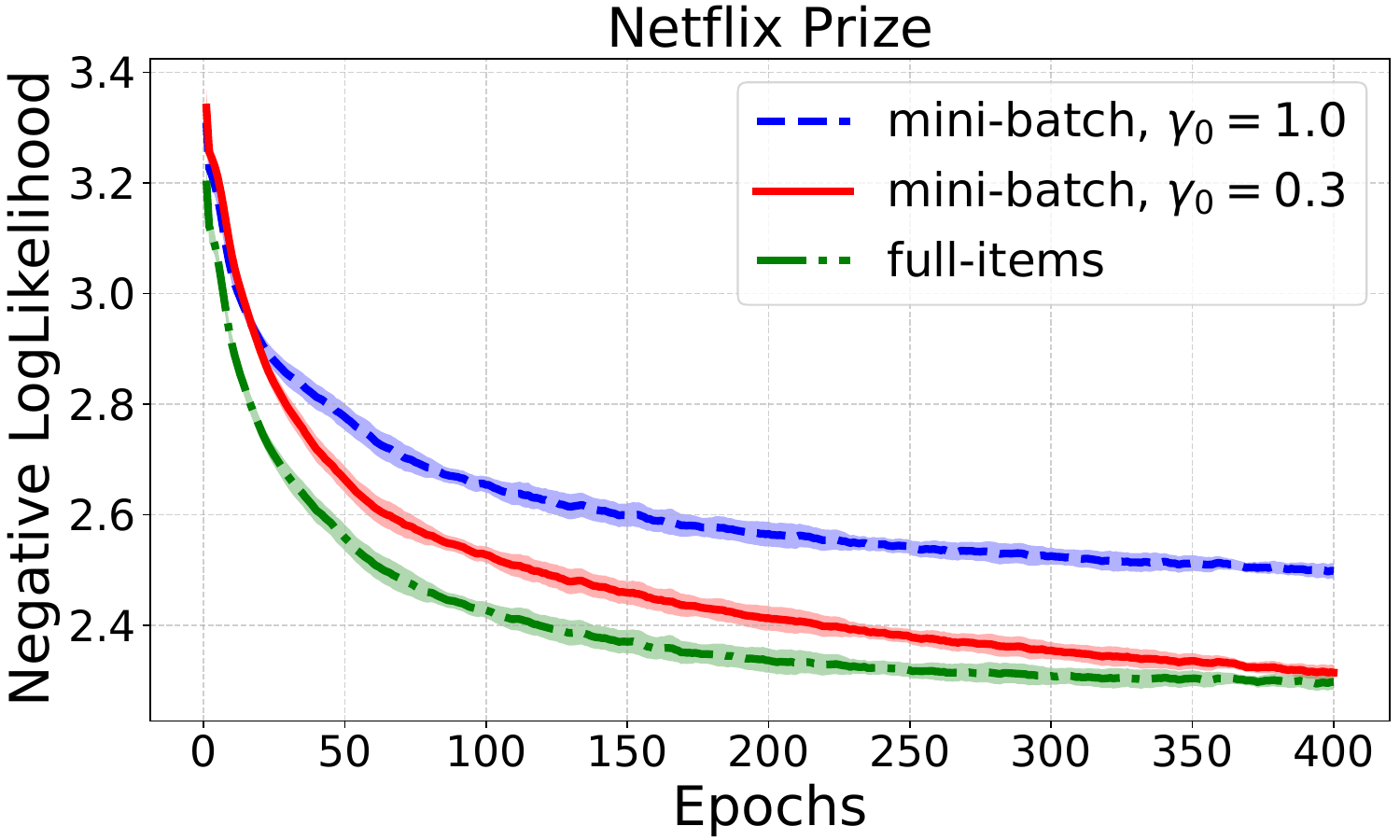}
\end{minipage}
\caption{Comparison of full-items and mini-batch training on SONG (left two) and warm-up (right two).}
\label{fig:full_batch_comp}
\end{figure}

\subsection{Experiments on Multi-label Classification}
\label{sec:multi-label-cls}

To further verify the effectiveness of our methods, we also conduct experiments on multi-label datasets. Similar to Learning to Rank task, we treat each instance as a query and each label as an item.  We adopt XML-CNN\footnote{https://github.com/siddsax/XML-CNN}~\citep{liu2017deep} as our base model. We download data from The Extreme Classification Repository\footnote{http://manikvarma.org/downloads/XC/XMLRepository.html\#Bi13} and conduct experiments on two datasets: EUR-Lex~\citep{mencia2008efficient} and Wiki10-31K~\citep{zubiaga2012enhancing}. The statistics of these two datasets are presented in Table~\ref{stats_ml_data}. In our experiments, we use raw data to classify.

\begin{table}[!hbtb]
    \centering
    \caption{Statistics of Multi-label Datasets.}
    \begin{tabular}{cccccc}
    \toprule
        \multirow{2}*{Dataset} & \multirow{2}*{Labels} & Training &Testing& Avg. Points&Avg. Labels \\
        && Samples&Samples& per Label&per Points\\ \midrule
        EURLex-4K & 3,993&	15,539&	3,809&	25.73&	5.31\\
        Wiki10-31K& 309,38& 14,146 & 6,616 & 8.52 & 18.64 \\\bottomrule
    \end{tabular}
    \label{stats_ml_data}
\end{table}

In extreme multi-label (XML) classification, label spaces usually are large; however, each instance only has very few relevant labels. Therefore, we adopt NDCG@$k$ as our evaluation metric, which is also a common way in evaluating XML methods.  

Upon XML-CNN, we compare our method with other NDCG optimization methods: ApproxNDCG~\citep{ApproxNDCG} and NeuralNDCG~\citep{NeuralNDCG}. The results are summarized in Table~\ref{results_ml_data}.

\begin{table}[!hbtb]
    \centering
    \caption{Results in NDCG@$k$; bold indicates the best performance among all methods}
    \begin{tabular}{ccccccc}
    \toprule
         Datasets & Metrics & Baseline & ApproxNDCG & NeuralNDCG & SONG & K-SONG \\\midrule
         \multirow{2}{*}{EUR-Lex} 
         &NDCG@3&67.15&66.59&67.68&67.84&\textbf{68.11} \\
         &NDCG@5&61.13&60.23&\textbf{61.86}&61.32&61.74 \\ \midrule
         \multirow{2}{*}{Wiki10-31K}
         &NDCG@3&71.26&71.49&71.52&72.90&\textbf{74.01} \\
         &NDCG@5&63.23&62.43&62.85&65.10&\textbf{66.15} \\ \bottomrule
    \end{tabular}
    \label{results_ml_data}
\end{table}

\section{Convergence Analysis}\label{appendix:convergence-analysis}

\subsection{Analysis of SONG}
For simplicity, we rewrite problem~(\ref{eqn:NDCG}) as the following compositional optimization problem,
\begin{equation}\label{prob_ana1}
\min_\w \quad\frac{1}{n}\sum_{i\in\S}  f_i(g_i(\w)).
\end{equation} 
One may reorder the set of queries $\S$ so that each pair $(q,\x_i^q)$ has a single index. We abuse the notation $\S$ denoting the set of the new indexing. Then the equivalence between problem~(\ref{eqn:NDCG}) and (\ref{prob_ana1}) is established. Furthermore, SONG can be rewritten as Algorithm~\ref{algo_ana1} accordingly. In fact, problem~(\ref{prob_ana1}) can be seen as a special case of problem~(\ref{prob_ana2}), where $\psi_i$'s are constant functions. Hence, Theorem~\ref{thm:2} naturally follows from Theorem~\ref{thm:3}, of which the proof will be presented in the following section.

\begin{algorithm}[!h]
\caption{}\label{algo_ana1}
\begin{algorithmic}
\REQUIRE $\w_0,\m_0, u^0, \gamma_0, \beta_1, \eta$
\ENSURE $\w_T$
\FOR{$t=0,1,\dots,T-1$}
\STATE Draw batch of queries $\B_1^t\in \{1,\dots,n\}$
\STATE Draw batch of items $\B_{2,i}^t$ for each $i\in \B_1^t$
\STATE Compute $ u_i^{t+1}=\begin{cases}(1-\gamma_0) u_{i}^t+\gamma_0 g_i(\w_t;\B_{2,i}^t)\quad & \text{if }i\in \B_1^t\\  u_{i}^t & \text{o.w.} \end{cases}$
\STATE Compute stochastic gradient estimator $G(\w_t)=\frac{1}{|\B_1^t|}\sum_{i\in \B_1^t} \nabla g_i(\w_t;\B_{2,i}^t)\nabla f_i( u_i^t)$
\STATE $\m_{t+1}=\beta_1\m_t+(1-\beta_1) G(\w_t)$
\STATE $\w_{t+1}=\w_t-\eta \m_{t+1}$
\ENDFOR
\end{algorithmic}
\end{algorithm}

\subsection{Analysis of K-SONG}
In this section, we present a convergence analysis for K-SONG. Similarly to the analysis of SONG, we reorder the set of queries $\S$ and generalize problem~(\ref{eqn:KNDCG}) into the following compositional bilevel optimization problem,
\begin{equation}\label{prob_ana2}
\begin{aligned}
&\min_\w \quad F(\w):=\frac{1}{n}\sum_{i\in\S} \psi_i(\w, \lambda_i(\w)) f_i(g_i(\w))\\
& s.t.\quad \lambda_i(\w) = \arg\min_{\lambda} L_i(\w, \lambda).
\end{aligned}
\end{equation}
This allows us to rewrite K-SONG into Algorithm~\ref{algo_ana2} accordingly.

\textbf{Notations:} Throughout this convergence analysis section, all subscript $i$ represents the block of variable or function corresponding to the $i$th query. The following notations will be used,
\begin{equation*}
    \begin{aligned}
    &\delta_{\lambda,t}:=\|\lambda(\w_t)-\lambda^t\|^2,\quad
    \delta_{g,t}:=\|g(\w_t)-u^t\|^2, \quad
    \delta_{L\lambda\lambda,t}:=\|\nabla^2_{\lambda\lambda}L(\w_t,\lambda(\w_t))-s^t\|^2
    \end{aligned}
\end{equation*}

\begin{algorithm}[t]
\caption{}\label{algo_ana2}
\begin{algorithmic}
\REQUIRE $\w_0,\m_0,\lambda^0,u^0,s^0, \gamma_0,  \gamma_0', \beta_1,\eta_0,  \eta_1$
\ENSURE $\w_T$
\FOR{$t=0,1,\dots,T-1$}
\STATE Draw batch of queries $\B_1^t\in \{1,\dots,n\}$
\STATE Draw batch of items $\B_{2,i}^t$ for each $i\in \B_1^t$
\STATE Compute $  u_i^{t+1}=\begin{cases}(1-\gamma_0)  u_{i}^t+\gamma_0 g_i(\w_t;\B_{2,i}^t)\quad & \text{if }i\in \B_1^t\\   u_{i}^t & \text{o.w.} \end{cases}$
\STATE Compute $\lambda_i^{t+1}=\begin{cases}\lambda_{i}^t-\eta_0 \nabla_\lambda L_i(\w_t,\lambda_i^t;\B_{2,i}^t)\quad & \text{if }i\in \B_1^t\\ \lambda_{i}^t & \text{o.w.}\end{cases}$
\STATE Compute $s_i^{t+1}=\begin{cases} (1-\gamma_0')s_i^t+\gamma_0'\nabla^2_{\lambda\lambda}L_i(\w_t,\lambda_i^t;\B_{2,i}^t) \quad & \text{if }i\in \B_1^t\\ s_i^t &\text{o.w.}\end{cases}$
\STATE Compute stochastic gradient estimator $G(\w_t)$ according to (\ref{update_grad})
\STATE $\m_{t+1}=\beta_1\m_t+(1-\beta_1) G(\w_t)$
\STATE $\w_{t+1}=\w_t-\eta_1 \m_{t+1}$
\ENDFOR
\end{algorithmic}
\end{algorithm}

We make the following assumptions regarding problem~(\ref{prob_ana2}).
\begin{ass}\label{assump_ana}
\,
\begin{itemize}
    \item Functions $\psi_i,f_i,g_i,L_i$ are $L_\psi,L_f,L_g, L_L$-smooth respectively for all $i$. 
    \item Functions $\psi_i,f_i,g_i,\lambda_i$ are $C_\psi,C_f,C_g,C_\lambda$-Lipschitz continuous respectively for all $i$. Function $L_i$ is $\mu_L$-strongly convex for all $i$.
    \item $\nabla^2_{\w\lambda}L_i(\w,\lambda),\nabla^2_{\lambda\lambda}L_i(\w,\lambda)$ are $L_{L\w\lambda},L_{L\lambda\lambda}$-Lipschitz continuous respectively with respect to $(\w,\lambda)$ for all $i$. 
    \item $\psi_i$ and $f_i$ are bounded by $B_\psi$ and $B_f$ respectively, i.e. $\|\psi_i(\w,\lambda)\|\leq B_\psi$ and $\|f_i(g)\|\leq B_f$ for all $\w,\lambda, i,g$.
    \item $\|\nabla^2_{\w\lambda}L_i(\w,\lambda)\|^2\leq C_{L\w\lambda}^2$, $\gamma I \preceq \nabla^2_{\lambda\lambda}L_i(\w,\lambda;\B)\preceq L_L I$ for all $i$
    \item Unbiased stochastic oracles $g_i,\nabla g_i,\nabla_\lambda L_i,\nabla^2_{\lambda\lambda}L_i,\nabla^2_{\w\lambda}L_i$
    have bounded variance $\sigma^2$.
\end{itemize}
\end{ass}

Now we show that problem~(\ref{eqn:KNDCG}) satisfies Assumption~\ref{assump_ana}. Here we consider the squared hinge loss $\ell(h_q(\x';\w),h_q(\x;\w))=\max\{0,h_q(\x';\w)-h_q(\x;\w)+c\}^2$ where $c$ is a margin parameter. Suppose the score function and its gradients $h_q(\x; \w),\nabla_\w h_q(\x; \w),\nabla^2_\w h_q(\x; \w)$ are bounded by finite constants $c_h,c_{h'},c_{h''}$ respectively. As an average of squared hinge loss, function $g_i(\w)$ in problem~(\ref{prob_ana2}) has bounded gradients $\nabla g_i(\w)\leq 8c_h c_{h'}$ and $\nabla^2 g_i(\w)\leq 8c_{h'}^2+8c_h c_{h''}$ for each $i\in \S$. Hence $g_i$ is Lipschitz continuous and smooth. Moreover, with $m>2c_h$, there exists $c_\ell>0$ such that $\ell(h_q(\x_1;\w)-h_q(\x_2;\w))\geq c_\ell$ for all $\x_1,\x_2$. Function $f_i(g)=f_{q,i}(g)=\frac{1}{Z_q}\frac{1-2^{y^q_i}}{\log_2(N_q g+ 1)}$ is thus bounded, Lipschitz continuous and smooth for each $i=(q,\x_i^q)\in \S$. For function $\psi_i=\psi(h_q(\x_i^q;\w)-\lambda)$, we consider the logistic loss, then $\psi_i$ is naturally bounded, Lipschitz continuous and smooth. The smoothness and strong convexity of $L_i$ are proved in Lemma~\ref{lemma:4}. In fact, the strong convexity of $L_i$ implies the lower boundedness $\gamma=\tau_2$ of $\nabla_{\lambda\lambda}L_i(\w,\lambda;\B)$. To show the Lipschitz continuity of $\nabla^2_{\w\lambda} L_q(\lambda; \w)$ and $\nabla^2_{\lambda\lambda} L_q(\lambda; \w)$ one may simply take the third gradients of $L_q(\lambda;\w)$ and use the fact $\exp(\frac{\lambda-h_q(\x_i;\w)}{\tau_1})>0$ and the assumption of the boundedness of $h_q(\x;\w)$ and its gradients to verify.

By using the implicit function theorem, the stochastic gradient estimator of $\nabla F(\w_t)$ in Algorithm~\ref{algo_ana2} is given by
\begin{equation}\label{update_grad}
    \begin{aligned}
    G(\w_t):=\frac{1}{|\B_1^t|}\sum_{i\in \B_1^t} G_i(\w_t)
    &=\frac{1}{|\B_1^t|}\sum_{i\in \B_1^t} \bigg[\nabla_\w\psi_i(\w_t,\lambda_i^t)-\nabla_{\w \lambda}^2 L_i(\w_t,\lambda_i^t;\B_{2,i}^t) [s_i^{t}]^{-1} \nabla_\lambda \psi_i(\w_t,\lambda_i^t)\bigg]f_i(u_i^t)\\
    &\quad\quad\quad\quad\quad +\psi_i(\w_t,\lambda_i^t)\nabla g_i(\w_t;\B_{2,i}^t)\nabla f_i( u_i^t)
    \end{aligned}
\end{equation}
Note that the parameter $\tau_\lambda$ in the update of $\lambda^{t+1}$ exists only for theoretical analysis reason. In practical, $\tau_\lambda \eta_0$ can be treated as one parameter. Moreover, we define the gradient approximation at iteration $t$
\begin{equation*}
    \begin{aligned}
    \nabla F(\w_t,\lambda^t, u^t)
    &=\frac{1}{n}\sum_{i\in \S} \bigg[\nabla_\w\psi_i(\w_t,\lambda_i^t)-\nabla_{\w \lambda}^2 L_i(\w_t,\lambda_i^t) [\nabla_{\lambda\lambda}^2 L_i(\w_t,\lambda_i^t)]^{-1} \nabla_\lambda \psi_i(\w_t,\lambda_i^t)\bigg]f_i( u_i^t)\\
    &\quad\quad\quad\quad +\psi_i(\w_t,\lambda_i^t)\nabla g_i(\w_t)\nabla f_i( u_i^t)
    \end{aligned}
\end{equation*}

Now we present the formal statement of Theorem~\ref{thm:3} regarding to problem~(\ref{prob_ana2}).

\begin{thm}\label{thm_ana}
Let $F(\w_0)-F(\w^*)\leq \Delta_F$. Under Assumption~\ref{assump_ana} and consider Algorithm~\ref{algo_ana2}, with 
$\eta_0\leq \min\left\{\frac{\mu_L}{L_L^2},\frac{2n}{|\B_1^t|\mu_L}, \frac{\mu_L\epsilon^2}{48C_5 \sigma^2}\right\} $,
$\gamma_0\leq \left\{\frac{1}{2}, \frac{\epsilon^2}{96C_6\sigma^2}\right\}$,
$\gamma_0'\leq \left\{1,\frac{\epsilon^2}{96C_7\sigma^2}\right\}$,
$\gamma_1\leq\frac{\epsilon^2}{12(\frac{C_8}{|\B_1^t|}+C_9\sigma^2)}$, $\beta_1=1-\gamma_1$,
$\eta_1^2\leq \min\left\{\frac{\gamma_1^2}{64L_F^2},\frac{|\B_1^t|^2\eta_0^2\mu_L^2}{128n^2C_5C_\lambda^2},\frac{|\B_1^t|^2\gamma_0^2}{128n^2C_6 C_g^2},\frac{|\B_1^t|^2\gamma_0'^2}{512n^2C_7L_{L\lambda\lambda}^2(1+C_\lambda^2)}\right\}$,
$T\geq \left\{\frac{30\Delta_F}{\eta_1 \epsilon^2}, \frac{15\E[\|\nabla F(\w_0)-\m_1\|^2]}{\gamma_1\epsilon^2},\frac{30C_5\delta_{\lambda,0}}{|\B_1^t|\eta_0\mu_L\epsilon^2},\frac{30C_6\delta_{g,0}}{|\B_1^t|\gamma_0\epsilon^2} ,\frac{60C_7\delta_{L\lambda\lambda,0}}{|\B_1^t|\gamma_0'\epsilon^2} \right\}$,
we have
\begin{equation*}
    \E[\|\nabla F(\w_\tau)\|^2]\leq \epsilon^2,\quad  \E[\|\nabla F(\w_\tau)-\m_{\tau+1})\|^2]<2\epsilon^2,
\end{equation*}
where $\tau$ is randomly sampled from $\{0,\dots,T\}$, $C_5,C_6,C_7,C_8$ are constants defined in the proof, and $L_F$ is the Lipschitz continuity constant of $\nabla F(\w)$.
\end{thm}

\subsection{Proof of Theorem~\ref{thm_ana}}

To prove Theorem~\ref{thm_ana}, we first present some required Lemmas.

\begin{lemma}\label{lemma_Fsmooth}
Under Assumption~\ref{assump_ana}, $F(\w)$ is $L_F$-smooth for some constant $L_F\in \R$.
\end{lemma}

\begin{lemma}\label{lemma_1}
Consider the update $\w_{t+1}= \w_t-\eta_1 \m_{t+1}$. Then under Assumption~\ref{assump_ana}, with $ \eta_1 L_F\leq \frac{1}{2}$, we have
\begin{equation*}
F(\w_{t+1})\leq F(\w_t)+\frac{\eta_{w}}{2}\|\nabla F(\w_t)-\m_{t+1}\|^2-\frac{\eta_1}{2}\|\nabla F(\w_t)\|^2-\frac{\eta_1}{4}\|\m_{t+1}\|^2.
\end{equation*}
\end{lemma}

\begin{lemma}[Lemma 4.3  \citet{lin2019gradient}]\label{lemma_lambLip}
Under Assumption~\ref{assump_ana}, $\lambda_i(\w)$ is $C_\lambda$-Lipschitz continuous with $C_\lambda=L_L/\mu_L$ for all $i=1,\dots,n$.
\end{lemma}

\begin{lemma}\label{lem_y_M_new}
Consider the updates in Algorithm~\ref{algo_ana2}, under Assumption~\ref{assump_ana}, with $\eta_0\leq \min\{\mu_L/L_L^2,\frac{2n}{|\B_1^t|\mu_L}\}$ we have
\begin{equation}
    \begin{aligned}
    \sum_{t=0}^T\E[\delta_{\lambda,t}] \leq \frac{2n}{|\B_1^t|\eta_0\mu_L}\delta_{\lambda,0}+\frac{4n\eta_0 T\sigma^2}{\mu_L} +\frac{8n^3C_\lambda^2\eta_1^2}{|\B_1^t|^2\eta_0^2\mu_L^2}\sum_{t=0}^{T-1}\E[\|\m_{t+1}\|^2]\\
    \end{aligned}
\end{equation}
\end{lemma}

\begin{lemma}\label{lemma_u}
Consider Algorithm~\ref{algo_ana2}, under Assumption~\ref{assump_ana}, with $\gamma_0<1/2$ we have
\begin{equation}
    \begin{aligned}
    &\sum_{t=0}^T\E[\delta_{g,t}]\leq \frac{2n}{|\B_1^t|\gamma_0} \delta_{g,0}+8n\gamma_0\sigma^2T+\frac{8n^3 C_g^2\eta_1^2}{|\B_1^t|^2\gamma_0^2}\sum_{t=0}^{T-1}\E[\left\|\m_{t+1} \right\|^2]
    \end{aligned}
\end{equation}
\end{lemma}


\begin{lemma}\label{MA_s_multi_tasks}
Consider Algorithm~\ref{algo_ana2}, under Assumption~\ref{assump_ana}, with $\gamma_0'\leq 1$ we have
\begin{equation*}
    \begin{aligned}
    & \sum_{t=0}^T\E[\delta_{L\lambda\lambda,t}]\leq \frac{4n}{|\B_1^t|\gamma_0'}\delta_{L\lambda\lambda,0} +32L_{L\lambda\lambda}^2\sum_{t=0}^{T-1}\E[\delta_{\lambda,t}] +8n\gamma_0'T\sigma^2 +\frac{32n^3L_{L\lambda\lambda}^2(1+C_\lambda^2)\eta_1^2}{|\B_1^t|^2\gamma_0'^2}\sum_{t=0}^{T-1}\E[\|\m_{t+1}\|^2]
    \end{aligned}
\end{equation*}
\end{lemma}

\begin{proof}[Proof of Theorem~\ref{thm_ana}]

First, recall and define the following definitions
\begin{equation*}
    \begin{aligned}
    &\nabla F(\w_t):=\frac{1}{n}\sum_{i\in \S}\bigg[\nabla_\w\psi_i(\w_t,\lambda_i(\w_t))-\nabla_{\w \lambda}^2 L_i(\w_t,\lambda_i(\w_t)) [\nabla^2_{\lambda\lambda}L_i(\w_t,\lambda_i(\w_t))]^{-1}\nabla_\lambda \psi_i(\w_t,\lambda_i(\w_t))\bigg]f_i(g_i(\w_t))\\
    &\quad\quad\quad\quad\quad\quad\quad\quad +\psi_i(\w_t,\lambda_i(\w_t))\nabla g_i(\w_t)\nabla f_i(g_i(\w_t))\\
    &\nabla F(\w_t,\lambda^t):=\frac{1}{n}\sum_{i\in 
    \S}\nabla F_i(\w_t,\lambda^t)\\
    &\quad\quad\quad\quad\quad:= \frac{1}{n}\sum_{i\in \S} \bigg[\nabla_\w\psi_i(\w_t,\lambda_i^t)-\nabla_{\w \lambda}^2 L_i(\w_t,\lambda_i^t) [s_i^t]^{-1} \nabla_\lambda \psi_i(\w_t,\lambda_i^t)\bigg]f_i(u_i^t)\\
    &\quad\quad\quad\quad\quad\quad\quad\quad +\psi_i(\w_t,\lambda_i^t)\nabla g_i(\w_t)\nabla f_i(u_i^t)\\
    &G(\w_t)=\frac{1}{|\B_1^t|}\sum_{i\in \B_1^t} \bigg[\nabla_\w\psi_i(\w_t,\lambda_i^t)-\nabla_{\w \lambda}^2 L_i(\w_t,\lambda_i^t;\B_{2,i}^t) [s_i^{t}]^{-1} \nabla_\lambda \psi_i(\w_t,\lambda_i^t)\bigg]f_i(  u_i^t) +\psi_i(\w_t,\lambda_i^t)\nabla g_i(\w_t;\B_{2,i}^t)\nabla f_i( u_i^t)
    \end{aligned}
\end{equation*}
Considering the update $\m_{t+1}=(1-\gamma_1)\m_t+\gamma_1 G(\w_t)$ in Algorithm~\ref{algo_ana2}, where $\gamma_1 = 1-\beta_1$, we have
\begin{equation}\label{ineq111}
    \begin{aligned}
        &\E_t[\|\nabla F(\w_t)-\m_{t+1}\|^2]\\
        &=\E_t[\|\nabla F(\w_t)-(1-\gamma_1)\m_t-\gamma_1 G(\w_t)\|^2]\\
        &=\E_t[\|(1-\gamma_1)(\nabla F(\w_{t-1})-\m_t)+(1-\gamma_1)(\nabla F(\w_t)-\nabla F(\w_{t-1})) +\gamma_1(\nabla F(\w_t)-\nabla F(\w_t,\lambda^t))\\
        &\quad+\gamma_1(\nabla F(\w_t,\lambda^t)-G(\w_t))\|^2]\\
        &\stackrel{(a)}{=} \E_t[\|(1-\gamma_1)(\nabla F(\w_{t-1})-\m_t)+(1-\gamma_1)(\nabla F(\w_t)-\nabla F(\w_{t-1})) +\gamma_1(\nabla F(\w_t)-\nabla F(\w_t,\lambda^t))\|^2\\
        &\quad+\|\gamma_1(\nabla F(\w_t,\lambda^t)-G(\w_t))\|^2]\\
        &\stackrel{(b)}{\leq}(1+\gamma_1)(1-\gamma_1)^2\|\nabla F(\w_{t-1})-\m_t\|^2+2(1+\frac{1}{\gamma_1})\bigg[\|\nabla F(\w_t)-\nabla F(\w_{t-1})\|^2 +\gamma_1^2\|\nabla F(\w_t)-\nabla F(\w_t,\lambda^t)\|^2\bigg] \\
        &\quad+\gamma_1^2 \E_t[\|\nabla F(\w_t,\lambda^t)-G(\w_t)\|^2]\\
        &\leq (1-\gamma_1)\|\nabla F(\w_{t-1})-\m_t\|^2+2(1+\frac{1}{\gamma_1})\bigg[L_F^2\|\w_t-\w_{t-1}\|^2 +\gamma_1^2\underbrace{\|\nabla F(\w_t)-\nabla F(\w_t,\lambda^t)\|^2}_{\text{\textcircled{a}}}\bigg] \\
        &\quad+\gamma_1^2 \underbrace{\E_t[\|\nabla F(\w_t,\lambda^t)-G(\w_t)\|^2]}_{\text{\textcircled{b}}}
    \end{aligned}
\end{equation}

where the $(a)$ follows from $\E_t[\widehat{\nabla} F(\w_t,\lambda^t)]=G(\w_t)$, and $(b)$ is due to $\|a+b\|^2\leq (1+\beta)\|a\|^2+(1+\frac{1}{\beta})\|b\|^2$. Furthermore, one may bound the last two terms in (\ref{ineq111}) as following 

\begin{equation*}
    \begin{aligned}
    \text{\textcircled{a}}&=\E_t[\|\nabla F(\w_t)-\nabla F(\w_t,\lambda^t)\|^2]\\
    &\leq \frac{1}{n}\sum_{i\in \S} 6\|\nabla_\w\psi_i(\w_t,\lambda_i(\w_t))[f_i(g_i(\w_t))-f_i(u_i^t)]\|^2+6\|[\nabla_\w\psi_i(\w_t,\lambda_i(\w_t))-\nabla_\w\psi_i(\w_t,\lambda_i^t)]f_i(u_i^t)\|^2\\
    &\quad +12\|[\nabla_{\w \lambda}^2 L_i(\w_t,\lambda_i(\w_t))-\nabla_{\w \lambda}^2 L_i(\w_t,\lambda_i^t)] [\nabla^2_{\lambda\lambda}L_i(\w_t,\lambda_i(\w_t))]^{-1}\nabla_\lambda \psi_i(\w_t,\lambda_i(\w_t))f_i(g_i(\w_t))\|^2\\
    &\quad +12\|\nabla_{\w \lambda}^2 L_i(\w_t,\lambda_i^t) [\nabla^2_{\lambda\lambda}L_i(\w_t,\lambda_i(\w_t))]^{-1}[\nabla_\lambda \psi_i(\w_t,\lambda_i(\w_t))-\nabla_\lambda \psi_i(\w_t,\lambda_i^t)]f_i(g_i(\w_t))\|^2\\
    &\quad +12\|\nabla_{\w \lambda}^2 L_i(\w_t,\lambda_i^t) [\nabla^2_{\lambda\lambda}L_i(\w_t,\lambda_i(\w_t))]^{-1} \nabla_\lambda \psi_i(\w_t,\lambda_i^t)[f_i(g_i(\w_t))-f_i(u_i^t)]\|^2\\
    &\quad +12\|\nabla_{\w \lambda}^2 L_i(\w_t,\lambda_i^t) \big[[\nabla^2_{\lambda\lambda}L_i(\w_t,\lambda_i(\w_t))]^{-1}-[s_i^t]^{-1}\big] \nabla_\lambda \psi_i(\w_t,\lambda_i^t)f_i(u_i^t)\|^2\\
    &\quad +6\|[\psi_i(\w_t,\lambda_i(\w_t))-\psi_i(\w_t,\lambda_i^t)]\nabla g_i(\w_t)\nabla f_i(g_i(\w_i))\|^2\\
    &\quad +6\|\psi_i(\w_t,\lambda_i^t)\nabla g_i(\w_t)[\nabla f_i(g_i(\w_i))-\nabla f_i(u_i^t)]\|^2\\
    &\leq \left(\frac{6C_\psi^2C_f^2}{n}+\frac{12C_{L\w\lambda}^2C_\psi^2C_f^2}{\mu_L^2n}+\frac{6B_\psi^2C_g^2L_f^2}{n}\right)\|g(\w_t)-u^t\|^2+\frac{12C_{L\w\lambda}^2C_\psi^2 B_f^2}{\mu_L^2\gamma^2n}\|\nabla^2_{\lambda\lambda}L(\w_t,\lambda(\w_t))-s^t\|^2\\
    &\quad +\left(\frac{6L_\psi^2B_f^2}{n}+\frac{12L_{L\w\lambda}^2C_\psi^2B_f^2}{\mu_L^2n}+\frac{12C_{L\w\lambda}^2L_\psi^2B_f^2}{\mu_L^2n}+\frac{6C_g^2C_f^2}{n}\right) \|\lambda(\w_t)-\lambda^t\|^2\\
    &=:\frac{C_6}{4n}\delta_{g,t}+\frac{C_7}{4n}\delta_{L\lambda\lambda,t}+\frac{\widetilde{C}_5}{4n}\delta_{\lambda,t}
    \end{aligned}
\end{equation*}
with some properly chosen constants $\widetilde{C}_5, C_6, C_7$.

On the other hand, part $\text{\textcircled{b}}$ can be bounded by some constant,
\begin{equation*}
    \begin{aligned}
    \text{\textcircled{b}}&=\E_t[\|\nabla F(\w_t,\lambda^t)-G(\w_t)\|^2]\\
    &\leq \E_t\left[2\left\|\frac{1}{n}\sum_{i\in \S} \nabla F_i(\w_t,\lambda^t)-\frac{1}{|\B_1^t|}\sum_{i\in \B_1^t} \nabla F_i(\w_t,\lambda^t)\right\|^2+2\left\|\frac{1}{|\B_1^t|}\sum_{i\in \B_1^t} \nabla F_i(\w_t,\lambda^t)-\frac{1}{|\B_1^t|}\sum_{i\in \B_1^t}G_i(\w_t)\right\|^2\right]\\
    &\leq \frac{12C_\psi^2B_f^2+\frac{12C_{L\w\lambda}^2C_\psi^2B_f^2}{\gamma^2}+12B_\psi^2C_g^2C_f^2}{|\B_1^t|} \\
    &\quad +2\E_t\bigg[\frac{1}{|\B_1^t|}\sum_{i\in \B_1^t} \left\|[\nabla_{\w \lambda}^2 L_i(\w_t,\lambda_i^t)  -\nabla_{\w \lambda}^2 L_i(\w_t,\lambda_i^t;\B_{2,i}^t)] [s_i^{t}]^{-1} \nabla_\lambda \psi_i(\w_t,\lambda_i^t)f_i(u_i^t)\right\|^2 \\
    &\quad +\left\|\psi_i(\w_t,\lambda_i^t)[\nabla g_i(\w_t)-\nabla g_i(\w_t;\B_{2,i}^t)]\nabla f_i( u_i^t)\right\|^2\bigg]\\
    &\leq \frac{12C_\psi^2B_f^2+\frac{12C_{L\w\lambda}^2C_\psi^2B_f^2}{\gamma^2}+12B_\psi^2C_g^2C_f^2}{|\B_1^t|} +\frac{C_\psi^2B_f^2\sigma^2}{\gamma^2}+B_\psi^2C_f^2 \sigma^2 =:\frac{C_8}{|\B_1^t|}+C_9\sigma^2
    \end{aligned}
\end{equation*}

Thus, with the natural assumption $\gamma_1\leq 1$, we have
\begin{equation}
    \begin{aligned}
        &\E_t[\|\nabla F(\w_t)-\m_{t+1}\|^2]\\
        &\leq (1-\gamma_1)\|\nabla F(\w_{t-1})-\m_t\|^2+\frac{4}{\gamma_1}\bigg[L_F^2\eta_1^2\|\m_{t-1}\|^2 +\gamma_1^2 \frac{\widetilde{C}_5}{4n}\delta_{\lambda,t} +\gamma_1^2 \frac{C_6}{4n}\delta_{g,t}+\gamma_1^2 \frac{C_7}{4n}\delta_{L\lambda\lambda,t}\bigg]+\gamma_1^2 (\frac{C_8}{|\B_1^t|}+C_9\sigma^2)
    \end{aligned}
\end{equation}


Take expectation over all randomness and summation over $t=1,\dots,T$ to get

\begin{equation}\label{ineq:130}
    \begin{aligned}
        \sum_{t=0}^T\E[\|\nabla F(\w_t)-\m_{t+1}\|^2]
        &\leq \frac{1}{\gamma_1}\E[\|\nabla F(\w_0)-\m_1\|^2]+\frac{4L_F^2\eta_1^2}{\gamma_1^2}\sum_{t=1}^T\E[\|\m_t\|^2] + \frac{\widetilde{C}_5}{n}\sum_{t=1}^T\E[\delta_{\lambda,t}]\\
        &\quad + \frac{C_6}{n}\sum_{t=1}^T\E[\delta_{g,t}]+ \frac{C_7}{n}\sum_{t=1}^T\E[\delta_{L\lambda\lambda,t}]+\gamma_1 (\frac{C_8}{|\B_1^t|}+C_9\sigma^2)T
    \end{aligned}
\end{equation}

Recall that from Lemma~\ref{lem_y_M_new} Lemma~\ref{lemma_u}, and Lemma~\ref{MA_s_multi_tasks} we have bounds for $\sum_{t=0}^T\E[\delta_{\lambda,t}]$, $\sum_{t=0}^T\E[\delta_{g,t}]$ and $\sum_{t=0}^T\E[\delta_{L\lambda\lambda,t}]$,
\begin{equation}\label{delta_lamb}
    \begin{aligned}
    \sum_{t=0}^T\E[\delta_{\lambda,t}] \leq \frac{2n}{|\B_1^t|\eta_0\mu_L}\delta_{\lambda,0}+\frac{4n\eta_0 T\sigma^2}{\mu_L} +\frac{8n^3C_\lambda^2\eta_1^2}{|\B_1^t|^2\eta_0^2\mu_L^2}\sum_{t=0}^{T-1}\E[\|\m_{t+1}\|^2]\\
    \end{aligned}
\end{equation}

\begin{equation}\label{delta_u}
    \begin{aligned}
    \sum_{t=0}^T\E[\delta_{g,t}]\leq \frac{2n}{|\B_1^t|\gamma_0} \delta_{g,0}+8n\gamma_0\sigma^2T+\frac{8n^3 C_g^2\eta_1^2}{|\B_1^t|^2\gamma_0^2}\sum_{t=0}^{T-1}\E[\left\|\m_{t+1} \right\|^2]
    \end{aligned}
\end{equation}

\begin{equation}\label{delta_Llamlam}
    \begin{aligned}
    \sum_{t=0}^T\E[\delta_{L\lambda\lambda,t}]\leq \frac{4n}{|\B_1^t|\gamma_0'}\delta_{L\lambda\lambda,0} +32L_{L\lambda\lambda}^2\sum_{t=0}^{T-1}\E[\delta_{\lambda,t}] +8n\gamma_0'T\sigma^2 +\frac{32n^3L_{L\lambda\lambda}^2(1+C_\lambda^2)\eta_1^2}{|\B_1^t|^2\gamma_0'^2}\sum_{t=0}^{T-1}\E[\|\m_{t+1}\|^2]
    \end{aligned}
\end{equation}

By plugging (\ref{delta_lamb}), (\ref{delta_u}) and (\ref{delta_Llamlam}) into inequality~(\ref{ineq:130}), we obtain
\begin{equation}\label{ineq104}
    \begin{aligned}
        &\sum_{t=0}^T\E[\|\nabla F(\w_t)-\m_{t+1}\|^2]\\
        &\leq\frac{1}{\gamma_1}\E[\|\nabla F(\w_0)-\m_1\|^2]+\frac{4L_F^2\eta_1^2}{\gamma_1^2}\sum_{t=1}^T\E[\|\m_t\|^2]+\frac{C_5}{n} \sum_{t=0}^T\E[\delta_{\lambda,t}] +\frac{C_6}{n}\sum_{t=0}^T\E[\delta_{g,t}]\\
        &\quad +\frac{C_7}{n}\left[\frac{4n}{|\B_1^t|\gamma_0'}\delta_{L\lambda\lambda,0}  +8n\gamma_0'T\sigma^2 +\frac{32n^3L_{L\lambda\lambda}^2(1+C_\lambda^2)\eta_1^2}{|\B_1^t|^2\gamma_0'^2}\sum_{t=0}^{T-1}\E[\|\m_{t+1}\|^2]\right] +\gamma_1 (\frac{C_8}{|\B_1^t|}+C_9\sigma^2)T\\
        &\leq \frac{1}{\gamma_1}\E[\|\nabla F(\w_0)-\m_1\|^2]+\frac{2C_5}{|\B_1^t|\eta_0\mu_L}\delta_{\lambda,0}+\frac{4C_5\eta_0 T\sigma^2}{\mu_L}+\frac{2C_6}{|\B_1^t|\gamma_0} \delta_{g,0}+8C_6\gamma_0\sigma^2T\\
        &\quad +\frac{4C_7}{|\B_1^t|\gamma_0'}\delta_{L\lambda\lambda,0}  +8C_7\gamma_0'T\sigma^2+\gamma_1 (\frac{C_8}{|\B_1^t|}+C_9\sigma^2)T\\
        &\quad +\left[\frac{4L_F^2\eta_1^2}{\gamma_1^2}+\frac{8n^2C_5C_\lambda^2\eta_1^2}{|\B_1^t|^2\eta_0^2\mu_L^2}+\frac{8n^2C_6 C_g^2\eta_1^2}{|\B_1^t|^2\gamma_0^2}+\frac{32n^2C_7L_{L\lambda\lambda}^2(1+C_\lambda^2)\eta_1^2}{|\B_1^t|^2\gamma_0'^2}\right]\sum_{t=1}^T\E[\|\m_t\|^2]
    \end{aligned}
\end{equation}
where $C_5:=32L_{L_\lambda\lambda}^2C_7+\widetilde{C}_5$.

Recall Lemma~\ref{lemma_1}, we have
\begin{equation*}
F(\w_{t+1})\leq F(\w_t)+\frac{\eta_{w}}{2}\|\nabla F(\w_t)-\m_{t+1}\|^2-\frac{\eta_1}{2}\|\nabla F(\w_t)\|^2-\frac{\eta_1}{4}\|\m_{t+1}\|^2.
\end{equation*}

Combing with~(\ref{ineq104}), we obtain
\begin{equation}\label{ineq105}
    \begin{aligned}
        & \frac{1}{T+1}\sum_{t=0}^T\E[\|\nabla F(\x_t)\|^2]\\
        &\leq \frac{2\E[F(\w_0)-F(\w_{T+1})]}{\eta_1 T}+\frac{1}{T}\sum_{t=0}^T\E[\|\nabla F(\w_t)-\m_{t+1}\|^2]-\frac{1}{2T}\sum_{t=0}^T\E[\|\m_{t+1}\|^2]\\
        &\leq \frac{2[F(\w_0)-F(\w^*)]}{\eta_1 T}+ \frac{1}{T}\Bigg[\frac{\E[\|\nabla F(\w_0)-\m_1\|^2]}{\gamma_1}+\frac{2C_5\delta_{\lambda,0}}{|\B_1^t|\eta_0\mu_L}+\frac{2C_6\delta_{g,0}}{|\B_1^t|\gamma_0} +\frac{4C_7\delta_{L\lambda\lambda,0}}{|\B_1^t|\gamma_0'} \Bigg]\\
        &\quad  +\frac{4C_5\eta_0 \sigma^2}{\mu_L}+8C_6\gamma_0\sigma^2+8C_7\gamma_0'\sigma^2+\gamma_1 (\frac{C_8}{|\B_1^t|}+C_9\sigma^2)\\
        &\quad +\frac{1}{T}\left[\frac{4L_F^2\eta_1^2}{\gamma_1^2}+\frac{8n^2C_5C_\lambda^2\eta_1^2}{|\B_1^t|^2\eta_0^2\mu_L^2}+\frac{8n^2C_6 C_g^2\eta_1^2}{|\B_1^t|^2\gamma_0^2}+\frac{32n^2C_7L_{L\lambda\lambda}^2(1+C_\lambda^2)\eta_1^2}{|\B_1^t|^2\gamma_0'^2}-\frac{1}{2}\right]\sum_{t=1}^T\E[\|\m_t\|^2]\\
    \end{aligned}
\end{equation}

By setting
\begin{gather*}
    \eta_1^2\leq \min\left\{\frac{\gamma_1^2}{64L_F^2},\frac{|\B_1^t|^2\eta_0^2\mu_L^2}{128n^2C_5C_\lambda^2},\frac{|\B_1^t|^2\gamma_0^2}{128n^2C_6 C_g^2},\frac{|\B_1^t|^2\gamma_0'^2}{512n^2C_7L_{L\lambda\lambda}^2(1+C_\lambda^2)}\right\}
\end{gather*}
we have
\begin{gather*}
    \frac{4L_F^2\eta_1^2}{\gamma_1^2}+\frac{8n^2C_5C_\lambda^2\eta_1^2}{|\B_1^t|^2\eta_0^2\mu_L^2}+\frac{8n^2C_6 C_g^2\eta_1^2}{|\B_1^t|^2\gamma_0^2}+\frac{32n^2C_7L_{L\lambda\lambda}^2(1+C_\lambda^2)\eta_1^2}{|\B_1^t|^2\gamma_0'^2}-\frac{1}{4}\leq 0
\end{gather*}
which implies that the last term of the RHS of inequality~(\ref{ineq105}) are less or equal to zero. Hence
\begin{equation}
    \begin{aligned}
        & \frac{1}{T+1}\sum_{t=0}^T\E[\|\nabla F(\x_t)\|^2]\\
        &\leq  \frac{2[F(\w_0)-F(\w^*)]}{\eta_1 T}+ \frac{1}{T}\Bigg[\frac{\E[\|\nabla F(\w_0)-\m_1\|^2]}{\gamma_1}+\frac{2C_5\delta_{\lambda,0}}{|\B_1^t|\eta_0\mu_L}+\frac{2C_6\delta_{g,0}}{|\B_1^t|\gamma_0} +\frac{4C_7\delta_{L\lambda\lambda,0}}{|\B_1^t|\gamma_0'} \Bigg]\\
        &\quad  +\frac{4C_5\eta_0 \sigma^2}{\mu_L}+8C_6\gamma_0\sigma^2+8C_7\gamma_0'\sigma^2+\gamma_1 (\frac{C_8}{|\B_1^t|}+C_9\sigma^2)
    \end{aligned}
\end{equation}
With
\begin{gather*}
    \eta_0\leq \frac{\mu_L\epsilon^2}{48C_5 \sigma^2},\gamma_0\leq \frac{\epsilon^2}{96C_6\sigma^2},\gamma_0'\leq \frac{\epsilon^2}{96C_7\sigma^2},\gamma_1\leq\frac{\epsilon^2}{12(\frac{C_8}{|\B_1^t|}+C_9\sigma^2)},\\
    T\geq \left\{\frac{30[F(\w_0)-F(\w^*)]}{\eta_1 \epsilon^2}, \frac{15\E[\|\nabla F(\w_0)-\m_1\|^2]}{\gamma_1\epsilon^2},\frac{30C_5\delta_{\lambda,0}}{|\B_1^t|\eta_0\mu_L\epsilon^2},\frac{30C_6\delta_{g,0}}{|\B_1^t|\gamma_0\epsilon^2} ,\frac{60C_7\delta_{L\lambda\lambda,0}}{|\B_1^t|\gamma_0'\epsilon^2} \right\}
\end{gather*}

we have
\begin{equation*}
     \frac{1}{T+1}\sum_{t=0}^T\E[\|\nabla F(\x_t)\|^2]\leq \frac{1}{3}\epsilon^2+\frac{1}{3}\epsilon^2 < \epsilon^2.
\end{equation*}
Furthermore, to show the second part of the theorem, following from inequality~(\ref{ineq104}), we have
\begin{equation*}
    \begin{aligned}
        \sum_{t=0}^T \E[\|\nabla F(\w_t)-\m_{t+1}\|^2]
        &\leq \frac{1}{\gamma_1}\E[\|\nabla F(\w_0)-\m_1\|^2]+\frac{2C_5}{|\B_1^t|\eta_0\mu_L}\delta_{\lambda,0}+\frac{4C_5\eta_0 T\sigma^2}{\mu_L}+\frac{2C_6}{|\B_1^t|\gamma_0} \delta_{g,0}\\
        &\quad+8C_6\gamma_0\sigma^2T +\frac{4C_7}{|\B_1^t|\gamma_0'}\delta_{L\lambda\lambda,0}  +8C_7\gamma_0'T\sigma^2+\gamma_1 (\frac{C_8}{|\B_1^t|}+C_9\sigma^2)T\\
        &\quad+\frac{1}{2}\sum_{t=0}^{T-1}\E[\|\nabla F(\w_t)\|^2+\|\nabla F(\w_t)-\m_{t+1}\|^2.
    \end{aligned}
\end{equation*}
With parameters set above, it follows that
\begin{equation*}
    \frac{1}{T+1}\sum_{t=0}^T\E[\|\nabla F(\w_t)-\m_{t+1}\|^2]<2\epsilon^2.
\end{equation*}
\end{proof}

\subsection{Proofs of Lemmas}
\subsubsection{Proof of Lemma~\ref{lemma:1}}
\begin{proof}
Given $\ell(\w; \x', \x, q)\geq \I(h_q(\x'; \w)- h_q(\x; \w)\geq 0)$, we have
\begin{equation*}
    \bar{g}(\w;\x^q_i,\S_q)\geq r(\w;\x^q_i,\S_q),
\end{equation*}
for each $(q,\x_i^q)$, which immediately follows the desired conclusion.
\end{proof}

\subsubsection{Proof of Lemma~\ref{lemma:2}}
\begin{proof}
To show the equivalence in the Lemma, it suffices to show that $\lambda_q(\w)$ is the $(K+1)$-th largest value in the set $\{h_q(\x'; \w)|\x'\in \S_q\}$. Let $\{\theta_1, \theta_2, \cdots, \theta_{N_q}\}$ denote a sequence of values defined by sorting $\{h_q(\x'; \w)|\x'\in \S_q\}$ in descending order, i.e., $\theta_1\geq\theta_2\geq\cdots\geq\theta_{N_q}$. $\theta_k$ denote the $k$-th largest value. 

Recall
\begin{equation*}
\lambda_q(\w) = \arg\min_{\lambda}(K+\varepsilon)\lambda +\sum_{\x'\in\S_q}(h_q(\x'; \w) -\lambda)_+,
\end{equation*}
where $\varepsilon\in(0,1)$. Define function $\Lambda_q(\lambda):=(K+\varepsilon)\lambda +\sum_{i=1}^{N_q}(\theta_i -\lambda)_+$, then it follows that $\lambda_q(\w)=\arg\min_\lambda \Lambda_q(\lambda)$. Take the derivative of $\Lambda_q(\lambda)$, we have
\begin{align*}
    \nabla_{\lambda}\Lambda_q(\lambda)=K+\varepsilon-\sum_{i=1}^{N_q}d(\theta_i -\lambda),\ \text{where}\
    d(\theta_i -\lambda)=\begin{cases}
            1, &\theta_i >\lambda\\
            \varepsilon^{'} \in[0,1], &\theta_i =\lambda\\
            0, &\theta_i <\lambda
            \end{cases}.
\end{align*}

First, we assume $\theta_K > \theta_{K+1}$. One may consider this problem in three cases.
\begin{itemize}
    \item If $\lambda>\theta_{K+1}$, then $\sum_{i=1}^{N_q}d(\theta_i -\lambda)\leq K$, so we have $\nabla_{\lambda}\Lambda_q(\lambda)\geq K+\varepsilon-K=\varepsilon>0$.
    \item If $\lambda<\theta_{K+1}$, then $\sum_{i=1}^{N_q}d(\theta_i -\lambda)\geq K+1$, so we have $\nabla_{\lambda}\Lambda_q(\lambda)\leq K+\varepsilon-K-1=\varepsilon-1<0$.
    \item If $\lambda=\theta_{K+1}$, then $\sum_{i=1}^{N_q}d(\theta_i -\lambda)= K+\varepsilon^{'}$, so we have $\nabla_{\lambda}\Lambda_q(\lambda)= K+\varepsilon-K-\varepsilon^{'}=\varepsilon-\varepsilon^{'}$. Thus we will have $\nabla_{\lambda}\Lambda_q(\lambda)=0$ by setting $\varepsilon^{'}=\varepsilon$. Hence $\lambda_q(\w)=\theta_{K+1}$.
\end{itemize}

Second, if $\theta_K = \theta_{K+1}$. One may consider this problem in three cases.
\begin{itemize}
    \item If $\lambda>\theta_{K+1}$, then $\sum_{i=1}^{N_q}d(\theta_i -\lambda)\leq K-1$, so we have $\nabla_{\lambda}\Lambda_q(\lambda)\geq K+\varepsilon-K+1=1+\epsilon>0$.
    \item If $\lambda<\theta_{K+1}$, then $\sum_{i=1}^{N_q}d(\theta_i -\lambda)\geq K+1$, so we have $\nabla_{\lambda}\Lambda_q(\lambda)\leq K+\varepsilon-K-1<0$.
    \item If $\lambda=\theta_{K}=\theta_{K+1}$, then $\sum_{i=1}^{N_q}d(\theta_i -\lambda)= K-1+2\epsilon^{'}$, so we have $\nabla_{\lambda}\Lambda_q(\lambda)= K+\varepsilon-K+1-2\varepsilon^{'}=1+\varepsilon-2\varepsilon^{'}$. Thus we will have $\nabla_{\lambda}\Lambda_q(\lambda)=0$ by setting $\varepsilon^{'}=(1+\varepsilon)/2$. Hence $\lambda_q(\w)=\theta_{K+1}$.
\end{itemize}
In summary, $\theta_{K+1} =\lambda_q(\w) =\arg\min_\lambda \Lambda_q(\lambda)$. The proof is finished.

\end{proof}

\subsubsection{Proof of Lemma~\ref{lemma:3}}
\begin{proof}
Given the condition $\psi(h_q(\x^q_i; \w)- \lambda_q(\w))\leq C\I(h_q(\x^q_i; \w)- \lambda_q(\w)>0)$ and $\ell(\w; \x', \x, q)\geq \I(h_q(\x'; \w)- h_q(\x; \w)>0)$, we have
\begin{equation*}
    \frac{\psi(h_q(\x^q_i; \w)- \lambda_q(\w))(2^{y^q_i}-1)}{CZ_q^K\log_2(\bar{g}(\w; \x^q_i, \S_q)+1) } \leq  \frac{\I(\x_i^q\in\S_q[K])(2^{y^q_i}-1)}{Z^K_q\log_2(r(\w; \x^q_i, \S_q)+1) }
\end{equation*}
for each $(q,\x_i^q)$. The desired result follows.
\end{proof}

\subsubsection{Proof of Lemma~\ref{lemma:4}}
\begin{proof}
Recall
\begin{equation*}
    L_q(\lambda; \w)= \frac{K}{N_q}\lambda+ \frac{\tau_2}{2}\lambda^2+\frac{1}{N_q}\sum_{\x_i\in\S_q}\tau_1\ln(1+\exp((h_q(\x_i;\w)-\lambda)/\tau_1)).
\end{equation*}
Define
\begin{gather*}
    \tilde{L}_q(\lambda; \w)= \frac{K}{N_q}\lambda+\frac{1}{N_q}\sum_{\x_i\in\S_q}(h_q(\x_i;\w)-\lambda)_+\\
    \hat{L}_q(\lambda; \w)= \frac{K}{N_q}\lambda+ \frac{\tau_2}{2}\lambda^2+\frac{1}{N_q}\sum_{\x_i\in\S_q}(h_q(\x_i;\w)-\lambda)_+.
\end{gather*}
For simplicity, we denote $\lambda_*=\arg\min_{\lambda}L_q(\lambda; \w)$, $ \tilde{\lambda}_*=\arg\min_{\lambda}\tilde{L}_q(\lambda; \w)$, $\hat{\lambda}_*=\arg\min_{\lambda}\hat{L}_q(\lambda; \w)$. Note that it is obvious to see that when $\lambda\geq 2c_h$, function $\tilde{L}_q(\lambda;\w)$ is monotonically increasing, and monotonically decreasing when $\lambda\leq 0$. Thus the optimal point is bounded, i.e. $\tilde{\lambda}_*\in [0,2c_h]$. Similarly, we have $\nabla_\lambda L_q(\lambda;\w)<0$ when $\lambda\leq 0$ and $\nabla_\lambda L_q(\lambda;\w)\geq 0$ when $\lambda\geq c_h+\tau_1 \ln N_{max}$ where $N_{max}=\max_q N_q$. This allows us to show that the optimal point $\lambda_*$ is also bounded, i.e. $\lambda_*\in [0,c_h+\tau_1 \ln N_{max}]$. By applying Lemma 8 in \citet{JMLR:v19:17-016} to $\tilde{L}_q(\lambda;\w)$, we know that there exists a constant $c_1>0$ such that for all $\lambda$ we have
\begin{equation}\label{ineqLL1}
    |\lambda-\lambda_q(\w)|\leq c_1(\tilde{L}_q(\lambda;\w)-\tilde{L}_q(\lambda_q(\w);\w)).
\end{equation}
It is trivial to show $\tau_1\ln(1+\exp(x/\tau_1))\geq x_+ \, \forall x\in \R$ and $\tau_1\ln(1+\exp(x/\tau_1))- x_+\leq (\ln 2)\tau_1$. Then it follows easily that
\begin{equation}\label{ineqLL}
    \hat{L}_q(\lambda;\w)\leq L_q(\lambda;\w) \leq \hat{L}_q(\lambda;\w)+c_2\tau_1
\end{equation}
where $c_2 = \ln 2$. Then with inequality~(\ref{ineqLL}) and the optimality of $\lambda_*$, we have
\begin{equation*}
\begin{aligned}
    \tilde{L}_q(\lambda_*; \w) &=\hat{L}_q(\lambda_*;\w)-\frac{\tau_2}{2}\lambda_*^2\leq L_q(\lambda_*;\w)-\frac{\tau_2}{2}\lambda_*^2\leq  L_q(\tilde{\lambda}_*;\w)-\frac{\tau_2}{2}\lambda_*^2\\
    &\leq \hat{L}_q(\tilde{\lambda}_*;\w)+c_2\tau_1-\frac{\tau_2}{2}\lambda_*^2=\tilde{L}_q(\tilde{\lambda}_*;\w)+\frac{\tau_2}{2}\tilde{\lambda}_*^2+c_2\tau_1-\frac{\tau_2}{2}\lambda_*^2
\end{aligned}
\end{equation*}
which follows that
\begin{equation}\label{ineqLL2}
    |\tilde{L}_q(\lambda_*; \w)-\tilde{L}_q(\tilde{\lambda}_*;\w)|\leq  \frac{\tau_2}{2}\tilde{\lambda}_*^2+c_2\tau_1-\frac{\tau_2}{2}\lambda_*^2
\end{equation}
Combining inequalities~(\ref{ineqLL1}), (\ref{ineqLL2}) and the boundedness of $\lambda_*,\tilde{\lambda}_*$, and setting $\tau_1=\tau_2=\varepsilon$, we obtain 
\begin{equation*}
|\lambda_q(\w)-\hat{\lambda}_q(\w)| \leq c_1\left(\frac{\tau_2}{2}\tilde{\lambda}_*^2+c_2\tau_1-\frac{\tau_2}{2}\lambda_*^2\right)=\mathcal{O}(\varepsilon)
\end{equation*}
To show the smoothness of $L_q(\lambda;\w)$, we first show
\begin{equation}\label{eqLL1}
    \tau_1 \ln (1+\exp(x/\tau_1)) = \max_{\alpha\in [0,1]} x\alpha -\tau_1[\alpha\ln (\alpha)+(1-\alpha)\ln(1-\alpha)]=:A(\alpha)
\end{equation} 
Note that the solution to $A'(\alpha)=x-\tau_1[\ln(\alpha)-\ln(1-\alpha)]=0$ is $\alpha^* = 1-(1+\exp(x/\tau_1))^{-1}$. Then $A(\alpha^*)= \tau_1 \ln (1+\exp(x/\tau_1))$, which implies (\ref{eqLL1}). Moreover, $A(\alpha)$ is strong concave because
\begin{equation*}
    (A(\alpha)+\tau_1\alpha^2)''=-\tau_1(\frac{1}{\alpha}+\frac{1}{1-\alpha})+2\tau_1<0
\end{equation*}
It follows that
\begin{equation*}
    L_q(\lambda; \w)= \frac{K}{N_q}\lambda+ \frac{\tau_2}{2}\lambda^2+ \frac{1}{N_q}\sum_{\x_i\in\S_q}\max_{\alpha\in (0,1)}(h_q(\x_i;\w)-\lambda)\alpha -\tau_1 [\alpha\ln (\alpha)+(1-\alpha)\ln(1-\alpha)].
\end{equation*}
Then by Theorem 1 in \citet{10.1007/s10107-004-0552-5}, $L_q(\lambda;\w)$ is smooth.
The strong convexity of $L_q(\lambda;\w)$ follows from the convexity of $L_q(\lambda;\w)-\frac{\tau_2}{2}\lambda^2$, which can be proved by checking the non-negativity of its second derivative
\begin{gather*}
     \nabla^2 (L_q(\lambda;\w)-\frac{\tau_2}{2}\lambda^2)=\frac{1}{N_q}\sum_{x_i\in \S_q}\frac{\frac{1}{\tau_1}\exp((\lambda-h_q(\x_i;\w))/\tau_1)}{[1+\exp((\lambda-h_q(\x_i;\w))/\tau_1)]^2}\geq 0
\end{gather*}
\end{proof}

\subsubsection{Proof of Lemma~\ref{lemma_Fsmooth}}
\begin{proof}
Take arbitrary $\w_1,\w_2\in \R^d$, we have
\begin{equation*}
    \begin{aligned}
    \|\nabla F(\w_1)-\nabla F(\w_2)\|
    &\leq \frac{1}{n}\sum_{i\in \S} \|\nabla_\w\psi_i(\w_1,\lambda_i(\w_1))f_i(g_i(\w_1))-\nabla_\w\psi_i(\w_2,\lambda_i(\w_2))f_i(g_i(\w_2))\|\\
    &\quad+\frac{1}{n}\sum_{i\in \S}\|\nabla_{\w \lambda}^2 L_i(\w_2,\lambda_i(\w_2)) [\nabla_{\lambda\lambda}^2 L_i(\w_2,\lambda_i(\w_2))]^{-1} \nabla_\lambda \psi_i(\w_2,\lambda_i(\w_2))f_i(g_i(\w_2))\\
    &\quad\quad\quad\quad\quad  -\nabla_{\w \lambda}^2 L_i(\w_1,\lambda_i(\w_1)) [\nabla_{\lambda\lambda}^2 L_i(\w_1,\lambda_i(\w_1))]^{-1} \nabla_\lambda \psi_i(\w_1,\lambda_i(\w_1))f_i(g_i(\w_1))\|\\
    &\quad +\frac{1}{n}\sum_{i\in \S}\|\psi_i(\w_1,\lambda_i(\w_1))\nabla g_i(\w_1)\nabla f_i(g_i(\w_1))-\psi_i(\w_2,\lambda_i(\w_2))\nabla g_i(\w_2)\nabla f_i(g_i(\w_2))\|
    \end{aligned}
\end{equation*}
For each $i$ we have
\begin{equation*}
    \begin{aligned}
    &\|\nabla_\w\psi_i(\w_1,\lambda_i(\w_1))f_i(g_i(\w_1))-\nabla_\w\psi_i(\w_2,\lambda_i(\w_2))f_i(g_i(\w_2))\|^2\\
    &\leq \|\nabla_\w\psi_i(\w_1,\lambda_i(\w_1))[f_i(g_i(\w_1)-f_i(g_i(\w_2))]\|^2+ \|[\nabla_\w\psi_i(\w_1,\lambda_i(\w_1))-\nabla_\w\psi_i(\w_2,\lambda_i(\w_2))]f_i(g_i(\w_2))\|^2\\
    & \leq C_\psi^2 C_f^2 \|g_i(\w_1)-g_i(\w_2)\|^2 + L_\psi^2[\|\w_1-\w_2\|^2+\|\lambda_i(\w_1)-\lambda_i(\w_2)\|^2]B_f^2\\
    &\leq (C_\psi^2 C_f^2C_g^2+B_f^2L_\psi^2(1+C_\lambda))\|\w_1-\w_2\|^2\\
    &=:C_1^2\|\w_1-\w_2\|^2
    \end{aligned}
\end{equation*} 
and 
\begin{equation*}
    \begin{aligned}
    &\|\nabla_{\w \lambda}^2 L_i(\w_2,\lambda_i(\w_2)) [\nabla_{\lambda\lambda}^2 L_i(\w_2,\lambda_i(\w_2))]^{-1} \nabla_\lambda \psi_i(\w_2,\lambda_i(\w_2))f_i(g_i(\w_2))\\
    &  -\nabla_{\w \lambda}^2 L_i(\w_1,\lambda_i(\w_1)) [\nabla_{\lambda\lambda}^2 L_i(\w_1,\lambda_i(\w_1))]^{-1} \nabla_\lambda \psi_i(\w_1,\lambda_i(\w_1))f_i(g_i(\w_1))\|\\
    &\leq 4\|[\nabla_{\w \lambda}^2 L_i(\w_2,\lambda_i(\w_2))-\nabla_{\w \lambda}^2 L_i(\w_1,\lambda_i(\w_1))] [\nabla_{\lambda\lambda}^2 L_i(\w_2,\lambda_i(\w_2))]^{-1} \nabla_\lambda \psi_i(\w_2,\lambda_i(\w_2))f_i(g_i(\w_2))\|^2\\
    &\quad +4\|\nabla_{\w \lambda}^2 L_i(\w_1,\lambda_i(\w_1)) \left[[\nabla_{\lambda\lambda}^2 L_i(\w_2,\lambda_i(\w_2))]^{-1}-[\nabla_{\lambda\lambda}^2 L_i(\w_1,\lambda_i(\w_1))]^{-1}\right] \nabla_\lambda \psi_i(\w_2,\lambda_i(\w_2))f_i(g_i(\w_2))\|^2\\
    &\quad+4\|\nabla_{\w \lambda}^2 L_i(\w_1,\lambda_i(\w_1)) [\nabla_{\lambda\lambda}^2 L_i(\w_1,\lambda_i(\w_1))]^{-1} [\nabla_\lambda \psi_i(\w_2,\lambda_i(\w_2))-\nabla_\lambda \psi_i(\w_1,\lambda_i(\w_1))]f_i(g_i(\w_2))\|^2\\
    &\quad +4\|\nabla_{\w \lambda}^2 L_i(\w_1,\lambda_i(\w_1)) [\nabla_{\lambda\lambda}^2 L_i(\w_1,\lambda_i(\w_1))]^{-1} \nabla_\lambda \psi_i(\w_1,\lambda_i(\w_1))[f_i(g_i(\w_2))-f_i(g_i(\w_1))]\|^2\\
    &\leq \left[\left(\frac{4L_{L\w\lambda}^2C_\psi^2 B_f^2}{\mu_L^2}+\frac{4C_{L\w\lambda}^2L_{L\lambda\lambda}^2C_\psi^2 B_f^2}{\mu_L^4}+\frac{4C_{L\w\lambda}^2L_\psi^2 B_f^2}{\mu_L^2}\right)(1+C_\lambda^2)+\frac{4C_{L\w\lambda}^2C_\psi^2 C_f^2C_g^2}{\mu_L^2}\right] \|\w_1-\w_2\|^2\\
    &=:C_2^2\|\w_1-\w_2\|^2
    \end{aligned}
\end{equation*}
and
\begin{equation*}
    \begin{aligned}
    &\|\psi_i(\w_1,\lambda_i(\w_1))\nabla g_i(\w_1)\nabla f_i(g_i(\w_1))-\psi_i(\w_2,\lambda_i(\w_2))\nabla g_i(\w_2)\nabla f_i(g_i(\w_2))\|^2\\
    &\leq 3\|[\psi_i(\w_1,\lambda_i(\w_1))-\psi_i(\w_2,\lambda_i(\w_2))]\nabla g_i(\w_1)\nabla f_i(g_i(\w_1))\|^2\\
    &\quad +3 \|\psi_i(\w_2,\lambda_i(\w_2))[\nabla g_i(\w_1)-\nabla g_i(\w_2)]\nabla f_i(g_i(\w_1))\|^2\\
    &\quad + 3\|\psi_i(\w_2,\lambda_i(\w_2))\nabla g_i(\w_2)[\nabla f_i(g_i(\w_1))-\nabla f_i(g_i(\w_2))]\|^2\\
    &\leq \left[3C_\psi^2C_g^2 C_f^2(1+C_\lambda^2)+3B_\psi^2 L_g^2C_f^2 +3B_\ell^2 C_g^2 L_f^2C_g^2\right]\|\w_1-\w_2\|^2\\
    &=:C_3^2\|\w_1-\w_2\|^2.
    \end{aligned}
\end{equation*}
Hence
\begin{equation*}
    \begin{aligned}
    \|\nabla F(\w_1)-\nabla F(\w_2)\|&\leq\frac{1}{n}\sum_{i\in S}(C_1+C_2+C_3)\|\w_1-\w_2\| = L_F\|\w_1-\w_2\|,
    \end{aligned}
\end{equation*}
where $L_F:=C_1+C_2+C_3$.
\end{proof}

\subsubsection{Proof of Lemma~\ref{lemma_1}}
\begin{proof}
By $L_F$-smoothness of $F(\w)$, with $\eta_1 \leq \frac{1}{2L_F}$, we have
\begin{align*}
    F(\w_{t+1})&\leq F(\w_t)+\nabla F(\w_t)^T (\w_{t+1}-\w_t)+\frac{L_F}{2}\|\w_{t+1}-\w_t\|^2\\
    &=F(\w_t)-\eta_1 \nabla F(\w_t)^T\m_{t+1}+\frac{L_F}{2}\eta_1^2\|\m_{t+1}\|^2\\
    &=F(\w_t)+\frac{\eta_1}{2}\|\nabla F(\w_t)-\m_{t+1}\|^2-\frac{\eta_1}{2}\|\nabla F(\w_t)\|^2+\left(\frac{L_F }{2}\eta_1^2-\frac{\eta_1}{2}\right) \|\m_{t+1}\|^2.
\end{align*}
\end{proof}

\subsubsection{Proof of Lemma~\ref{lem_y_M_new}}
\begin{proof}
Recall and define the following notations
\begin{equation*}
    \lambda_i^{t+1}=\begin{cases}\lambda_{i}^t -\eta_0 \nabla_\lambda L_i(\w_t,\lambda_i^t;\B_{2,i}^t) \quad & \text{if }i\in \B_1^t\\ \lambda_{i}^t & \text{o.w.}\end{cases},\quad \widetilde{\lambda}_i^t:=\lambda_{i}^t-\eta_0 \nabla_\lambda L_i(\w_t,\lambda_i^t;\B_{2,i}^t) 
\end{equation*}
Then
\begin{equation}\label{ineq:10}
    \begin{aligned}
    & \E_{\B_{2,i}^t}[\|\widetilde{\lambda}_i^t-\lambda_i(\w_t)\|^2]\\
    &=\E_{\B_{2,i}^t}[\|\lambda_{i}^t- \eta_0 \nabla_\lambda L_i(\w_t,\lambda_i^t;\B_{2,i}^t) -\lambda_i(\w_t)\|^2]\\
    &=\E_{\B_{2,i}^t}[\|\lambda_{i}^t- \eta_0 \nabla_\lambda L_i(\w_t,\lambda_i^t;\B_{2,i}^t) -\lambda_i(\w_t)+\eta_0\nabla_\lambda L_i(\w_t,\lambda_i(\w_t))+\eta_0\nabla_\lambda L_i(\w_t,\lambda_i^t)-\eta_0\nabla_\lambda L_i(\w_t,\lambda_i^t)\|^2]\\
    &=\|\lambda_{i}^t-\lambda_i(\w_t)+\eta_0\nabla_\lambda L_i(\w_t,\lambda_i(\w_t))-\eta_0\nabla_\lambda L_i(\w_t,\lambda_i^t)\|^2 +\E_{\B_{2,i}^t}[\|\eta_0\nabla_\lambda L_i(\w_t,\lambda_i^t)- \eta_0 \nabla_\lambda L_i(\w_t,\lambda_i^t;\B_{2,i}^t) \|^2]\\
    &\leq \|\lambda_{i}^t-\lambda_i(\w_t)\|^2+\eta_0^2\|\nabla_\lambda L_i(\w_t,\lambda_i(\w_t))-\nabla_\lambda L_i(\w_t,\lambda_i^t)\|^2\\
    &\quad +2\eta_0\langle \lambda_{i}^t-\lambda_i(\w_t),\nabla_\lambda L_i(\w_t,\lambda_i(\w_t))-\nabla_\lambda L_i(\w_t,\lambda_i^t)\rangle +\eta_0^2 \sigma^2\\
    &\stackrel{(a)}{\leq} \|\lambda_{i}^t-\lambda_i(\w_t)\|^2+\eta_0^2L_L^2\|\lambda_{i}^t-\lambda_i(\w_t)\|^2-2\eta_0\mu_L\|\lambda_{i}^t-\lambda_i(\w_t)\|^2+\eta_0^2 \sigma^2\\
    &\stackrel{(b)}{\leq}(1-\eta_0\mu_L)\|\lambda_{i}^t-\lambda_i(\w_t)\|^2+\eta_0^2 \sigma^2
    \end{aligned}
\end{equation}
where inequality $(a)$ uses the strong monotonicity of $L_i(\w_t,\cdot)$ as it is assumed to be $\mu_L$-strongly convex, and $(b)$ uses the assumption $\eta_0\leq \mu_L/L_L^2$.\\
Moreover, consider the randomness on the query sampling $\B_1^t$, we have
\begin{equation*}
    \E_t[\|\lambda_i^{t+1}-\lambda_i(\w_t)\|^2]=\frac{|\B_1^t|}{n}\E_{\B_{2,i}^t}[\|\widetilde{\lambda}_i^t-\lambda_i(\w_t)\|^2]+\frac{n-|\B_1^t|}{n}\|\lambda_i^t-\lambda_i(\w_t)\|^2
\end{equation*}
which follows 
\begin{equation}\label{ineq:11}
    \E_{\B_{2,i}^t}[\|\widetilde{\lambda}_i^t-\lambda_i(\w_t)\|^2]=\frac{n}{|\B_1^t|}\E_t[\|\lambda_i^{t+1}-\lambda_i(\w_t)\|^2]-\frac{n-|\B_1^t|}{|\B_1^t|}\|\lambda_i^t-\lambda_i(\w_t)\|^2
\end{equation}
Combining inequalities~(\ref{ineq:10}) and (\ref{ineq:11}), we obtain
\begin{equation}
    \begin{aligned}
    \E_t[\|\lambda_i^{t+1}-\lambda_i(\w_t)\|^2]&\leq \frac{n-|\B_1^t|}{n}\|\lambda_i^t-\lambda_i(\w_t)\|^2+\frac{|\B_1^t|}{n}(1-\eta_0\mu_L)\|\lambda_{i}^t-\lambda_i(\w_t)\|^2+\frac{|\B_1^t|}{n}\eta_0^2 \sigma^2\\
    &\leq (1-\frac{|\B_1^t|\eta_0\mu_L}{n})\|\lambda_i^t-\lambda_i(\w_t)\|^2+\frac{|\B_1^t|\eta_0^2 \sigma^2}{n}\\
    \end{aligned}
\end{equation}
Thus
\begin{equation}
    \begin{aligned}
    &\E_t[\|\lambda_i^{t+1}-\lambda_i(\w_{t+1})\|^2]\\
    &\leq (1+\frac{|\B_1^t|\eta_0\mu_L}{2n})\E_t[\|\lambda_i^{t+1}-\lambda_i(\w_t)\|^2]+(1+\frac{2n}{|\B_1^t|\eta_0\mu_L})\E_t[\|\lambda_i(\w_{t+1})-\lambda_i(\w_t)\|^2]\\
    & \leq (1-\frac{|\B_1^t|\eta_0\mu_L}{2n})\|\lambda_i^t-\lambda_i(\w_t)\|^2+\frac{2|\B_1^t|\eta_0^2 \sigma^2}{n}+\frac{4nC_\lambda^2}{|\B_1^t|\eta_0\mu_L}\E_t[\|\w_{t+1}-\w_t\|^2]\\
    \end{aligned}
\end{equation}
where we use the assumption $\eta_0\leq \frac{2n}{|\B_1^t|\mu_L}$ i.e. $\frac{|\B_1^t|\eta_0\mu_L}{2n}\leq 1$.\\
Taking summation over all queries and expectation over all randomness, we have
\begin{equation}
    \begin{aligned}
    \E[\|\lambda^{t+1}-\lambda(\w_{t+1})\|^2]\leq (1-\frac{|\B_1^t|\eta_0\mu_L}{2n})\E[\|\lambda^t-\lambda(\w_t)\|^2]+2|\B_1^t|\eta_0^2 \sigma^2+\frac{4n^2C_\lambda^2}{|\B_1^t|\eta_0\mu_L}\E[\|\w_{t+1}-\w_t\|^2]\\
    \end{aligned}
\end{equation}
Taking summation over $t=0,\dots,T-1$, we have
\begin{equation}
    \begin{aligned}
    \sum_{t=0}^T\E[\|\lambda^t-\lambda(\w_t)\|^2] \leq \frac{2n}{|\B_1^t|\eta_0\mu_L}\|\lambda^0-\lambda(\w_0)\|^2+\frac{4n\eta_0 T\sigma^2}{\mu_L} +\frac{8n^3C_\lambda^2}{|\B_1^t|^2\eta_0^2\mu_L^2}\sum_{t=0}^{T-1}\E[\|\w_{t+1}-\w_t\|^2]\\
    \end{aligned}
\end{equation}
\end{proof}

\subsubsection{Proof of Lemma~\ref{lemma_u}}
\begin{proof}
Consider
\begin{equation}\label{ineq16}
\begin{aligned}
\| u^{t+1}-g(\w_t  )\|^2&= \| u^{t+1}- u^t+ u^t-g(\w_t )\|^2\\
&=\| u^{t+1}- u^t\|^2+\| u^t-g(\w_t)\|^2 +2\langle  u^{t+1}- u^t,  u^t-g(\w_t) \rangle\\
&=\| u^{t+1}- u^t\|^2+\| u^t-g(\w_t)\|^2 +2\sum_{i\in \B_1^t}\langle  u_i^{t+1}- u_i^{t},  u_i^{t}-g_i(\w_t) \rangle\\
&=\| u^{t+1}- u^t\|^2+\| u^t-g(\w_t)\|^2 +\underbrace{2\sum_{i\in \B_1^t}\langle  u_i^{t+1}- u_i^{t},  u_i^{t}-g_i(\w_t;\B_{2,i}^t) \rangle}_{A_5}\\
&\quad +\underbrace{2\sum_{i\in \B_1^t}\langle  u_i^{t+1}- u_i^{t}, g_i(\w_t ;\B_{2,i}^t)-g_i(\w_t ) \rangle}_{A_6}\\
\end{aligned}
\end{equation}
With $ u_i^{t}- u_i^{t+1}=\gamma_0( u_i^{t}-g_i(\w_t;\B_{2,i}^t))\, \forall i\in \B_1^t$ and the inequality $2\langle b-a,a-c\rangle = \|b-c\|^2-\|a-b\|^2-\|a-c\|^2$, we have
\begin{equation}\label{ineq13}
\begin{aligned}
A_5&= 2\sum_{i\in \B_1^t}\langle  u_i^{t+1}-g_i(\w_t ),  u_i^{t}-g_i(\w_t ,\B_{2,i}^t) \rangle+2\sum_{i\in \B_1^t}\langle g_i(\w_t )-  u_i^{t},  u_i^{t}-g_i(\w_t ,\B_{2,i}^t) \rangle\\
&= \frac{2}{\gamma_0}\sum_{i\in \B_1^t}\langle  u_i^{t+1}-g_i(\w_t ),  u_i^{t}- u_i^{t+1} \rangle+2\sum_{i\in \B_1^t}\langle g_i(\w_t )-  u_i^{t},  u_i^{t}-g_i(\w_t ,\B_{2,i}^t) \rangle\\
&=  \frac{1}{\gamma_0}\sum_{i\in \B_1^t}[\| u_i^{t}-g_i(\w_t )\|^2-\| u_i^{t+1}-g_i(\w_t )\|^2-\| u_i^{t+1}- u_i^{t}\|^2]\\
&\quad +2\sum_{i\in \B_1^t}\langle g_i(\w_t )-  u_i^{t},  u_i^{t}-g_i(\w_t ,\B_{2,i}^t) \rangle\\
&=  \frac{1}{\gamma_0}\| u^t-g(\w_t)\|^2-\frac{1}{\gamma_0}\| u^{t+1}-g(\w_t)\|^2-\frac{1}{\gamma_0}\| u^{t+1}- u^t\|^2\\
&\quad +2\sum_{i\in \B_1^t}\langle g_i(\w_t )-  u_i^{t},  u_i^{t}-g_i(\w_t ,\B_{2,i}^t) \rangle
\end{aligned}
\end{equation}
where the last equality is due to the fact $\| u_i^t-g_i(\w_t)\|^2=\| u_i^{t+1}-g_i(\w_t)\|^2$ and $\| u_i^{t+1}- u_i^t\|^2=0$ for all $i\not \in \B_1^t$. Taking expectation over the randomness at iteration $t$ we have
\begin{equation}\label{ineq14}
\begin{aligned}
\E_t[A_5]&\leq  \frac{1}{\gamma_0}\| u^t-g(\w_t)\|^2-\frac{1}{\gamma_0}\E_t[\| u^{t+1}-g(\w_t)\|^2]-\frac{1}{\gamma_0}\E_t[\| u^{t+1}- u^t\|^2]\\
&\quad - 2\E_{\B_1^t} \left[\sum_{i\in \B_1^t}\| u_i^{t}-g_i(\w_t )\|^2\right]\\
&=  \frac{1}{\gamma_0}\| u^t-g(\w_t)\|^2-\frac{1}{\gamma_0}\E_t[\| u^{t+1}-g(\w_t)\|^2]-\frac{1}{\gamma_0}\E_t[\| u^{t+1}- u^t\|^2]\\
&\quad - 2\frac{|\B_1^t|}{n}\| u^t-g(\w_t)\|^2
\end{aligned}
\end{equation}
On the other hand, with the assumption $\gamma_0<1/2$, we have
\begin{equation}\label{ineq15}
\begin{aligned}
A_6&\leq (\frac{1}{\gamma_0}-1)\sum_{i\in \B_1^t}\| u_i^{t+1}- u_i^{t}\|^2+\frac{1}{\frac{1}{\gamma_0}-1}\sum_{i\in \B_1^t}\| g_i(\w_t ;\B_{2,i}^t)-g_i(\w_t )\|^2\\
&\leq (\frac{1}{\gamma_0}-1)\sum_{i\in \B_1^t}\| u_i^{t+1}- u_i^{t}\|^2+2\gamma_0\sum_{i\in \B_1^t}\| g_i(\w_t ;\B_{2,i}^t)-g_i(\w_t )\|^2\\
&\leq (\frac{1}{\gamma_0}-1)\| u^{t+1}- u^t\|^2+2\gamma_0|\B_1^t|\sigma^2
\end{aligned}
\end{equation}

Then by plugging (\ref{ineq13}), (\ref{ineq14}), (\ref{ineq15}) back into (\ref{ineq16}), we obtain
\begin{equation*}
\begin{aligned}
&\E[\| u^{t+1}-g(\w_t)\|^2]\\
&\leq \E[\| u^{t+1}- u^t\|^2]+\E[\| u^t-g(\w_t)\|^2]+\frac{1}{\gamma_0}\E[\| u^t-g(\w_t)\|^2]-\frac{1}{\gamma_0}\E[\| u^{t+1}-g(\w_t)\|^2]\\
&\quad-\frac{1}{\gamma_0}\E[\| u^{t+1}- u^t\|^2] - 2\frac{|\B_1^t|}{n}\E[\| u^t-g(\w_t)\|^2]+ (\frac{1}{\gamma_0}-1)\E[\| u^{t+1}- u^t\|^2]+2\gamma_0|\B_1^t|\sigma^2\\
&=(1+\frac{1}{\gamma_0}- 2\frac{|\B_1^t|}{n})\E[\| u^t-g(\w_t)\|^2]-\frac{1}{\gamma_0}\E[\| u^{t+1}-g(\w_t)\|^2]+ 2\gamma_0|\B_1^t|\sigma^2\\
\end{aligned}
\end{equation*}

Note that $\frac{(1+\frac{1}{\gamma_0}- 2\frac{|\B_1^t|}{n})}{1+\frac{1}{\gamma_0}}=1-\frac{2|\B_1^t|\gamma_0}{(1+\gamma_0)n}\leq 1-\frac{|\B_1^t|\gamma_0}{n}$ and $(1+\frac{a}{2})(1-a)\leq 1-\frac{a}{2}$. It follows
\begin{equation*}
\begin{aligned}
\E[\| u^{t+1}-g(\w_t)\|^2]\leq (1-\frac{|\B_1^t|\gamma_0}{n})\E[\| u^t-g(\w_t)\|^2]+ 2\gamma_0^2|\B_1^t|\sigma^2
\end{aligned}
\end{equation*}

Moreover, we have
\begin{equation*}
    \begin{aligned}
    &\E[\| u^{t+1}-g(\w_{t+1})\|^2]\\
    &\leq \left(1+\frac{|\B_1^t|\gamma_0}{2n}\right)\E[\| u^{t+1}-g(\w_t)\|^2]+\left(1+\frac{2n}{|\B_1^t|\gamma_0}\right)\E[\left\|g(\w_{t})-g(\w_{t+1}) \right\|^2]\\
    &\leq \left(1+\frac{|\B_1^t|\gamma_0}{2n}\right)\left[(1-\frac{|\B_1^t|\gamma_0}{n})\E[\| u^t-g(\w_t)\|^2]+ 2\gamma_0^2|\B_1^t|\sigma^2\right]+\left(1+\frac{2n}{|\B_1^t|\gamma_0}\right)C_g^2n\E[\left\|\w_{t}-\w_{t+1} \right\|^2]\\
    &\leq \left(1-\frac{|\B_1^t|\gamma_0}{2n}\right)\E[\| u^t-g(\w_t)\|^2]+4\gamma_0^2|\B_1^t|\sigma^2+\frac{4n^2C_g^2\eta_1^2}{|\B_1^t|\gamma_0}\E[\left\|\m_{t+1} \right\|^2]
    \end{aligned}
\end{equation*}

Take summation over $t=0,\dots,T-1$ to get
\begin{equation*}
    \begin{aligned}
    &\sum_{t=0}^T\E[\| u^t-g(\w_t)\|^2]\leq \frac{2n}{|\B_1^t|\gamma_0} \E[\| u^0-g(\w_0)\|^2]+8n\gamma_0\sigma^2T+\frac{8n^3 C_g^2\eta_1^2}{|\B_1^t|^2\gamma_0^2}\sum_{t=0}^{T-1}\E[\left\|\m_{t+1} \right\|^2]
    \end{aligned}
\end{equation*}
\end{proof}

\subsubsection{Proof of Lemma~\ref{MA_s_multi_tasks}}
\begin{proof}
Recall and define the following notations
\begin{equation*}
s_i^{t+1}=\begin{cases} (1-\gamma_0')s_i^t+\gamma_0'\nabla^2_{\lambda\lambda}L_i(\w_t,\lambda_i^t;\B_{2,i}^t) \quad & \text{if }i\in \B_1^t\\ s_i^t &\text{o.w.}\end{cases},\quad \widetilde{s}_i^t = (1-\gamma_0')s_i^t+\gamma_0'\nabla^2_{\lambda\lambda}L_i(\w_t,\lambda_i^t;\B_{2,i}^t)
\end{equation*}
Consider 
\begin{equation}\label{ineq:200} 
    \begin{aligned}
    & \E_{\B_{2,i}^t}[\|\widetilde{s}_i^t-\nabla^2_{\lambda\lambda}L_i(\w_t,\lambda_i(\w_t))\|^2]\\
    &=\E_{\B_{2,i}^t}[[\|(1-\gamma_0')s_i^t+\gamma_0'\nabla^2_{\lambda\lambda}L_i(\w_t,\lambda_i^t;\B_{2,i}^t)-\nabla^2_{\lambda\lambda}L_i(\w_t,\lambda_i(\w_t))\|^2]\\
    &=\E_{\B_{2,i}^t}[\|(1-\gamma_0')[s_i^t-\nabla^2_{\lambda\lambda}L_i(\w_t,\lambda_i(\w_t))]+\gamma_0'[\nabla^2_{\lambda\lambda}L_i(\w_t,\lambda_i^t;\B_{2,i}^t)-\nabla^2_{\lambda\lambda}L_i(\w_t,\lambda_i^t)]\\
    &\quad +\gamma_0'[\nabla^2_{\lambda\lambda}L_i(\w_t,\lambda_i^t)- \nabla^2_{\lambda\lambda}L_i(\w_t,\lambda_i(\w_t))]\|^2]\\
    &=\|(1-\gamma_0')[s_i^t-\nabla^2_{\lambda\lambda}L_i(\w_t,\lambda_i(\w_t))]+\gamma_0'[\nabla^2_{\lambda\lambda}L_i(\w_t,\lambda_i^t)- \nabla^2_{\lambda\lambda}L_i(\w_t,\lambda_i(\w_t))]\|^2\\
    &\quad +\E_{\B_{2,i}^t}[\|\gamma_0'[\nabla^2_{\lambda\lambda}L_i(\w_t,\lambda_i^t;\B_{2,i}^t)-\nabla^2_{\lambda\lambda}L_i(\w_t,\lambda_i^t)]\|^2]\\
    &\leq (1+\frac{\gamma_0'}{2})(1-\gamma_0')^2\|s_i^t-\nabla^2_{\lambda\lambda}L_i(\w_t,\lambda_i(\w_t))\|^2+(1+\frac{2}{\gamma_0'})\gamma_0'^2\|\nabla^2_{\lambda\lambda}L_i(\w_t,\lambda_i^t)- \nabla^2_{\lambda\lambda}L_i(\w_t,\lambda_i(\w_t))\|^2 +\gamma_0'^2\sigma^2\\
    &\leq (1-\frac{\gamma_0'}{2})\|s_i^t-\nabla^2_{\lambda\lambda}L_i(\w_t,\lambda_i(\w_t))\|^2+4\gamma_0'L_{L\lambda\lambda}^2\|\lambda_i^t-\lambda_i(\w_t)\|^2 +\gamma_0'^2\sigma^2\\
    \end{aligned}
\end{equation}
Note that for the randomness of query sampling we have
\begin{equation*}
    \E_t[\|s_i^{t+1}-\nabla^2_{\lambda\lambda}L_i(\w_t,\lambda_i(\w_t))\|^2]=\frac{|\B_1^t|}{n}\E_{\B_{2,i}^t}[\|\widetilde{s}_i^t-\nabla^2_{\lambda\lambda}L_i(\w_t,\lambda_i(\w_t))\|^2]+\frac{n-|\B_1^t|}{n}\|s_i^t-\nabla^2_{\lambda\lambda}L_i(\w_t,\lambda_i(\w_t))\|^2
\end{equation*}
which follows that
\begin{equation}\label{eq:1}
    \E_{\B_{2,i}^t}[\|\widetilde{s}_i^t-\nabla^2_{\lambda\lambda}L_i(\w_t,\lambda_i(\w_t))\|^2]=\frac{n}{|\B_1^t|}\E_t[\|s_i^{t+1}-\nabla^2_{\lambda\lambda}L_i(\w_t,\lambda_i(\w_t))\|^2]-\frac{n-|\B_1^t|}{|\B_1^t|}\|s_i^t-\nabla^2_{\lambda\lambda}L_i(\w_t,\lambda_i(\w_t))\|^2
\end{equation}
Then by plugging the equality~(\ref{eq:1}) into inequality~(\ref{ineq:200}), we obtain
\begin{equation*}
    \begin{aligned}
    & \frac{n}{|\B_1^t|}\E_t[\|s_i^{t+1}-\nabla^2_{\lambda\lambda}L_i(\w_t,\lambda_i(\w_t))\|^2]-\frac{n-|\B_1^t|}{|\B_1^t|}\|s_i^t-\nabla^2_{\lambda\lambda}L_i(\w_t,\lambda_i(\w_t))\|^2\\
    &\leq (1-\frac{\gamma_0'}{2})\|s_i^t-\nabla^2_{\lambda\lambda}L_i(\w_t,\lambda_i(\w_t))\|^2+4\gamma_0'L_{L\lambda\lambda}^2\|\lambda_i^t-\lambda_i(\w_t)\|^2+\gamma_0'^2\sigma^2\\
    \end{aligned}
\end{equation*}
It follows
\begin{equation*}
    \begin{aligned}
    & \E_t[\|s_i^{t+1}-\nabla^2_{\lambda\lambda}L_i(\w_t,\lambda_i(\w_t))\|^2]\\
    &\leq (1-\frac{|\B_1^t|\gamma_0'}{2n})\|s_i^t-\nabla^2_{\lambda\lambda}L_i(\w_t,\lambda_i(\w_t))\|^2 +\frac{4|\B_1^t|\gamma_0'L_{L\lambda\lambda}^2}{n}\|\lambda_i^t-\lambda_i(\w_t)\|^2 +\frac{|\B_1^t|\gamma_0'^2\sigma^2}{n}\\
    \end{aligned}
\end{equation*}
Furthermore,
\begin{equation*}
    \begin{aligned}
    & \E_t[\|s_i^{t+1}-\nabla^2_{\lambda\lambda}L_i(\w_{t+1},\lambda_i(\w_{t+1}))\|^2]\\
    &\leq (1+\frac{|\B_1^t|\gamma_0'}{4n})\E_t[\|s_i^{t+1}-\nabla^2_{\lambda\lambda}L_i(\w_t,\lambda_i(\w_t))\|^2]+(1+\frac{4n}{|\B_1^t|\gamma_0'})L_{L\lambda\lambda}^2(1+C_\lambda^2)\E_t[\|\w_t-\w_{t+1}\|^2]\\
    &\leq (1-\frac{|\B_1^t|\gamma_0'}{4n})\|s_i^t-\nabla^2_{\lambda\lambda}L_i(\w_t,\lambda_i(\w_t))\|^2 +\frac{8|\B_1^t|\gamma_0'L_{L\lambda\lambda}^2}{n}\|\lambda_i^t-\lambda_i(\w_t)\|^2 +\frac{2|\B_1^t|\gamma_0'^2\sigma^2}{n}\\
    &\quad +\frac{8nL_{L\lambda\lambda}^2(1+C_\lambda^2)}{|\B_1^t|\gamma_0'}\E_t[\|\w_t-\w_{t+1}\|^2]
    \end{aligned}
\end{equation*}
where we use the assumption $\gamma_0'\leq 1\leq  \frac{4n}{|\B_1^t|}$ i.e. $\frac{4n}{|\B_1^t|\gamma_0'}\geq 1$.\\
Taking expectation over all randomness and taking summation over all queries, we have
\begin{equation*}
    \begin{aligned}
    & \E[\|s^{t+1}-\nabla^2_{\lambda\lambda}L(\w_{t+1},\lambda(\w_{t+1}))\|^2]\\
    &\leq (1-\frac{|\B_1^t|\gamma_0'}{4n})\E[\|s^t-\nabla^2_{\lambda\lambda}L(\w_t,\lambda_i(\w_t))\|^2] +\frac{8|\B_1^t|\gamma_0'L_{L\lambda\lambda}^2}{n}\E[\|\lambda^t-\lambda(\w_t)\|^2] +2|\B_1^t|\gamma_0'^2\sigma^2\\
    &\quad +\frac{8n^2L_{L\lambda\lambda}^2(1+C_\lambda^2)}{|\B_1^t|\gamma_0'}\E[\|\w_t-\w_{t+1}\|^2]
    \end{aligned}
\end{equation*}
Taking summation over $t=0,\dots,T-1$, we obtain
\begin{equation*}
    \begin{aligned}
    & \sum_{t=0}^T\E[\|s^t-\nabla^2_{\lambda\lambda}L(\w_t,\lambda(\w_t))\|^2]\\
    &\leq \frac{4n}{|\B_1^t|\gamma_0'}\|s^0-\nabla^2_{\lambda\lambda}L(\w_0,\lambda_i(\w_0))\|^2+32L_{L\lambda\lambda}^2\sum_{t=0}^{T-1}\E[\|\lambda^t-\lambda(\w_t)\|^2] +8n\gamma_0'T\sigma^2\\
    &\quad +\frac{32n^3L_{L\lambda\lambda}^2(1+C_\lambda^2)}{|\B_1^t|^2\gamma_0'^2}\sum_{t=0}^{T-1}\E[\|\w_t-\w_{t+1}\|^2]\\
    \end{aligned}
\end{equation*}
\end{proof}

\section{SONG and K-SONG with Faster Convergence}

In this section, we try to improve the convergence rate of SONG/K-SONG from $O(1/\epsilon^4)$ to $O(1/\epsilon^3)$. To this end, we apply two kinds of variance reduction updates, STORM and MSVR, to approximate $\nabla{F}(\w_t)$. 

Assume that we have a stochastic gradient estimator $\nabla f_i(\x_{t};\xi_t)$. The STORM estimator is updated by
\begin{equation*}
	\d^{t}_{i}=\begin{cases} 
	(1-\gamma_t)(\d^{t-1}_{i}-\nabla f_i(\x_{t-1};\xi_t))+ \nabla f_i(\x_{t};\xi_t) \quad & \text{if }i\in \B_1^t\\ 
	\d^{t-1}_{i} & \text{o.w.}
	\end{cases},
\end{equation*}
where $\xi_t$ and $\B_1^t$ denote a set of samples for updating $d^{t}_{i}$.

The MSVR estimator is updated by
\begin{equation*}
	\d^{t}_{i}=\begin{cases} 
	(1-\gamma_t)\d^{t-1}_{i} + \gamma_t \nabla f_i(\x_{t};\xi_t) + \beta_t (\nabla f_i(\x_{t};\xi_t)-\nabla f_i(\x_{t-1};\xi_t)) \quad & \text{if }i\in \B_1^t\\ 
	\d^{t-1}_{i} & \text{o.w.}
	\end{cases},
\end{equation*}
where $\beta_t$ can be set to $\frac{n-B_1}{B_1(1-\gamma_t)}+(1-\gamma_t)$ according to the analysis. We can observe that if we set $\beta_t$ to $1-\gamma_t$, then the MSVR estimator will become the STORM estimator.

A single-point version of MSVR (named as MSVR-SP) estimator is updated by
\begin{equation*}
	\d^{t}_{i}=\begin{cases} 
	(1-\gamma_t)\d^{t-1}_{i} + \gamma_t \nabla f_i(\x_{t};\xi_t) + \beta_t \nabla f_i(\x_{t};\xi_t)^T (\x_{t}-\x_{t-1}) \quad & \text{if }i\in \B_1^t\\ 
	\d^{t-1}_{i} & \text{o.w.}
	\end{cases}.
\end{equation*}
It can be proved that the MSVR-SP estimator enjoys the similar recurrence for the estimation error as MSVR.

\begin{algorithm}[t]
\caption{Faster K-SONG}\label{algo_faster_song}
\begin{algorithmic}
\REQUIRE $\w_0,\w_1,\m_0,\lambda^0,\lambda^1,z^0,u^0,s^0$
\ENSURE $\w_{T+1}$
\FOR{$t=1,2,\dots,T$}
\STATE Draw batch of queries $\B_1^t\in \{1,\dots,n\}$
\STATE Draw batch of items $\B_{2,i}^t$ for each $i\in \B_1^t$
\STATE $  u_i^{t}=\begin{cases}(1-\gamma_{u,t})u_{i}^{t-1} +\gamma_{u,t} g_i(\w_{t};\B_{2,i}^t) + \beta_{u,t}(g_i(\w_{t};\B_{2,i}^t)-g_i(\w_{t-1};\B_{2,i}^t)) \quad & \text{if }i\in \B_1^t\\   u_{i}^{t-1} & \text{o.w.} \end{cases}$
\STATE  $s_i^{t}=\begin{cases} (1-\gamma_{s,t})s_{i}^{t-1} +\gamma_{s,t} \nabla^2_{\lambda\lambda}L_i(\w_{t},\lambda_i^{t};\B_{2,i}^t) + \beta_{s,t}(\nabla^2_{\lambda\lambda}L_i(\w_{t},\lambda_i^{t};\B_{2,i}^t)-\nabla^2_{\lambda\lambda}L_i(\w_{t-1},\lambda_i^{t-1};\B_{2,i}^t)) \quad & \text{if }i\in \B_1^t\\   s_{i}^{t-1} & \text{o.w.} \end{cases}$
\STATE $z_i^{t}=\begin{cases} (1-\gamma_{z,t})z_i^{t-1} + \gamma_{z,t} \nabla L_i (\w_{t}, \lambda_i^{t};\B_{2,i}^t)) + \beta_{z,t}(\nabla_{\lambda} L_i (\w_t, \lambda_i^t;\B_{2,i}^t)-\nabla_{\lambda} L_i (\w_{t-1}, \lambda_i^{t-1};\B_{2,i}^t))) \quad & \text{if }i\in \B_1^t\\   z_i^{t-1} & \text{o.w.} \end{cases}$
\STATE $\lambda_i^{t+1}=\begin{cases} \lambda_i^t - \tau\tau_t z_i^t  \quad & \text{if }i\in \B_1^t\\   \lambda_i^t & \text{o.w.} \end{cases}$
\STATE Compute stochastic gradient estimator $G(\w_{t-1})$ and $G(\w_t)$ according to (\ref{update_grad})
\STATE $\m_{t}=(1-\gamma_{m,t})(\m_{t-1}-G(\w_{t-1}))+G(\w_{t})$
\STATE $\w_{t+1}=\w_t-\alpha\eta_t \m_{t}$
\ENDFOR
\end{algorithmic}
\end{algorithm}

Now we present a convergence analysis for Faster K-SONG. Similarly to the analysis of SONG and K-SONG, we consider the compositional bilevel optimization problem in (\ref{prob_ana2}), and we make the same assumptions as in Assumption~\ref{assump_ana}.

Now we present the formal statement of Theorem~\ref{thm:faster_ksong} regarding to algorithm~\ref{algo_faster_song}.

\begin{thm}\label{thm:faster_ksong}
	Under assumption~\ref{assump_ana}, with $\tau\leq\min\{\frac{1}{2L_L},\frac{8n}{\mu_L B_1}\}$, $\alpha\leq\min\{\frac{B_1}{8 C_{\text{max}}n}, \frac{1}{8C_\lambda}\sqrt{\frac{\tau\mu_L B_1}{Cn}}\}$, $\eta_t=\tau_t=c/(c_0+t)^{1/3}$, $\gamma_{z,t+1}=\frac{n\eta_t^2}{B_1}\left(\frac{1}{7 L_F c^3} + \frac{8C\alpha\tau_t\tau B_1}{\mu_L n\eta_t}\right)$, $\gamma_{u,t+1}=\left(\frac{2n}{7B_1 L_F c^3}+\frac{4C_1\alpha n}{B_1}\right)\eta_t^2$, $\gamma_{s,t+1}=\left(\frac{2n}{7B_1 L_F c^3}+\frac{4C_2\alpha n}{B_1}\right)\eta_t^2$, $\gamma_{m,t+1}=\left(\frac{1}{7L_F c^3}+\alpha\right)\eta_t^2$, $\beta_t=1-\gamma_t+\frac{n-B_1}{B_1 (1-\gamma_t)}$, where $C\geq\max\{\frac{8C_0 n}{\tau\mu_L B_1}, \frac{2 C_{\text{max}} n^2}{2\alpha B_1^2}\}$, $c_0\geq\max\{2,(4L_F c)^3,\left(\frac{8n}{7 B_1 L_F c}\right)^{3/2}, \left(\frac{32 C \alpha\tau c^2}{\mu_L}\right)^{3/2},\left(\frac{16C_1\alpha n c^2}{B_1}\right)^{3/2},\left(\frac{64nC_3}{7L_F B_1 c}\right)^{3/2},\left(\frac{128C_1 n C_3\alpha c^2}{B_1}\right)^{3/2},\left(\frac{16C_2\alpha n c^2}{B_1}\right)^{3/2},\\ \left(\frac{64nC_6}{7L_F B_1 c}\right)^{3/2},\left(\frac{128C_2 n C_6\alpha c^2}{B_1}\right)^{3/2},\left(\frac{4}{7L_F c}\right)^{3/2},\left(4\alpha c^2\right)^{3/2}\}$ and $C_0,C_1,C_2,C_3,C_4,C_5,C_6,C_{\text{max}}$ are constants specified in the proof, Algorithm~\ref{algo_faster_song} ensures that after $T=\O(\frac{1}{\epsilon^3})$ iterations, we can find an $\epsilon$-stationary solution of $F(\w_t)$, i.e., $\E\left[\sum_{t=1}^T\frac{1}{T}\left\|\nabla F(\w_t) \right\|^2\right] \leq \O\left(\frac{1}{T^{2/3}}\right)$.
\end{thm}

\subsection{Convergence Analysis of Theorem~\ref{thm:faster_ksong}}

In this section, we present the convergence analysis of Theorem~\ref{thm:faster_ksong}. To this end, we will first present several technical lemmas.

\begin{lemma}\label{lemma_start}

Consider the update $\w_{t+1}=\w_t-\alpha\eta_t \m_{t}$. Then under Assumption~\ref{assump_ana}, with $\alpha\eta_t L_F\leq\frac{1}{2}$, we have
\begin{equation*}
	F(\w_{t+1})\leq F(\w_t)+\frac{\alpha\eta_t}{2}||\nabla F(\w_t)-\m_{t}||^2-\frac{\alpha\eta_t}{2}||\nabla F(\w_t)||^2-\frac{\alpha\eta_t}{4}||\m_{t}||^2.
\end{equation*}

\end{lemma}

By $L_F$-smoothness of $F(\w)$ (proved in Lemma~\ref{lemma_Fsmooth}), Lemma~\ref{lemma_start} can be proved similarly to Lemma~\ref{lemma:1}. From the above lemma, we can see that the key to the proof of Theorem~\ref{thm:faster_ksong} is to bound $||\nabla F(\w_t)-\m_{t}||^2$. The lemma below will decompose this error into several terms that can be bounded separately.

\begin{lemma}\label{lemma_decomp}

Denote $\sum_{i\in\S}\left\| \lambda_i(\w_t) - \lambda_i^t \right\|^2 = \left\| \lambda(\w_t) - \lambda^t \right\|^2$, $\sum_{i\in\S}\left\|u_i^t-g_i(\w_t) \right\|^2=\left\|u^t-g(\w_t)  \right\|^2$ and $\sum_{i\in\S}\left\|s_i^t -\nabla_{\lambda\lambda}^2 L_i(\w_t,\lambda_i^t) \right\|^2=\left\|s^t -\nabla_{\lambda\lambda}^2 L(\w_t,\lambda^t) \right\|^2$. Consider the updates in Algorithm~\ref{algo_faster_song}, under Assumption~\ref{assump_ana}, for all $t>0$, we have
\begin{equation*}
\begin{aligned}
&\left\|\nabla F\left(\w_{t}\right)-\m_{t}\right\|^{2} \\
\leq & 2\left\|\frac{1}{n} \sum_{i \in \S} G_{i}\left(\w_{t}\right)-\m_{t}\right\|^{2}+  \frac{4 C_0}{n} \left\| \lambda(\w_t)-\lambda^t \right\|^2 \\&+  \frac{4 C_1}{n} \left\| u^t - g(\w_t) \right\|^2 + \frac{4 C_2}{n}  \left\| s^t - \nabla_{\lambda\lambda}^2 L(\w_t,\lambda^t) \right\|^2,
\end{aligned}
\end{equation*}
where $C_0=3 L_{\psi}^2 B_f^2+9\frac{L_{L\w\lambda}^2 C_{\psi}^2 B_f^2}{\gamma^2}+9\frac{C_{L\w\lambda}^2 L_{L\lambda\lambda}^2 C_{\psi}^2 B_f^2}{\gamma^4}+9\frac{C_{L\w\lambda}^2 L_{\psi}^2 B_f^2}{\gamma^2}+ 3 C_{g}^2 C_{f}^2 C_{\psi}^2$, $C_1=3 C_{\psi}^2 C_f^2 + 3 B_{\psi}^2 C_g^2 L_f^2 + 6\frac{C_{L\w\lambda}^2 C_{\psi}^2 C_f^2}{\gamma^2}$, $C_2=6 \frac{C_{L\w\lambda}^2 C_{\psi}^2 B_f^2}{\gamma^4}$.
	
\end{lemma}

Next, we will bound each term on the RHS in the inequality of the above lemma separately. We first bound the first term.

\begin{lemma}\label{lemma_decomp_2}

Denote $\sum_{i\in\S}\left\| u_i^t - u_i^{t-1} \right\|^2 = \left\| u^t - u^{t-1} \right\|^2$, $\sum_{i\in\S}\left\|\lambda_i^t-\lambda_i^{t-1} \right\|^2=\left\|\lambda^t-\lambda^{t-1} \right\|^2$ and $\sum_{i\in\S}\left\|s_i^t -s_i^{t-1} \right\|^2=\left\|s^t -s^{t-1} \right\|^2$. Assume $\E\left[\frac{1}{B_1}\sum_{i\in\B_1^t}G_i(\w_t)-\frac{1}{n}\sum_{i\in \S}G_i(\w_t) \right]\leq \sigma^2$, we have

\begin{equation*}
\begin{aligned}
	&  \E\left[ \left\| \m_t-\frac{1}{n}\sum_{i \in \S} G_i(\w_t) \right\|^2\right] \\
	\leq & (1-\gamma_{m,t}) \E\left[ \left\| \m_{t-1}-\frac{1}{n}\sum_{i\in\S}G_i(\w_{t-1}) \right\|^2 \right]+ 2\gamma_{m,t}^2\sigma^2 + \frac{2 C_3}{n} \left\| u^t - u^{t-1} \right\|^2 \\
	& \quad + 2 C_4 \left\| \w_t - \w_{t-1} \right\|^2 + \frac{2 C_5}{n} \left\| \lambda^t - \lambda^{t-1} \right\|^2 + \frac{2 C_6}{n} \left\| s^t - s^{t-1} \right\|^2,
\end{aligned}
\end{equation*}
where $C_3=6 C_{\psi}^2 C_{f}^2 + 12 \frac{C_{L\w\lambda}^2 C_{\psi}^2 C_f^2}{\gamma^2} + 9 B_{\psi}^2 C_g^2 L_f^2$, $C_4=12 B_f^2 L_{\psi}^2 + 24 \frac{C_{\psi}^2 B_f^2}{\gamma^2} L_{L\w\lambda}^2 + 24 \frac{L_{\psi}^2 B_f^2}{\gamma^2} C_{L\w\lambda}^2 + 18 C_g^2 C_f^2 C_{\psi}^2 + 9 B_{\psi}^2 C_f^2 L_g^2$, $C_5=12 B_f^2 L_{\psi}^2 + 24 \frac{C_{\psi}^2 B_f^2}{\gamma^2} L_{L\w\lambda}^2 + 24 \frac{L_{\psi}^2 B_f^2}{\gamma^2} C_{L\w\lambda}^2 + 18 C_g^2 C_f^2 C_{\psi}^2$, and $C_6=12\frac{C_{L\w\lambda}^2 C_{\psi}^2 B_f^2}{\gamma^4}$.

\end{lemma}

To bound the $\left\| \lambda_i(\w_t)-\lambda_i^t \right\|^2$ term in Lemma~\ref{lemma_decomp}, we can use the following lemma.

\begin{lemma}\label{lemma_fksong_lambda}

Consider the update in Algorithm~\ref{algo_faster_song}. Then under Assumption~\ref{assump_ana}, with $\tau_t\leq\frac{1}{2}$ and $\tau_t\tau\leq\frac{4n}{\mu_L B_1}$, we have

\begin{equation*}
\begin{aligned}
	& \E\left[\left\|\lambda^{t+1}-\lambda(\w_{t+1})\right\|^2 \right] \\
	\leq & \left(1-\frac{\tau\tau_t\mu_L B_1}{4n}\right)\E\left[\left\| \lambda(\w_t)-\lambda^t\right\|^2 \right] + \frac{8\tau_t\tau B_1}{\mu_L n}\left\| \nabla_{\lambda}L(\w_t,\lambda^t)-z^t \right\|^2 \\
	&\quad - \frac{3\tau B_1}{\tau_t n} \left(\frac{1}{\tau}-L_L\right)\left\|\lambda^{t+1} - \lambda^t \right\|^2 + \frac{8n^2 C_\lambda^2}{\tau\tau_t\mu_L B_1}\E\left[\left\| \w_{t+1}-\w_t \right\|^2 \right].
\end{aligned}
\end{equation*}

\end{lemma}

The following lemma bound the terms involving $u^t$, $z^t$ and $s^t$ on the RHS in the inequalities of the above two lemmas.

\begin{lemma}\label{lemma_msvr}
Suppose $f_i, i\in\S$ is a mapping, $\E[f_i(\x;\xi)]=f_i(\x)$ and $\E[f_i(\x;\xi)-f_i(\x)]\leq\sigma^2$. let
\begin{equation*}
	\d^{t}_{i}=\begin{cases} 
	(1-\gamma_t)\d^{t-1}_{i} + \gamma_t f_i(\x_{t};\xi_t) + \beta_t (f_i(\x_{t};\xi_t)- f_i(\x_{t-1};\xi_t)) \quad & \text{if }i\in \B_1^t\\ 
	\d^{t-1}_{i} & \text{o.w.}
	\end{cases}.
\end{equation*}

By setting $\gamma_t \leq \frac{1}{2}$ and $\beta_t=1-\gamma_t+\frac{n-B_1}{B_1 (1-\gamma_t)}$, for $t\geq 1$, we have
\begin{equation}
\begin{aligned}
	& \E\left[\left\|\d^t -f(\x_t) \right\|^2 \right] =\sum_{i\in\S} \E\left[\left\|\d_i^t -f_i(\x_t) \right\|^2 \right] \\
	 \leq & \left(1-\frac{\gamma_t B_1}{n}\right)\E\left[\left\| \d^{t-1}- f(\x_{t-1}) \right\|^2\right] + \frac{8 n}{B_1}\sum_{i\in\S} \E\left[\left\|f_i(\x_t;\xi_t)-f_i(\x_{t-1};\xi_t) \right\|^2 \right] + 2 B_1 \gamma_t^2 \sigma^2,
\end{aligned}	
\end{equation}
and
\begin{equation*}
\begin{aligned}
	&\E\left[\left\| \d^t - \d^{t-1} \right\|^2\right] = \sum_{i\in\S} \E\left[\left\| \d_i^t - \d_i^{t-1} \right\|^2\right]  \\
	\leq & 2 B_1 \gamma_t^2 \sigma^2 + \frac{4 B_1 \gamma_t^2}{n}\E\left[\left\| f(\x_{t-1}) - \d^{t-1} \right\|^2 \right] + \frac{9n}{B_1}\sum_{i\in\S}\E\left[ \left\|f_i(\x_t;\xi_t)-f_i(\x_{t-1};\xi_t) \right\|^2\right].
\end{aligned}	
\end{equation*}

\end{lemma}

\begin{proof}[Proof of Theorem~\ref{thm:faster_ksong}]

First, we apply Lemma~\ref{lemma_msvr} to $\delta_{L\lambda,t}=\left\|z^t -\nabla_{\lambda}L(\w_t,\lambda^t) \right\|^2$. We have
\begin{equation*}
	\E\left[\delta_{L\lambda,t+1}\right] \leq \left(1-\frac{\gamma_{z,t+1}B_1}{n}\right) \E\left[\delta_{L\lambda,t}\right] + \frac{16 n^2 L_L^2}{B_1}\left(\left\|\w_{t+1}-\w_t \right\|^2 +\frac{1}{n} \left\|\lambda^{t+1}-\lambda^t \right\|^2\right) + 2B_1 \gamma_{z,t+1}^2\sigma^2,
\end{equation*}

and thus
\begin{equation}\label{eq:fksong_proof_1}
\begin{aligned}
	 \E\left[\frac{\delta_{L\lambda,t+1}}{n\eta_t} - \frac{\delta_{L\lambda,t}}{n\eta_{t-1}} \right] & \leq \frac{2B_1 \gamma_{z,t+1}^2\sigma^2}{n\eta_t} + \frac{1}{n}\left(\frac{1}{\eta_t} -\frac{1}{\eta_{t-1}}-\frac{\gamma_{z,t+1}B_1}{n \eta_t} \right)\E\left[\delta_{L\lambda,t}\right] \\
	 &\quad + \frac{16 n L_L^2}{B_1\eta_t}\left(\left\|\w_{t+1}-\w_t \right\|^2 +\frac{1}{n} \left\|\lambda^{t+1}-\lambda^t \right\|^2\right).
\end{aligned}
\end{equation}

Denote $\delta_{\lambda,t}=\left\| \lambda^t - \lambda(\w_t) \right\|^2$, from Lemma~\ref{lemma_fksong_lambda} we have
\begin{equation}\label{eq:fksong_proof_2}
\begin{aligned}
	 \E\left[\frac{C\alpha}{n}(\delta_{\lambda,t+1}-\delta_{\lambda,t}) \right] & \leq -\frac{C\alpha\tau\tau_t\mu_L B_1}{4n^2}\E\left[\delta_{\lambda,t}\right] + \frac{8C\alpha\tau_t\tau B_1}{\mu_L n^2}\E\left[\delta_{L\lambda,t}\right] \\
	 &\quad -\frac{3C\alpha\tau B_1}{\tau_t n^2}\left(\frac{1}{\tau}-L_L\right) \E\left[\left\|\lambda^{t+1}-\lambda^t \right\|^2 \right] + \frac{8C\alpha n C_{\lambda}^2}{\tau\tau_t\mu_L B_1}\E\left[\left\|\w_{t+1}-\w_t \right\|^2\right],
\end{aligned}	
\end{equation}
where $C$ will be given below.

Set $\eta_t=\tau_t=\frac{c}{{(c_0 + t)}^{1/3}}$. To ensure $\eta_t\leq\frac{1}{4L_F}$, we need $c_0\geq (4L_F c)^3$. Thus
\begin{equation*}
\begin{aligned}
	\frac{1}{\eta_t} - \frac{1}{\eta_{t-1}} & = \frac{(c_0 + t)^{1/3}}{c} - \frac{(c_0 + t-1)^{1/3}}{c} \\
	& \leq \frac{1}{3c(c_0 + t-1)^{2/3}} \leq \frac{1}{3c(c_0/2 + t)^{2/3}} \\
	& \leq\frac{2^{2/3}}{3c(c_0 + t)^{2/3}}\leq \frac{2^{2/3}}{3c^3}\eta_t^2\leq \frac{1}{7L_F c^3}\eta_t,
\end{aligned}
\end{equation*}
where the first inequality holds by the concavity of the function $f(x)=x^{1/3}$, i.e., $(x+y)^{1/3}\leq x^{1/3} + \frac{y}{3 x^{2/3}}$, the second inequality is because $c_0\geq 2$. Then with $\gamma_{z,t+1}=\frac{n\eta_t^2}{B_1}\left(\frac{1}{7 L_F c^3} + \frac{8C\alpha\tau_t\tau B_1}{\mu_L n\eta_t}\right)$, where $\gamma_{z,t+1}<\frac{1}{2}$ for $c_0\geq\max\left\{(\frac{4n}{7 B_1 L_F c})^{3/2}, (\frac{32 C \alpha\tau c^2}{\mu_L})^{3/2}\right\}$, by combining (\ref{eq:fksong_proof_1}) and (\ref{eq:fksong_proof_2}), we have

\begin{equation}\label{eq:fksong_proof_3}
\begin{aligned}
	& \E\left[\frac{\delta_{L\lambda,t+1}}{n\eta_t} - \frac{\delta_{L\lambda,t}}{n\eta_{t-1}} \right] + \E\left[\frac{C\alpha}{n}(\delta_{\lambda,t+1}-\delta_{\lambda,t}) \right] \\
	\leq &  \frac{2B_1 \gamma_{z,t+1}^2\sigma^2}{n\eta_t} + \frac{16 n L_L^2}{B_1\eta_t}\left(\left\|\w_{t+1}-\w_t \right\|^2 +\frac{1}{n} \left\|\lambda^{t+1}-\lambda^t \right\|^2\right) \\
	- & \frac{C\alpha\tau\tau_t\mu_L B_1}{4n^2}\E\left[\delta_{\lambda,t}\right] -\frac{3C\alpha\tau B_1}{\tau_t n^2}\left(\frac{1}{\tau}-L_L\right) \E\left[\left\|\lambda^{t+1}-\lambda^t \right\|^2 \right] + \frac{8C\alpha n C_{\lambda}^2}{\tau\tau_t\mu_L B_1}\E\left[\left\|\w_{t+1}-\w_t \right\|^2\right].
\end{aligned}	
\end{equation}

To continue, we add the above recursion with the recursions for $\delta_{g,t}=\left\|u^t-g(\w_t) \right\|^2$, $\delta_{L\lambda\lambda,t}=\left\|s^t-\nabla_{\lambda\lambda}^2 L(\w_t,\lambda^t) \right\|^2$, and $\delta_{m,t}=\left\|m_t - \frac{1}{n}\sum_{i\in\S}G_i(\w_t) \right\|^2$. Using Lemma~\ref{lemma_msvr} and Lemma~\ref{lemma_decomp_2}, we have
\begin{equation*}
\begin{aligned}
	\E\left[\frac{\delta_{g,t+1}}{n\eta_t}-\frac{\delta_{g,t}}{n\eta_{t-1}} \right] \leq & \frac{1}{n}\left(\frac{1}{\eta_t}-\frac{1}{\eta_{t-1}}-\frac{\gamma_{u,t+1}B_1}{n\eta_t}\right)\E\left[\delta_{g,t} \right] + \frac{8n C_g^2}{B_1\eta_t} \E\left[\left\|\w_{t+1}-\w_t \right\|^2\right] + \frac{2 B_1 \gamma_{u,t+1}^2\sigma^2}{n\eta_t} \\
	\E\left[\frac{\delta_{L\lambda\lambda,t+1}}{n\eta_t}-\frac{\delta_{L\lambda\lambda,t}}{n\eta_{t-1}} \right] \leq & \frac{1}{n}\left(\frac{1}{\eta_t}-\frac{1}{\eta_{t-1}}-\frac{\gamma_{s,t+1}B_1}{n\eta_t}\right)\E\left[\delta_{L\lambda\lambda,t} \right] \\
	& + \frac{16n L_{L\lambda\lambda}^2}{B_1\eta_t} \E\left[\left\|\w_{t+1}-\w_t \right\|^2 + \frac{1}{n}\left\|\lambda^{t+1}-\lambda^t\right\|^2\right] + \frac{2 B_1 \gamma_{s,t+1}^2\sigma^2}{n\eta_t} \\
	\E\left[\frac{\delta_{m,t+1}}{\eta_t} - \frac{\delta_{m,t}}{\eta_{t-1}} \right] \leq &\left(\frac{1}{\eta_t}-\frac{1}{\eta_{t-1}}-\frac{\gamma_{m,t+1}}{\eta_t}\right)\E\left[\delta_{m,t} \right] + \frac{2\gamma_{m,t+1}^2\sigma^2}{\eta_t} + \frac{2 C_3}{n \eta_t} \E\left[\left\| u^{t+1} - u^{t} \right\|^2\right] \\
	& + \frac{2 C_4}{\eta_t} \E\left[\left\| \w_{t+1} - \w_{t} \right\|^2\right] + \frac{2 C_5}{n\eta_t} \E\left[\left\| \lambda^{t+1} - \lambda^{t} \right\|^2\right] + \frac{2 C_6}{n\eta_t} \E\left[\left\| s^{t+1} - s^{t} \right\|^2\right].
\end{aligned}
\end{equation*}

To bound $\left\|u^{t+1}-u^t \right\|^2$ and $\left\|s^{t+1}-s^t \right\|^2$ in the above inequality, we use Lemma~\ref{lemma_msvr} and have
\begin{equation*}
\begin{aligned}
	& \frac{2 C_3}{n\eta_t} \E\left[\left\| u^{t+1} - u^{t} \right\|^2\right] + \frac{2 C_6}{n \eta_t} \E\left[\left\| s^{t+1} - s^{t} \right\|^2\right] \\
	\leq & \frac{8 B_1 C_3 \gamma_{u,t+1}^2}{n^2\eta_t}\E\left[\delta_{g,t}\right] + \frac{8 B_1 C_6 \gamma_{s,t+1}^2}{n^2\eta_t}\E\left[\delta_{L\lambda\lambda,t}\right] + \frac{2(2\gamma_{u,t+1}^2 C_3 + 2\gamma_{s,t+1}^2 C_6)B_1}{n\eta_t}\sigma^2 \\
	& + \left(\frac{18 C_3 n C_g^2}{\eta_t B_1} + \frac{36 C_6 n L_{L\lambda\lambda}^2}{\eta_t B_1} \right)\E\left[\left\|\w_{t+1}-\w_t\right\|^2\right] + \frac{36 C_6 L_{L\lambda\lambda}^2}{\eta_t B_1} \E\left[\left\|\lambda^{t+1}-\lambda^t\right\|^2\right].
\end{aligned}
\end{equation*}

Then, we combine the above recursions with (\ref{eq:fksong_proof_3}) and have
\begin{equation*}
\begin{aligned}
	& \E\left[\frac{\delta_{L\lambda,t+1}}{n\eta_t} - \frac{\delta_{L\lambda,t}}{n\eta_{t-1}} \right] + \E\left[\frac{C\alpha}{n}(\delta_{\lambda,t+1}-\delta_{\lambda,t}) \right] + \E\left[\frac{\delta_{g,t+1}}{n\eta_t}-\frac{\delta_{g,t}}{n\eta_{t-1}} \right] + \E\left[\frac{\delta_{L\lambda\lambda,t+1}}{n\eta_t}-\frac{\delta_{L\lambda\lambda,t}}{n\eta_{t-1}} \right] + \E\left[\frac{\delta_{m,t+1}}{\eta_t} - \frac{\delta_{m,t}}{\eta_{t-1}} \right] \\
	\leq & \frac{2(\gamma_{z,t+1}^2+\gamma_{u,t+1}^2+\gamma_{s,t+1}^2+2\gamma_{u,t+1}^2 C_3 + 2\gamma_{s,t+1}^2 C_6)B_1 + 2n\gamma_{m,t+1}^2}{n\eta_t}\sigma^2 - \frac{C\alpha\tau\tau_t\mu_L B_1}{4n^2}\E\left[\delta_{\lambda,t}\right] \\
	 + & \frac{1}{n}\left(\frac{1}{\eta_t}-\frac{1}{\eta_{t-1}}-\frac{\gamma_{u,t+1}B_1}{n\eta_t} + \frac{8 B_1 C_3 \gamma_{u,t+1}^2}{n\eta_t}\right)\E\left[\delta_{g,t} \right]
	 + \frac{1}{n}\left(\frac{1}{\eta_t}-\frac{1}{\eta_{t-1}}-\frac{\gamma_{s,t+1}B_1}{n\eta_t} + \frac{8 B_1 C_6 \gamma_{s,t+1}^2}{n\eta_t}\right)\E\left[\delta_{L\lambda\lambda,t} \right] \\
	 + & \left(\frac{1}{\eta_t}-\frac{1}{\eta_{t-1}}-\frac{\gamma_{m,t+1}}{\eta_t}\right)\E\left[\delta_{m,t} \right] + \left(\frac{16 L_L^2 + 16 L_{L\lambda\lambda}^2+2 C_5 + 36 C_6 L_{L\lambda\lambda}^2}{B_1\eta_t} -\frac{3C\alpha\tau B_1}{\tau_t n^2}\left(\frac{1}{\tau}-L_L\right) \right)\E\left[\left\|\lambda^{t+1}-\lambda^t\right\|^2\right] \\
	 + & \left( \frac{8C\alpha n C_{\lambda}^2}{\tau\tau_t\mu_L B_1} + \frac{(16 L_L^2 + 8 C_g^2+16 L_{L\lambda\lambda}^2+2 C_4+18 C_3 C_g^2+36 C_6 L_{L\lambda\lambda}^2) n}{B_1\eta_t} \right)\alpha^2\eta_t^2 \E\left[\left\|\m_t \right\|^2\right],
\end{aligned}
\end{equation*}
where we use $\w_{t+1}=\w_t - \alpha\eta_t\m_t$. To simplify the inequality above, we denote $C_{\text{max}}=\max\{16 L_L^2 + 16 L_{L\lambda\lambda}^2+2 C_5 + 36 C_6 L_{L\lambda\lambda}^2,16 L_L^2 + 8 C_g^2+16 L_{L\lambda\lambda}^2+2 C_4+18 C_3 C_g^2+36 C_6 L_{L\lambda\lambda}^2\}$. 

Now we plug Lemma~\ref{lemma_start} into the above inequality, and give us
\begin{equation*}
\begin{aligned}
	& \E\left[\frac{\delta_{L\lambda,t+1}}{n\eta_t} - \frac{\delta_{L\lambda,t}}{n\eta_{t-1}} \right] + \E\left[\frac{C\alpha}{n}(\delta_{\lambda,t+1}-\delta_{\lambda,t}) \right] + \E\left[\frac{\delta_{g,t+1}}{n\eta_t}-\frac{\delta_{g,t}}{n\eta_{t-1}} \right] \\
	& + \E\left[\frac{\delta_{L\lambda\lambda,t+1}}{n\eta_t}-\frac{\delta_{L\lambda\lambda,t}}{n\eta_{t-1}} \right] + \E\left[\frac{\delta_{m,t+1}}{\eta_t} - \frac{\delta_{m,t}}{\eta_{t-1}} \right] + \frac{\alpha\eta_t}{2}\left\|\nabla F(\w_t) \right\|^2 \\
	\leq & F(\w_t) - F(\w_{t+1}) + \frac{2(\gamma_{z,t+1}^2+\gamma_{u,t+1}^2+\gamma_{s,t+1}^2+2\gamma_{u,t+1}^2 C_3 + 2\gamma_{s,t+1}^2 C_6)B_1 + 2n\gamma_{m,t+1}^2}{n\eta_t}\sigma^2 \\
	\leq & \frac{2C_0\alpha\eta_t}{n}\E\left[\delta_{\lambda,t}\right] - \frac{C\alpha\tau\tau_t\mu_L B_1}{4n^2}\E\left[\delta_{\lambda,t}\right] \\
	 + & \frac{1}{n}\left(\frac{1}{\eta_t}-\frac{1}{\eta_{t-1}}-\frac{\gamma_{u,t+1}B_1}{n\eta_t} + \frac{8 B_1 C_3 \gamma_{u,t+1}^2}{n\eta_t}\right)\E\left[\delta_{g,t} \right] + \frac{2C_1\alpha\eta_t}{n}\E\left[\delta_{g,t}\right] \\
	 + & \frac{1}{n}\left(\frac{1}{\eta_t}-\frac{1}{\eta_{t-1}}-\frac{\gamma_{s,t+1}B_1}{n\eta_t} + \frac{8 B_1 C_6 \gamma_{s,t+1}^2}{n\eta_t}\right)\E\left[\delta_{L\lambda\lambda,t} \right] + \frac{2C_2\alpha\eta_t}{n}\E\left[\delta_{L\lambda\lambda,t}\right] \\
	 + & \left(\frac{1}{\eta_t}-\frac{1}{\eta_{t-1}}-\frac{\gamma_{m,t+1}}{\eta_t}\right)\E\left[\delta_{m,t} \right] + \alpha\eta_t\E\left[\delta_{m,t} \right] \\
	 + & \left(\frac{C_{\text{max}}}{B_1\eta_t} -\frac{3C\alpha\tau B_1}{\tau_t n^2}\left(\frac{1}{\tau}-L_L\right) \right)\E\left[\left\|\lambda^{t+1}-\lambda^t\right\|^2\right] \\
	 + & \left( \frac{8C\alpha n C_{\lambda}^2}{\tau\tau_t\mu_L B_1} + \frac{C_{\text{max}} n}{B_1\eta_t} \right)\alpha^2\eta_t^2 \E\left[\left\|\m_t \right\|^2\right] - \frac{\alpha\eta_t}{4}\E\left[\left\|\m_t \right\|^2\right].
\end{aligned}
\end{equation*}

With $C\geq\frac{8 C_0 n}{\tau\mu_L B_1}$, we have $\frac{2C_0\alpha\eta_t}{n}-\frac{C\alpha\tau\tau_t\mu_L B_1}{4n^2}\leq 0$. Besides, with $\gamma_{u,t+1}\leq \frac{1}{16C_3}$, we have $\frac{8 B_1 C_3 \gamma_{u,t+1}^2}{n\eta_t}\leq\frac{\gamma_{u,t+1}B_1}{2n\eta_t}$. Then we can set $\gamma_{u,t+1}=\left(\frac{2n}{7B_1 L_F c^3}+\frac{4C_1\alpha n}{B_1}\right)\eta_t^2$, and $\gamma_{u,t+1}\leq\min\{\frac{1}{2},\frac{1}{16C_3}\}$ can be achieved by setting $c_0\geq\max\left\{\left(\frac{8n}{7L_F B_1 c}\right)^{3/2},\left(\frac{16C_1\alpha n c^2}{B_1}\right)^{3/2},\left(\frac{64nC_3}{7L_F B_1 c}\right)^{3/2},\left(\frac{128C_1 n C_3\alpha c^2}{B_1}\right)^{3/2}\right\}$. Similarly, with $\gamma_{s,t+1}\leq \frac{1}{16C_6}$, we have $\frac{8 B_1 C_6 \gamma_{s,t+1}^2}{n\eta_t}\leq\frac{\gamma_{s,t+1}B_1}{2n\eta_t}$. Then we can set $\gamma_{s,t+1}=\left(\frac{2n}{7B_1 L_F c^3}+\frac{4C_2\alpha n}{B_1}\right)\eta_t^2$, and $\gamma_{s,t+1}\leq\min\{\frac{1}{2},\frac{1}{16C_6}\}$ can be achieved by setting $c_0\geq\max\left\{\left(\frac{8n}{7L_F B_1 c}\right)^{3/2},\left(\frac{16C_2\alpha n c^2}{B_1}\right)^{3/2},\left(\frac{64nC_6}{7L_F B_1 c}\right)^{3/2},\left(\frac{128C_2 n C_6\alpha c^2}{B_1}\right)^{3/2}\right\}$.

Then we aim to eliminate $\E\left[\delta_{m,t} \right]$. To this end, we can set $\gamma_{m,t+1}=\left(\frac{1}{7L_F c^3}+\alpha\right)\eta_t^2$ and $\gamma_{m,t+1}\leq \frac{1}{2}$ can be achieved by setting $c_0\geq\max\left\{\left(\frac{4}{7L_F c}\right)^{3/2},\left(4\alpha c^2\right)^{3/2} \right\}$. Afterwards, with $L_L\leq\frac{1}{2\tau}$, we have $\frac{C_{\text{max}}}{B_1\eta_t} -\frac{3C\alpha\tau B_1}{\tau_t n^2}\left(\frac{1}{\tau}-L_L\right)\leq \frac{C_{\text{max}}}{B_1\eta_t} -\frac{3C\alpha B_1}{2\tau_t n^2}$. By setting $C\geq\frac{2C_{\text{max}}n^2}{3\alpha B_1^2}$, we have $\frac{C_{\text{max}}}{B_1\eta_t} -\frac{3C\alpha B_1}{2\tau_t n^2}\leq 0$. Last, with $\alpha\leq\min\left\{\frac{B_1}{8C_{\text{max}}n},\frac{1}{8C_\lambda}\sqrt{\frac{\tau\mu_L B_1}{Cn}} \right\}$, we have $\left(\frac{8C\alpha n C_{\lambda}^2}{\tau\tau_t\mu_L B_1} + \frac{C_{\text{max}} n}{B_1\eta_t} \right)\alpha^2\eta_t^2 - \frac{\alpha\eta_t}{4}\leq 0$. As a result, we have
\begin{equation*}
\begin{aligned}
	& \E\left[\frac{\delta_{L\lambda,t+1}}{n\eta_t} - \frac{\delta_{L\lambda,t}}{n\eta_{t-1}} \right] + \E\left[\frac{C\alpha}{n}(\delta_{\lambda,t+1}-\delta_{\lambda,t}) \right] + \E\left[\frac{\delta_{g,t+1}}{n\eta_t}-\frac{\delta_{g,t}}{n\eta_{t-1}} \right] \\
	& + \E\left[\frac{\delta_{L\lambda\lambda,t+1}}{n\eta_t}-\frac{\delta_{L\lambda\lambda,t}}{n\eta_{t-1}} \right] + \E\left[\frac{\delta_{m,t+1}}{\eta_t} - \frac{\delta_{m,t}}{\eta_{t-1}} \right] + \frac{\alpha\eta_t}{2}\left\|\nabla F(\w_t) \right\|^2 \\
	\leq & F(\w_t) - F(\w_{t+1}) + \O(1)\eta_t^3
\end{aligned}
\end{equation*}

Take summation over $t=1,2,\cdots,T$, we have
\begin{equation}
	\E\left[\sum_{t=1}^T\frac{\alpha\eta_t}{2}\left\|\nabla F(\w_t) \right\|^2\right] \leq F(\w_1)-F(\w_{T+1}) + \frac{1}{n\eta_1}\E\left[\delta_{L\lambda,1}+\delta_{g,1}+\delta_{L\lambda\lambda,1}+n\delta_{m,1}\right] + \frac{C\alpha}{n}\E\left[\delta_{\lambda,1}\right] +\O(\log(T+1)).
\end{equation}

Denote $M=F(\w_1)-F(\w_{T+1}) + \frac{1}{n\eta_1}\E\left[\delta_{L\lambda,1}+\delta_{g,1}+\delta_{L\lambda\lambda,1}+n\delta_{m,1}\right] + \frac{C\alpha}{n}\E\left[\delta_{\lambda,1}\right] +\O(\log(T+1))$, then we have 
\begin{equation}
	\E\left[\sum_{t=1}^T\frac{\alpha}{2T}\left\|\nabla F(\w_t) \right\|^2\right] \leq \frac{M}{\eta_T T}.
\end{equation}

Note that $\eta_T=\frac{c}{(c_0+T)^{1/3}}$, so $\frac{M}{\eta_T T}=\frac{M}{T}\frac{(c_0+T)^{1/3}}{c}\leq \frac{M c_0^{1/3}}{Tc} + \frac{M T^{1/3}}{Tc}\sim\O\left(\frac{1}{T^{2/3}}\right)$, where the inequality is due to $(a+b)^{1/3}\leq a^{1/3} + b^{1/3}$, thus we have
\begin{equation}
	\E\left[\sum_{t=1}^T\frac{1}{T}\left\|\nabla F(\w_t) \right\|^2\right] \leq \O\left(\frac{1}{T^{2/3}}\right).
\end{equation}

\end{proof}

\subsection{Proofs of Lemmas}

\subsubsection{Proof of Lemma~\ref{lemma_decomp}}

\begin{proof}

First, we have

\begin{equation}\label{eq:temp_0}
\begin{aligned}
&\left\|\nabla F\left(\w_{t}\right)-\m_{t}\right\|^{2} \\
=&\left\|\nabla F\left(\w_{t}\right)-\frac{1}{n} \sum_{i \in \S} G_{i}\left(\w_{t}\right)+\frac{1}{n} \sum_{i \in \S} G_{i}\left(\w_{t}\right)-\m_{t}\right\|^{2} \\
\leq & 2\left\|\nabla F\left(\w_{t}\right)-\frac{1}{n} \sum_{i \in \S} G_{i}\left(\w_{t}\right)\right\|^{2}+2\left\|\frac{1}{n} \sum_{i \in \S} G_{i}\left(\w_{t}\right)-\m_{t}\right\|^{2} \\
= & 2\left\| \frac{1}{n} \sum_{i \in \S} \nabla  F_i\left(\w_{t}\right)-\frac{1}{n} \sum_{i \in \S} G_{i}\left(\w_{t}\right)\right\|^{2}+2\left\|\frac{1}{n} \sum_{i \in \S} G_{i}\left(\w_{t}\right)-\m_{t}\right\|^{2} \\
\leq & 2 \frac{1}{n} \sum_{i \in \S}\left\|\nabla F_{i}\left(\w_{t}\right)-G_{i}\left(\w_{t}\right)\right\|^{2}+2\left\|\frac{1}{n} \sum_{i \in \S} G_{i}\left(\w_{t}\right)-\m_{t}\right\|^{2}.
\end{aligned}
\end{equation}

In order to bound $\left\|\nabla F_{i}\left(\w_{t}\right)-G_{i}\left(\w_{t}\right)\right\|^{2}$, we introduce $\nabla F_i (\w_t, \lambda_i^t)$ and we have

\begin{equation}\label{eq:temp_1}
\begin{aligned}
	&\left\|\nabla F_{i}\left(\w_{t}\right)-\nabla F_{i}\left(\w_{t}, \lambda_{i}^{t}\right)\right\|^{2}\\
	\leq & \left\| \left[\nabla_\w \psi_i (\w_t, \lambda_i(\w_t)) - \nabla_{\w\lambda}^2 L_i (\w,\lambda_i(\w_t))[\nabla_{\lambda\lambda}^2 L_i(\w_t,\lambda_i(\w_t))]^{-1}\nabla_\lambda \psi_i(\w,\lambda_i(\w_t))\right] f_i(g_i(\w_t)) \right. \\ 
	& \left. + \psi_i(\w_t,\lambda_i(\w_t))\nabla g_i(\w_t) \nabla f_i(g_i(\w_t)) - \psi_i(\w_t,\lambda_i^t)\nabla g_i(\w_t) \nabla f_i(g_i(\w_t)) \right. \\
	& \left. -\left[\nabla_\w \psi_i (\w_t, \lambda_i^t) - \nabla_{\w\lambda}^2 L_i (\w,\lambda_i^t)[\nabla_{\lambda\lambda}^2 L_i(\w_t,\lambda_i^t)]^{-1}\nabla_\lambda \psi_i(\w,\lambda_i^t)\right]f_i(g_i(\w_t)) \right\|^{2} \\
	\leq & 3\left\| \nabla_\w \psi_i (\w_t, \lambda_i(\w_t)) f_i(g_i(\w_t)) - \nabla_\w \psi_i (\w_t, \lambda_i^t) f_i(g_i(\w_t)) \right\|^{2} \\
	& + 3\left\| \nabla_{\w\lambda}^2 L_i (\w,\lambda_i(\w_t))[\nabla_{\lambda\lambda}^2 L_i(\w_t,\lambda_i(\w_t))]^{-1}\nabla_\lambda \psi_i(\w,\lambda_i(\w_t)) f_i(g_i(\w_t)) \right. \\
	 &\quad\quad \left. - \nabla_{\w\lambda}^2 L_i (\w,\lambda_i^t)[\nabla_{\lambda\lambda}^2 L_i(\w_t,\lambda_i^t)]^{-1}\nabla_\lambda \psi_i(\w,\lambda_i^t)f_i(g_i(\w_t)) \right\|^{2} \\
	& + 3\left\| \psi_i(\w_t,\lambda_i(\w_t))\nabla g_i(\w_t) \nabla f_i(g_i(\w_t)) - \psi_i(\w_t,\lambda_i^t)\nabla g_i(\w_t) \nabla f_i(g_i(\w_t)) \right\|^{2}\\
	\leq & 3 L_{\psi}^2 B_f^2 \left\| \lambda_i(\w_t)-\lambda_i^t \right\|^2 + 9\frac{L_{L\w\lambda}^2 C_{\psi}^2 B_f^2}{\gamma^2} \left\| \lambda_i(\w_t)-\lambda_i^t \right\|^2 + 9\frac{C_{L\w\lambda}^2 L_{L\lambda\lambda}^2 C_{\psi}^2 B_f^2}{\gamma^4} \left\| \lambda_i(\w_t)-\lambda_i^t \right\|^2 \\
	& +  9\frac{C_{L\w\lambda}^2 L_{\psi}^2 B_f^2}{\gamma^2} \left\| \lambda_i(\w_t)-\lambda_i^t \right\|^2 + 3 C_{g}^2 C_{f}^2 C_{\psi}^2 \left\| \lambda_i(\w_t)-\lambda_i^t \right\|^2 \\
	= & \underbrace{\left(3 L_{\psi}^2 B_f^2+9\frac{L_{L\w\lambda}^2 C_{\psi}^2 B_f^2}{\gamma^2}+9\frac{C_{L\w\lambda}^2 L_{L\lambda\lambda}^2 C_{\psi}^2 B_f^2}{\gamma^4}+9\frac{C_{L\w\lambda}^2 L_{\psi}^2 B_f^2}{\gamma^2}+ 3 C_{g}^2 C_{f}^2 C_{\psi}^2\right)}_{C_0} \left\| \lambda_i(\w_t)-\lambda_i^t \right\|^2,
\end{aligned}
\end{equation}
where we use the conditions in Assumption~\ref{assump_ana} in the last inequality. Next we will bound $\left\|\nabla F_{i}\left(\w_{t}, \lambda_{i}^{t}\right)-G_{i}\left(\w_{t}\right)\right\|^{2}$

\begin{equation}\label{eq:temp_2}
\begin{aligned}
	&\left\|\nabla F_{i}\left(\w_{t}, \lambda_{i}^{t}\right)-G_{i}\left(\w_{t}\right)\right\|^{2} \\
	= & \left\| \left[\nabla_\w \psi_i (\w_t, \lambda_i^t) - \nabla_{\w\lambda}^2 L_i (\w,\lambda_i^t)[\nabla_{\lambda\lambda}^2 L_i(\w_t,\lambda_i^t)]^{-1}\nabla_\lambda \psi_i(\w,\lambda_i^t)\right] f_i(g_i(\w_t)) \right. \\ 
	& \left. + \psi_i(\w_t,\lambda_i^t)\nabla g_i(\w_t) \nabla f_i(g_i(\w_t)) - \psi_i(\w_t,\lambda_i^t)\nabla g_i(\w_t) \nabla f_i(u_i^t)) \right. \\
	& \left. -\left[\nabla_\w \psi_i (\w_t, \lambda_i^t) - \nabla_{\w\lambda}^2 L_i (\w,\lambda_i^t)[s_i^t]^{-1}\nabla_\lambda \psi_i(\w,\lambda_i^t)\right]f_i(u_i^t) \right\|^{2} \\
	\leq & 3 \left\| \nabla_\w \psi_i (\w_t, \lambda_i^t)f_i(g_i(\w_t)) - \nabla_\w \psi_i (\w_t, \lambda_i^t)f_i(u_i^t)) \right\|^2 \\
	& + 3 \left\|  \psi_i(\w_t,\lambda_i^t)\nabla g_i(\w_t) \nabla f_i(g_i(\w_t)) - \psi_i(\w_t,\lambda_i^t)\nabla g_i(\w_t) \nabla f_i(u_i^t)) \right\|^2 \\
	& + 3 \left\|  \nabla_{\w\lambda}^2 L_i (\w,\lambda_i^t)[\nabla_{\lambda\lambda}^2 L_i(\w_t,\lambda_i^t)]^{-1}\nabla_\lambda \psi_i(\w,\lambda_i^t) f_i(g_i(\w_t))  -\nabla_{\w\lambda}^2 L_i (\w,\lambda_i^t)[s_i^t]^{-1}\nabla_\lambda \psi_i(\w,\lambda_i^t) f_i(u_i^t) \right\|^2 \\
	\leq & \underbrace{\left(3 C_{\psi}^2 C_f^2 + 3 B_{\psi}^2 C_g^2 L_f^2 + 6\frac{C_{L\w\lambda}^2 C_{\psi}^2 C_f^2}{\gamma^2}\right)}_{C_1} \left\| u_i^t - g_i(\w_t) \right\|^2 + \underbrace{6 \frac{C_{L\w\lambda}^2 C_{\psi}^2 B_f^2}{\gamma^4}}_{C_2} \left\| s_i^t - \nabla_{\lambda\lambda}^2 L_i(\w_t,\lambda_i^t) \right\|^2.
\end{aligned}	
\end{equation}

Thus by combining (\ref{eq:temp_1}) and (\ref{eq:temp_2}), we have

\begin{equation}\label{eq:temp_3}
	\begin{aligned}
		&\left\|\nabla F_{i}\left(\w_{t}\right)-G_{i}\left(\w_{t}\right)\right\|^{2} \\
		\leq & 2\left\|\nabla F_{i}\left(\w_{t}\right)-\nabla F_{i}\left(\w_{t}, \lambda_{i}^{t}\right)\right\|^{2} + 2\left\|\nabla F_{i}\left(\w_{t}, \lambda_{i}^{t}\right)-G_{i}\left(\w_{t}\right)\right\|^{2} \\
		\leq & 2 C_0  \left\| \lambda_i(\w_t)-\lambda_i^t \right\|^2 + 2 C_1 \left\| u_i^t - g_i(\w_t) \right\|^2 + 2 C_2 \left\| s_i^t - \nabla_{\lambda\lambda}^2 L_i(\w_t,\lambda_i^t) \right\|^2, 
	\end{aligned}
\end{equation}
where $C_0=3 L_{\psi}^2 B_f^2+9\frac{L_{L\w\lambda}^2 C_{\psi}^2 B_f^2}{\gamma^2}+9\frac{C_{L\w\lambda}^2 L_{L\lambda\lambda}^2 C_{\psi}^2 B_f^2}{\gamma^4}+9\frac{C_{L\w\lambda}^2 L_{\psi}^2 B_f^2}{\gamma^2}+ 3 C_{g}^2 C_{f}^2 C_{\psi}^2$, $C_1=3 C_{\psi}^2 C_f^2 + 3 B_{\psi}^2 C_g^2 L_f^2 + 6\frac{C_{L\w\lambda}^2 C_{\psi}^2 C_f^2}{\gamma^2}$, $C_2=6 \frac{C_{L\w\lambda}^2 C_{\psi}^2 B_f^2}{\gamma^4}$. As a result, combining (\ref{eq:temp_0}) and (\ref{eq:temp_3}), we have

\begin{equation*}
\begin{aligned}
&\left\|\nabla F\left(\w_{t}\right)-\m_{t}\right\|^{2} \\
\leq & 2\left\|\frac{1}{n} \sum_{i \in \S} G_{i}\left(\w_{t}\right)-\m_{t}\right\|^{2}+  \frac{4 C_0}{n} \left\| \lambda(\w_t)-\lambda^t \right\|^2 \\&+  \frac{4 C_1}{n} \left\| u^t - g(\w_t) \right\|^2 + \frac{4 C_2}{n}  \left\| s^t - \nabla_{\lambda\lambda}^2 L(\w_t,\lambda^t) \right\|^2.
\end{aligned}
\end{equation*}

\end{proof}

\subsubsection{Proof of Lemma~\ref{lemma_decomp_2}}

\begin{proof}

\begin{equation*}
\begin{aligned}
	& \E\left[ \left\| \frac{1}{n}\sum_{i \in \S} G_i(\w_t)-\m_t \right\|^2\right] = \E\left[ \left\| \m_t-\frac{1}{n}\sum_{i \in \S} G_i(\w_t) \right\|^2\right] \\
	=& \E\left[ \left\| (1-\gamma_{m,t})\left(\m_{t-1}-\frac{1}{B_1}\sum_{i\in\B_1^t}G_i(\w_{t-1})\right) + \frac{1}{B_1}\sum_{i\in\B_1^t}G_i(\w_{t}) - \frac{1}{n}\sum_{i\in\S}G_i(\w_{t}) \right\|^2\right] \\
	=& \E\left[ \left\| (1-\gamma_{m,t})\left(\m_{t-1}-\frac{1}{n}\sum_{i\in\S}G_i(\w_{t-1})\right) + \gamma_{m,t} \left(\frac{1}{B_1}\sum_{i\in\B_1^t}G_i(\w_{t}) - \frac{1}{n}\sum_{i\in\S}G_i(\w_{t})\right) \right.\right. \\
	 &\quad \left.\left. +(1-\gamma_{m,t})\left(\frac{1}{B_1}\sum_{i\in\B_1^t}G_i(\w_{t}) - \frac{1}{B_1}\sum_{i\in\B_1^t}G_i(\w_{t-1}) - \frac{1}{n}\sum_{i\in\S}G_i(\w_{t}) +\frac{1}{n}\sum_{i\in\S}G_i(\w_{t-1})\right) \right\|^2\right].
\end{aligned}
\end{equation*}

We assume that $\E\left[\frac{1}{B_1}\sum_{i\in\B_1^t}G_i(\w_t)-\frac{1}{n}\sum_{i\in \S}G_i(\w_t) \right]\leq \sigma^2$. Due to the fact that the expectation over the last two terms equals to zero, we have

\begin{equation}\label{eq:713_1}
\begin{aligned}
	&  \E\left[ \left\| \m_t-\frac{1}{n}\sum_{i \in \S} G_i(\w_t) \right\|^2\right] \\
	\leq & \E\left[ (1-\gamma_{m,t})^2 \left\| \m_{t-1}-\frac{1}{n}\sum_{i\in\S}G_i(\w_{t-1}) \right\|^2 + 2\gamma_{m,t}^2\sigma^2 + 2(1-\gamma_{m,t})^2\frac{1}{B_1}\sum_{i\in\B_1^t}\left\| G_i(\w_t) - G_i(\w_{t-1}) \right\|^2 \right]. \\
\end{aligned}
\end{equation}

Next, we will bound $\left\| G_i(\w_t) - G_i(\w_{t-1})\right\|^2$

\begin{equation}\label{eq:713_2}
\begin{aligned}
	& \left\| G_i(\w_t) - G_i(\w_{t-1}) \right\|^2 \\
	= & \left\|  \left[\nabla_\w\psi_i(\w_t,\lambda_i^t)-\nabla_{\w \lambda}^2 L_i(\w_t,\lambda_i^t;\B_{2,i}^t) [s_i^{t}]^{-1} \nabla_\lambda \psi_i(\w_t,\lambda_i^t)\right]f_i(u_i^t)+\psi_i(\w_t,\lambda_i^t)\nabla g_i(\w_t;\B_{2,i}^t)\nabla f_i( u_i^t) \right. \\
	&\quad \left. -\left[\nabla_\w\psi_i(\w_{t-1},\lambda_i^{t-1})-\nabla_{\w \lambda}^2 L_i(\w_{t-1},\lambda_i^{t-1};\B_{2,i}^t) [s_i^{t-1}]^{-1} \nabla_\lambda \psi_i(\w_{t-1},\lambda_i^{t-1})\right]f_i(u_i^{t-1}) \right. \\ 
	&\quad \left. - \psi_i(\w_{t-1},\lambda_i^{t-1})\nabla g_i(\w_{t-1};\B_{2,i}^t)\nabla f_i( u_i^{t-1}) \right\|^2 \\
	\leq & 3 \left\| \nabla_\w\psi_i(\w_t,\lambda_i^t)f_i(u_i^t) - \nabla_\w\psi_i(\w_{t-1},\lambda_i^{t-1})f_i(u_i^{t-1}) \right\|^2 \\
	& + 3 \left\| \nabla_{\w \lambda}^2 L_i(\w_t,\lambda_i^t;\B_{2,i}^t) [s_i^{t}]^{-1} \nabla_\lambda \psi_i(\w_t,\lambda_i^t)f_i(u_i^t) - \nabla_{\w \lambda}^2 L_i(\w_{t-1},\lambda_i^{t-1};\B_{2,i}^t) [s_i^{t-1}]^{-1} \nabla_\lambda \psi_i(\w_{t-1},\lambda_i^{t-1})f_i(u_i^{t-1}) \right\|^2 \\
	& + 3 \left\| \psi_i(\w_t,\lambda_i^t)\nabla g_i(\w_t;\B_{2,i}^t)\nabla f_i( u_i^t) - \psi_i(\w_{t-1},\lambda_i^{t-1})\nabla g_i(\w_{t-1};\B_{2,i}^t)\nabla f_i( u_i^{t-1})\right\|^2 \\
	\leq & 6 C_{\psi}^2 C_{f}^2 \left\| u_i^t - u_i^{t-1} \right\|^2 + 12 B_f^2 L_{\psi}^2 \left\| \w_t - \w_{t-1} \right\|^2 +12 B_f^2 L_{\psi}^2 \left\| \lambda_i^t - \lambda_i^{t-1} \right\|^2 \\
	& + 12 \frac{C_{\psi}^2 B_f^2}{\gamma^2} L_{L\w\lambda}^2 (2\left\| \w_t - \w_{t-1} \right\|^2+2\left\| \lambda_i^t - \lambda_i^{t-1} \right\|^2) + 12 \frac{C_{L\w\lambda}^2 C_{\psi}^2 B_f^2}{\gamma^4} \left\| s_i^t - s_i^{t-1} \right\|^2 \\
	& + 12 \frac{L_{\psi}^2 B_f^2}{\gamma^2} C_{L\w\lambda}^2 (2\left\| \w_t - \w_{t-1} \right\|^2+2\left\| \lambda_i^t - \lambda_i^{t-1} \right\|^2) + 12 \frac{C_{L\w\lambda}^2 C_{\psi}^2 C_f^2}{\gamma^2} \left\| u_i^t - u_i^{t-1} \right\|^2 \\
	& + 9 C_g^2 C_f^2 C_{\psi}^2 (2\left\| \w_t - \w_{t-1} \right\|^2+2\left\| \lambda_i^t - \lambda_i^{t-1} \right\|^2) + 9 B_{\psi}^2 C_f^2 L_g^2 \left\| \w_t - \w_{t-1} \right\|^2 + 9 B_{\psi}^2 C_g^2 L_f^2 \left\| u_i^t - u_i^{t-1} \right\|^2 \\ 
	= & \underbrace{\left(6 C_{\psi}^2 C_{f}^2 + 12 \frac{C_{L\w\lambda}^2 C_{\psi}^2 C_f^2}{\gamma^2} + 9 B_{\psi}^2 C_g^2 L_f^2\right)}_{C_3} \left\| u_i^t - u_i^{t-1} \right\|^2 \\
	& + \underbrace{\left(12 B_f^2 L_{\psi}^2 + 24 \frac{C_{\psi}^2 B_f^2}{\gamma^2} L_{L\w\lambda}^2 + 24 \frac{L_{\psi}^2 B_f^2}{\gamma^2} C_{L\w\lambda}^2 + 18 C_g^2 C_f^2 C_{\psi}^2 + 9 B_{\psi}^2 C_f^2 L_g^2 \right)}_{C_4} \left\| \w_t - \w_{t-1} \right\|^2 \\
	& + \underbrace{\left(12 B_f^2 L_{\psi}^2 + 24 \frac{C_{\psi}^2 B_f^2}{\gamma^2} L_{L\w\lambda}^2 + 24 \frac{L_{\psi}^2 B_f^2}{\gamma^2} C_{L\w\lambda}^2 + 18 C_g^2 C_f^2 C_{\psi}^2 \right)}_{C_5} \left\| \lambda_i^t - \lambda_i^{t-1} \right\|^2 \\
	& + \underbrace{\left(12\frac{C_{L\w\lambda}^2 C_{\psi}^2 B_f^2}{\gamma^4} \right)}_{C_6} \left\| s_i^t - s_i^{t-1} \right\|^2.
\end{aligned}
\end{equation}

Then combine (\ref{eq:713_1}) and (\ref{eq:713_2}) and we have

\begin{equation*}
\begin{aligned}
	&  \E\left[ \left\| \m_t-\frac{1}{n}\sum_{i \in \S} G_i(\w_t) \right\|^2\right] \\
	\leq & \E\left[ (1-\gamma_{m,t})^2 \left\| \m_{t-1}-\frac{1}{n}\sum_{i\in\S}G_i(\w_{t-1}) \right\|^2 + 2\gamma_{m,t}^2\sigma^2 + 2(1-\gamma_{m,t})^2\frac{1}{B_1}\sum_{i\in\B_1^t}\left\| G_i(\w_t) - G_i(\w_{t-1}) \right\|^2 \right] \\
	\leq & (1-\gamma_{m,t})\E\left[ \left\| \m_{t-1}-\frac{1}{n}\sum_{i\in\S}G_i(\w_{t-1}) \right\|^2\right] + 2\gamma_{m,t}^2\sigma^2 + \frac{2(1-\gamma_{m,t})^2 C_3}{n} \left\| u^t - u^{t-1} \right\|^2 \\
	& \quad + 2(1-\gamma_{m,t})^2 C_4 \left\| \w_t - \w_{t-1} \right\|^2 + \frac{2(1-\gamma_{m,t})^2 C_5}{n} \left\| \lambda^t - \lambda^{t-1} \right\|^2 + \frac{2(1-\gamma_{m,t})^2 C_6}{n} \left\| s^t - s^{t-1} \right\|^2 \\
	\leq & (1-\gamma_{m,t}) \E\left[\left\| \m_{t-1}-\frac{1}{n}\sum_{i\in\S}G_i(\w_{t-1}) \right\|^2 \right]+ 2\gamma_{m,t}^2\sigma^2 + \frac{2 C_3}{n} \left\| u^t - u^{t-1} \right\|^2 \\
	& \quad + 2 C_4 \left\| \w_t - \w_{t-1} \right\|^2 + \frac{2 C_5}{n} \left\| \lambda^t - \lambda^{t-1} \right\|^2 + \frac{2 C_6}{n} \left\| s^t - s^{t-1} \right\|^2.
\end{aligned}
\end{equation*}

\end{proof}

\subsubsection{Proof of Lemma~\ref{lemma_fksong_lambda}}

\begin{proof}
Recall
\begin{equation*}
\begin{aligned}
& \lambda_i^{t+1}=\begin{cases} \lambda_i^t - \tau\tau_t z_i^t  \quad & \text{if }i\in \B_1^t\\   \lambda_i^t & \text{o.w.} \end{cases} \\
& z_i^{t}=\begin{cases} (1-\gamma_{z,t})(z_i^{t-1}-\nabla_{\lambda} L_i (\w_{t-1}, \lambda_i^{t-1};\B_{2,i}^t)) + \nabla_{\lambda} L_i (\w_t, \lambda_i^t;\B_{2,i}^t)  \quad & \text{if }i\in \B_1^t\\   z_i^{t-1} & \text{o.w.} \end{cases},
\end{aligned}
\end{equation*}

and define the following notations

\begin{equation*}
\tilde{\lambda}_i^t = \lambda_i^t -\tau z_i^t, \quad \bar{\lambda}_i^t = \lambda_i^t + \tau_t(\tilde{\lambda}_i^t - \lambda_i^t) \quad \text{if }i\in \B_1^t.
\end{equation*}

Note that

\begin{equation*}
\begin{aligned}
	& \left\|\bar{\lambda}_i^t - \lambda_i(\w_t) \right\|^2 \\
	=&\left\| \lambda_i^t + \tau_t(\tilde{\lambda}_i^t - \lambda_i^t) - \lambda_i(\w_t)\right\|^2 \\ 
	=&\left\|\lambda_i^t - \lambda_i(\w_t) \right\|^2 + \tau_t^2 \left\|\tilde{\lambda}_i^t - \lambda_i^t \right\|^2 + 2\tau_t(\lambda_i^t - \lambda_i(\w_t))(\tilde{\lambda}_i^t - \lambda_i^t). \\
\end{aligned}
\end{equation*}

As a result,

\begin{equation}\label{eq:fksong_lambda_lemma1}
	(\lambda_i^t - \lambda_i(\w_t))(\tilde{\lambda}_i^t - \lambda_i^t) =\frac{1}{2\tau_t}\left(\left\|\bar{\lambda}_i^t - \lambda_i(\w_t) \right\|^2- \left\|\lambda_i^t - \lambda_i(\w_t) \right\|^2 -\tau_t^2 \left\|\tilde{\lambda}_i^t - \lambda_i^t \right\|^2 \right).
\end{equation}

Due to smoothness of $L_i$, we have
\begin{equation*}
	L_i(\w_t,\tilde{\lambda}_i^t)\leq L_i(\w_t,\lambda_i^t) + \nabla_{\lambda}L_i(\w_t,\lambda_i^t)(\tilde{\lambda}_i^t - \lambda_i^t) + \frac{L_L}{2}\left\|\tilde{\lambda}_i^t-\lambda_i^t \right\|^2 .
\end{equation*}

Hence
\begin{equation*}
	L_i(\w_t,\tilde{\lambda}_i^t)-\nabla_{\lambda}L_i(\w_t,\lambda_i^t)(\tilde{\lambda}_i^t - \lambda_i^t)-\frac{L_L}{2}\left\|\tilde{\lambda}_i^t-\lambda_i^t \right\|^2 \leq L_i(\w_t,\lambda_i^t) .
\end{equation*}

Due to strong convexity of $L_i$, we have
\begin{equation*}
\begin{aligned}
	L_i(\w_t,\lambda) & \geq L_i(\w_t,\lambda_i^t) + \nabla_{\lambda}L_i(\w_t,\lambda_i^t)(\lambda-\lambda_i^t) + \frac{\mu_L}{2}\left\|\lambda-\lambda_i^t \right\|^2 \\
	& =  L_i(\w_t,\lambda_i^t) + \nabla_{\lambda}L_i(\w_t,\lambda_i^t)(\lambda-\tilde{\lambda}_i^t) +  \nabla_{\lambda}L_i(\w_t,\lambda_i^t)(\tilde{\lambda}_i^t-\lambda_i^t) + \frac{\mu_L}{2}\left\|\lambda-\lambda_i^t \right\|^2 \\
	& = L_i(\w_t,\lambda_i^t) + z_i^t(\lambda-\tilde{\lambda}_i^t) + (\nabla_{\lambda}L_i(\w_t,\lambda_i^t)-z_i^t)(\lambda-\tilde{\lambda}_i^t) \\
	&\quad + \nabla_{\lambda}L_i(\w_t,\lambda_i^t)(\tilde{\lambda}_i^t-\lambda_i^t) + \frac{\mu_L}{2}\left\|\lambda-\lambda_i^t \right\|^2
\end{aligned}
\end{equation*}

Combining the above inequalities, we have
\begin{equation*}
\begin{aligned}
	L_i(\w_t,\lambda) & \geq L_i(\w_t,\tilde{\lambda}_i^t) + z_i^t(\lambda-\tilde{\lambda}_i^t) + (\nabla_{\lambda}L_i(\w_t,\lambda_i^t)-z_i^t)(\lambda-\tilde{\lambda}_i^t)  \\
	&\quad + \frac{\mu_L}{2}\left\|\lambda-\lambda_i^t \right\|^2 - \frac{L_L}{2}\left\|\tilde{\lambda}_i^t-\lambda_i^t \right\|^2.
\end{aligned}
\end{equation*}

Note that
\begin{equation*}
\begin{aligned}
	z_i^t(\lambda-\tilde{\lambda}_i^t) &= \frac{1}{\tau}(\lambda_i^t-\tilde{\lambda}_i^t)(\lambda-\tilde{\lambda}_i^t) = \frac{1}{\tau}(\lambda_i^t-\tilde{\lambda}_i^t)(\lambda-\lambda_i^t) + \frac{1}{\tau}(\lambda_i^t-\tilde{\lambda}_i^t)(\lambda_i^t-\tilde{\lambda}_i^t) \\
	& = \frac{1}{\tau}(\lambda_i^t-\tilde{\lambda}_i^t)(\lambda-\lambda_i^t) + \frac{1}{\tau}\left\|\lambda_i^t-\tilde{\lambda}_i^t\right\|^2
\end{aligned}
\end{equation*}

Then we obtain
\begin{equation*}
\begin{aligned}
	L_i(\w_t,\lambda) & \geq L_i(\w_t,\tilde{\lambda}_i^t) + \frac{1}{\tau}(\lambda_i^t-\tilde{\lambda}_i^t)(\lambda-\lambda_i^t) + \frac{1}{\tau}\left\|\lambda_i^t-\tilde{\lambda}_i^t\right\|^2 + (\nabla_{\lambda}L_i(\w_t,\lambda_i^t)-z_i^t)(\lambda-\tilde{\lambda}_i^t) \\
	&\quad + \frac{\mu_L}{2}\left\|\lambda-\lambda_i^t \right\|^2 - \frac{L_L}{2}\left\|\tilde{\lambda}_i^t-\lambda_i^t \right\|^2.
\end{aligned}
\end{equation*}
	
Thus, combining the above inequality with (\ref{eq:fksong_lambda_lemma1}), we have
\begin{equation*}
\begin{aligned}
	L_i(\w_t,\tilde{\lambda}_i^t) & \geq L_i(\w_t, \lambda_i(\w_t)) \\
	& \geq L_i(\w_t,\tilde{\lambda}_i^t) + \frac{1}{\tau}(\lambda_i^t-\tilde{\lambda}_i^t)(\lambda_i(\w_t)-\lambda_i^t) + \frac{1}{\tau}\left\|\lambda_i^t-\tilde{\lambda}_i^t\right\|^2 + (\nabla_{\lambda}L_i(\w_t,\lambda_i^t)-z_i^t)(\lambda_i(\w_t)-\tilde{\lambda}_i^t) \\
	&\quad + \frac{\mu_L}{2}\left\|\lambda_i(\w_t)-\lambda_i^t \right\|^2 - \frac{L_L}{2}\left\|\tilde{\lambda}_i^t-\lambda_i^t \right\|^2 \\
	& \geq L_i(\w_t,\tilde{\lambda}_i^t) + \frac{1}{\tau}(\lambda_i^t-\tilde{\lambda}_i^t)(\lambda_i(\w_t)-\lambda_i^t) + \frac{1}{\tau}\left\|\lambda_i^t-\tilde{\lambda}_i^t\right\|^2 \\
	&\quad -\frac{2}{\mu_L}\left\| \nabla_{\lambda}L_i(\w_t,\lambda_i^t)-z_i^t \right\|^2 -\frac{\mu_L}{4} \left\|\lambda_i(\w_t)-\lambda_i^t \right\|^2 -\frac{\mu_L}{4}  \left\|\lambda_i^t-\tilde{\lambda}_i^t \right\|^2 \\
	&\quad + \frac{\mu_L}{2}\left\|\lambda_i(\w_t)-\lambda_i^t \right\|^2 - \frac{L_L}{2}\left\|\tilde{\lambda}_i^t-\lambda_i^t \right\|^2 \\
	&\geq L_i(\w_t,\tilde{\lambda}_i^t) + \frac{1}{2\tau_t\tau}\left(\left\|\bar{\lambda}_i^t - \lambda_i(\w_t) \right\|^2- \left\|\lambda_i^t - \lambda_i(\w_t) \right\|^2 -\tau_t^2 \left\|\tilde{\lambda}_i^t - \lambda_i^t \right\|^2 \right) \\
	&\quad -\frac{2}{\mu_L}\left\| \nabla_{\lambda}L_i(\w_t,\lambda_i^t)-z_i^t \right\|^2 + \frac{\mu_L}{4} \left\|\lambda_i(\w_t)-\lambda_i^t \right\|^2 + \left(\frac{1}{\tau}-\frac{\mu_L}{4}-\frac{L_L}{2} \right)\left\|\tilde{\lambda}_i^t - \lambda_i^t \right\|^2
\end{aligned}
\end{equation*}

Hence we have
\begin{equation*}
\begin{aligned}
	\left\|\bar{\lambda}_i^t - \lambda_i(\w_t) \right\|^2 & \leq \frac{4\tau_t\tau}{\mu_L}\left\| \nabla_{\lambda}L_i(\w_t,\lambda_i^t)-z_i^t \right\|^2 + \left(1-\frac{\tau\tau_t\mu_L}{2}\right) \left\|\lambda_i(\w_t)-\lambda_i^t \right\|^2 \\
	&\quad -2\tau\tau_t \left(\frac{1}{\tau}-\frac{\mu_L}{4}-\frac{L_L}{2} -\frac{\tau_t}{2\tau}\right)\left\|\tilde{\lambda}_i^t - \lambda_i^t \right\|^2 \\
	& \leq  \left(1-\frac{\tau\tau_t\mu_L}{2}\right) \left\|\lambda_i(\w_t)-\lambda_i^t \right\|^2 + \frac{4\tau_t\tau}{\mu_L}\left\| \nabla_{\lambda}L_i(\w_t,\lambda_i^t)-z_i^t \right\|^2 - 2\tau\tau_t \left(\frac{3}{4\tau}-\frac{3}{4}L_L\right)\left\|\tilde{\lambda}_i^t - \lambda_i^t \right\|^2 \\
	& \leq  \left(1-\frac{\tau\tau_t\mu_L}{2}\right) \left\|\lambda_i(\w_t)-\lambda_i^t \right\|^2 + \frac{4\tau_t\tau}{\mu_L}\left\| \nabla_{\lambda}L_i(\w_t,\lambda_i^t)-z_i^t \right\|^2 - 2\frac{\tau}{\tau_t} \left(\frac{3}{4\tau}-\frac{3}{4}L_L\right)\left\|\bar{\lambda}_i^t - \lambda_i^t \right\|^2,
\end{aligned}
\end{equation*}
where we use $\tau_t\leq\frac{1}{2}$ and $\mu_L\leq L_L$ in the second inequality, and use $\bar{\lambda}_i^t = \lambda_i^t + \tau_t(\tilde{\lambda}_i^t - \lambda_i^t)$ in the last inequality.

Notice that if $i\in\B_i^t$, then $\bar{\lambda}_i^t = \lambda_i^{t+1}$, so we have
\begin{equation*}
\begin{aligned}
	& \E\left[\left\|\lambda_i^{t+1}-\lambda_i(\w_t)\right\|^2 \right] = \frac{B_1}{n}\E\left[\left\|\bar{\lambda}_i^t - \lambda_i(\w_t) \right\|^2 \right] + \frac{n-B_1}{n}\E\left[\left\|\lambda_i^t - \lambda_i(\w_t) \right\|^2 \right] \\
	& \leq \left(1-\frac{\tau\tau_t\mu_L B_1}{2n}\right) \left\|\lambda_i(\w_t)-\lambda_i^t \right\|^2 + \frac{4\tau_t\tau B_1}{\mu_L n}\left\| \nabla_{\lambda}L_i(\w_t,\lambda_i^t)-z_i^t \right\|^2 - \frac{3\tau B_1}{2\tau_t n} \left(\frac{1}{\tau}-L_L\right)\left\|\lambda_i^{t+1} - \lambda_i^t \right\|^2.
\end{aligned}
\end{equation*}

Thus
\begin{equation*}
\begin{aligned}
	& \E\left[\left\|\lambda_i^{t+1}-\lambda_i(\w_{t+1})\right\|^2 \right] \\
	\leq & \left(1+\frac{\tau\tau_t\mu_L B_1}{4n}\right) \E\left[\left\|\lambda_i^{t+1}-\lambda_i(\w_t)\right\|^2 \right] + \left(1+\frac{4n}{\tau\tau_t\mu_L B_1}\right)\E\left[\left\|\lambda_i(\w_{t+1})-\lambda_i(\w_{t}) \right\|^2 \right] \\
	\leq & \left(1-\frac{\tau\tau_t\mu_L B_1}{4n}\right)\E\left[\left\| \lambda_i(\w_t)-\lambda_i^t\right\|^2 \right] + \frac{8\tau_t\tau B_1}{\mu_L n}\left\| \nabla_{\lambda}L_i(\w_t,\lambda_i^t)-z_i^t \right\|^2 \\
	&\quad - \frac{3\tau B_1}{\tau_t n} \left(\frac{1}{\tau}-L_L\right)\left\|\lambda_i^{t+1} - \lambda_i^t \right\|^2 + \frac{8n C_\lambda^2}{\tau\tau_t\mu_L B_1}\E\left[\left\| \w_{t+1}-\w_t \right\|^2 \right],
\end{aligned}
\end{equation*}
where we use $(1-\epsilon)(1+\frac{\epsilon}{2})\leq 1-\frac{\epsilon}{2}$ and the assumption $\tau_t\tau\leq\frac{4n}{\mu_L B_1}$ i.e., $\frac{\tau_t\tau\mu_L B_1}{4n}\leq 1$ in the last inequality.

Taking summation over all queries and expectation over all randomness, we have
\begin{equation*}
\begin{aligned}
	& \E\left[\left\|\lambda^{t+1}-\lambda(\w_{t+1})\right\|^2 \right] \\
	\leq & \left(1-\frac{\tau\tau_t\mu_L B_1}{4n}\right)\E\left[\left\| \lambda(\w_t)-\lambda^t\right\|^2 \right] + \frac{8\tau_t\tau B_1}{\mu_L n}\left\| \nabla_{\lambda}L(\w_t,\lambda^t)-z^t \right\|^2 \\
	&\quad - \frac{3\tau B_1}{\tau_t n} \left(\frac{1}{\tau}-L_L\right)\left\|\lambda^{t+1} - \lambda^t \right\|^2 + \frac{8n^2 C_\lambda^2}{\tau\tau_t\mu_L B_1}\E\left[\left\| \w_{t+1}-\w_t \right\|^2 \right].
\end{aligned}
\end{equation*}

\end{proof}

\subsubsection{Proof of Lemma~\ref{lemma_msvr}}

\begin{proof}

Denote $\bar{\d}_i^t= (1-\gamma_t)\d^{t-1}_{i} + \gamma_t f_i(\x_{t};\xi_t) + \beta_t (f_i(\x_{t};\xi_t)- f_i(\x_{t-1};\xi_t))$, then

\begin{equation}\label{eq:msvr_1}
\E\left[\left\|\d_i^t -f_i(\x_t) \right\|^2 \right]=\E\left[(1-\frac{B_1}{n})\left\| \d_i^{t-1}-f_i(\x_t) \right\|^2 + \frac{B_1}{n}\left\| \bar{\d}_i^t -f_i(\x_t) \right\|^2 \right].	
\end{equation}

First, we can decompose the first term on the RHS of (\ref{eq:msvr_1}) into 
\begin{equation}\label{eq:msvr_2}
\begin{aligned}
	&(1-\frac{B_1}{n})\E\left[\left\| \d_i^{t-1}-f_i(\x_t) \right\|^2 \right] = (1-\frac{B_1}{n})\E\left[\left\| \d_i^{t-1}- f_i(\x_{t-1}) + f_i(\x_{t-1})- f_i(\x_t) \right\|^2 \right] \\
	=&(1-\frac{B_1}{n})\E\left[\left\| \d_i^{t-1}- f_i(\x_{t-1}) \right\|^2\right] + (1-\frac{B_1}{n})\E\left[\left\| f_i(\x_{t-1}) -f_i(\x_t) \right\|^2\right] \\
	+& \underbrace{2(1-\frac{B_1}{n})\E\left[(\d_i^{t-1}- f_i(\x_{t-1}))(f_i(\x_{t-1}) -f_i(\x_t)) \right]}_{A}. 
\end{aligned}
\end{equation}

Then, we rewrite the the second term on the RHS of (\ref{eq:msvr_1}) into 
\begin{equation}\label{eq:msvr_3}
\begin{aligned}
	& \frac{B_1}{n}\E\left[ \left\| \bar{\d}_i^t - f_i(\x_t) \right\|^2 \right] \\
	= & \frac{B_1}{n} \E\left[\left\| (1-\gamma_t)(\d_i^{t-1} - f_i(\x_{t-1})) + (1-\gamma_t)(f_i(\x_{t-1}) - f_i(\x_t)) \right.\right. \\
	& \quad\quad\left.\left. + \gamma_t(f_i(\x_t;\xi_t) - f_i(\x_t)) + \beta_t(f_i(\x_t;\xi_t)-f_i(\x_{t-1};\xi_t)) \right\|^2 \right] \\
	= & \frac{B_1}{n} \E\left[\left\| (1-\gamma_t)(\d_i^{t-1} - f_i(\x_{t-1})) + (1-\gamma_t)(f_i(\x_{t-1}) - f_i(\x_t)) + \beta_t(f_i(\x_t;\xi_t)-f_i(\x_{t-1};\xi_t)) \right\|^2 \right] \\
	  &\ + \frac{B_1 \gamma_t^2}{n} \E\left[\left\|f_i(\x_t;\xi_t) - f_i(\x_t)\right\|^2 \right] \\
	  &\ + \frac{2 B_1 \gamma_t \beta_t}{n} \E\left[(f_i(\x_t;\xi_t) - f_i(\x_t))(f_i(\x_t;\xi_t)-f_i(\x_{t-1};\xi_t))\right] \\
	= &  \frac{B_1 (1-\gamma_t)^2}{n} \E\left[\left\| \d_i^{t-1} - f_i(\x_{t-1})\right\|^2\right]\\
	  &\ +  \frac{B_1}{n} \E\left[\left\| (1-\gamma_t)(f_i(\x_{t-1}) - f_i(\x_t)) + \beta_t(f_i(\x_t;\xi_t)-f_i(\x_{t-1};\xi_t)) \right\|^2 \right] \\
	  &\ + \underbrace{\frac{2 B_1}{n}(1-\gamma_t)(1-\gamma_t-\beta_t)\E\left[(\d_i^{t-1}- f_i(\x_{t-1}))(f_i(\x_{t-1}) -f_i(\x_t)) \right]}_{B} \\
	  &\ + \frac{B_1 \gamma_t^2}{n} \E\left[\left\|f_i(\x_t;\xi_t) - f_i(\x_t)\right\|^2 \right] \\
	  &\ + \frac{2 B_1 \gamma_t \beta_t}{n} \E\left[(f_i(\x_t;\xi_t) - f_i(\x_t))(f_i(\x_t;\xi_t)-f_i(\x_{t-1};\xi_t))\right].
\end{aligned}
\end{equation}

To make $A+B=0$, we require $2(1-\frac{B_1}{n}) + \frac{2 B_1}{n}(1-\gamma_t)(1-\gamma_t-\beta_t) = 0$, which gives us $\beta_t = 1-\gamma_t + \frac{n-B_1}{B_1 (1-\gamma_t)}$. Then, we plug (\ref{eq:msvr_2}) and (\ref{eq:msvr_3}) into (\ref{eq:msvr_1}) and have

\begin{equation*}
\begin{aligned}
	& \E\left[\left\|\d_i^t -f_i(\x_t) \right\|^2 \right]=\E\left[(1-\frac{B_1}{n})\left\| \d_i^{t-1}-f_i(\x_t) \right\|^2 + \frac{B_1}{n}\left\| \bar{\d}_i^t -f_i(\x_t) \right\|^2 \right] \\
	= & \left(1-\frac{B_1}{n} +  \frac{B_1 (1-\gamma_t)^2}{n}\right)\E\left[\left\| \d_i^{t-1}- f_i(\x_{t-1}) \right\|^2\right] + (1-\frac{B_1}{n})\E\left[\left\| f_i(\x_{t-1}) -f_i(\x_t) \right\|^2\right] \\
	  &\ +  \frac{B_1}{n} \E\left[\left\| (1-\gamma_t)(f_i(\x_{t-1}) - f_i(\x_t)) + \beta_t(f_i(\x_t;\xi_t)-f_i(\x_{t-1};\xi_t)) \right\|^2 \right] \\
	  &\ + \frac{B_1 \gamma_t^2}{n} \E\left[\left\|f_i(\x_t;\xi_t) - f_i(\x_t)\right\|^2 \right] \\
	  &\ + \frac{2 B_1 \gamma_t \beta_t}{n} \E\left[(f_i(\x_t;\xi_t) - f_i(\x_t))(f_i(\x_t;\xi_t)-f_i(\x_{t-1};\xi_t))\right] \\
	= & \left(1-\frac{B_1}{n} +  \frac{B_1 (1-\gamma_t)^2}{n}\right)\E\left[\left\| \d_i^{t-1}- f_i(\x_{t-1}) \right\|^2\right] + (1-\frac{B_1}{n})\E\left[\left\| f_i(\x_{t-1}) -f_i(\x_t) \right\|^2\right] \\
	  &\ +  \frac{B_1 (1-\gamma_t)^2}{n} \E\left[\left\|f_i(\x_{t-1}) - f_i(\x_t)\right\|^2 \right] + \frac{B_1 \beta_t^2}{n}\E\left[\left\|f_i(\x_t;\xi_t)-f_i(\x_{t-1};\xi_t) \right\|^2 \right]\\
	  &\ -  \frac{2 B_1 (1-\gamma_t)\beta_t}{n} \left\| f_i(\x_{t-1}) - f_i(\x_t) \right\|^2 \\
	  &\ + \frac{B_1 \gamma_t^2}{n} \E\left[\left\|f_i(\x_t;\xi_t) - f_i(\x_t)\right\|^2 \right] \\
	  &\ + \frac{2 B_1 \gamma_t \beta_t}{n} \E\left[(f_i(\x_t;\xi_t) - f_i(\x_t))(f_i(\x_t;\xi_t)-f_i(\x_{t-1};\xi_t))\right].
\end{aligned}	
\end{equation*}

Let $\gamma_t \leq \frac{1}{2}$, thus $\beta_t=1-\gamma_t+\frac{n-B_1}{B_1 (1-\gamma_t)}\leq 1-\gamma_t+\frac{2(n-B_1)}{B_1}=1-\gamma_t + \frac{2n}{B_1}-2 \leq \frac{2n}{B_1}$. Moreover, from $\beta_t=1-\gamma_t+\frac{n-B_1}{B_1 (1-\gamma_t)}$ we have $\frac{B_1 (1-\gamma_t)\beta_t}{n} = \frac{B_1 (1-\gamma_t)^2}{n}+1-\frac{B_1}{n}$, thus $1-\frac{B_1}{n} + \frac{B_1 (1-\gamma_t)^2}{n} -\frac{2 B_1 (1-\gamma_t)\beta_t}{n} \leq 0$. So we have

\begin{equation*}
\begin{aligned}
	& \E\left[\left\|\d_i^t -f_i(\x_t) \right\|^2 \right] \\
	\leq & \left(1-\frac{\gamma_t B_1}{n}\right) \E\left[\left\| \d_i^{t-1}- f_i(\x_{t-1}) \right\|^2\right] + \frac{4 n}{B_1}\E\left[\left\|f_i(\x_t;\xi_t)-f_i(\x_{t-1};\xi_t) \right\|^2 \right]+ \frac{B_1 \gamma_t^2}{n} \E\left[\left\|f_i(\x_t;\xi_t) - f_i(\x_t)\right\|^2 \right] \\
	  &\ + \frac{2 B_1 \gamma_t \beta_t}{n} \E\left[(f_i(\x_t;\xi_t) - f_i(\x_t))(f_i(\x_t;\xi_t)-f_i(\x_{t-1};\xi_t))\right] \\
	 \leq & \left(1-\frac{\gamma_t B_1}{n}\right)\E\left[\left\| \d_i^{t-1}- f_i(\x_{t-1}) \right\|^2\right] + \frac{4 n}{B_1}\E\left[\left\|f_i(\x_t;\xi_t)-f_i(\x_{t-1};\xi_t) \right\|^2 \right]+ \frac{B_1 \gamma_t^2}{n} \E\left[\left\|f_i(\x_t;\xi_t) - f_i(\x_t)\right\|^2 \right]\\
	 & + \frac{B_1 \gamma_t^2}{n} \E\left[\left\| f_i(\x_t;\xi_t) - f_i(\x_t)\right\|^2 \right] + 
	 \frac{B_1 \beta_t^2}{n} \E\left[\left\| f_i(\x_t;\xi_t)-f_i(\x_{t-1};\xi_t)\right\|^2 \right] \\
	 \leq & \left(1-\frac{\gamma_t B_1}{n}\right)\E\left[\left\| \d_i^{t-1}- f_i(\x_{t-1}) \right\|^2\right] + \frac{8 n}{B_1}\E\left[\left\|f_i(\x_t;\xi_t)-f_i(\x_{t-1};\xi_t) \right\|^2 \right] + \frac{2 B_1 \gamma_t^2 \sigma^2}{n}.
\end{aligned}	
\end{equation*}

Finally, we have

\begin{equation*}
\begin{aligned}
	& \E\left[\left\|\d^t -f(\x_t) \right\|^2 \right] =\sum_{i\in\S} \E\left[\left\|\d_i^t -f_i(\x_t) \right\|^2 \right] \\
	 \leq & \left(1-\frac{\gamma_t B_1}{n}\right)\E\left[\left\| \d^{t-1}- f(\x_{t-1}) \right\|^2\right] + \frac{8 n}{B_1}\sum_{i\in\S} \E\left[\left\|f_i(\x_t;\xi_t)-f_i(\x_{t-1};\xi_t) \right\|^2 \right] + 2 B_1 \gamma_t^2 \sigma^2.
\end{aligned}	
\end{equation*}

Then, we derive the bound for $\left\| \d^t - \d^{t-1} \right\|^2$

\begin{equation*}
\begin{aligned}
	&\E\left[\left\| \d^t - \d^{t-1} \right\|^2\right] = \sum_{i\in\S} \E\left[\left\| \d_i^t - \d_i^{t-1} \right\|^2\right]  \\
	= & \sum_{i\in\S} \E\left[\frac{B_1}{n} \left\| \bar{\d}_i^t - \d_i^{t-1} \right\|^2 + \frac{n-B_1}{n} \left\| \d_i^{t-1} - \d_i^{t-1} \right\|^2 \right] =\frac{B_1}{n} \sum_{i\in\S} \E\left[ \left\| \bar{\d}_i^t - \d_i^{t-1} \right\|^2 \right] \\
	= & \frac{B_1}{n}\sum_{i\in\S} \E\left[ \left\|\gamma_t (f_i(\x_t;\xi_t)-\d_i^{t-1}) + \beta_t(f_i(\x_t;\xi_t)-f_i(\x_{t-1};\xi_t)) \right\|^2  \right] \\
	\leq & \frac{B_1}{n} \sum_{i\in\S}\E \left[ 2\gamma_t^2 \left\|f_i(\x_t;\xi_t)-\d_i^{t-1} \right\|^2 + 2\beta_t^2 \left\|f_i(\x_t;\xi_t)-f_i(\x_{t-1};\xi_t) \right\|^2 \right] \\
	\leq & \frac{2 B_1 \gamma_t^2}{n} \sum_{i\in\S} \E\left[\left\| f_i(\x_t;\xi_t) -f_i(\x_t) \right\|^2 + \left\| f_i(\x_t) - \d_i^{t-1} \right\|^2 \right] \\
	& + \frac{8n}{B_1}\sum_{i\in\S}\E\left[ \left\|f_i(\x_t;\xi_t)-f_i(\x_{t-1};\xi_t) \right\|^2\right] \\
	\leq & 2 B_1 \gamma_t^2 \sigma^2 +  \frac{4 B_1 \gamma_t^2}{n} \sum_{i\in\S} \E\left[\left\| f_i(\x_t) - f_i(\x_{t-1}) \right\|^2 \right] + \frac{4 B_1 \gamma_t^2}{n}\sum_{i\in\S}\E\left[\left\| f_i(\x_{t-1}) - \d_i^{t-1} \right\|^2 \right] \\
	& + \frac{8n}{B_1}\sum_{i\in\S}\E\left[ \left\|f_i(\x_t;\xi_t)-f_i(\x_{t-1};\xi_t) \right\|^2\right] \\
	\leq & 2 B_1 \gamma_t^2 \sigma^2 + \frac{4 B_1 \gamma_t^2}{n}\E\left[\left\| f(\x_{t-1}) - \d^{t-1} \right\|^2 \right] + \frac{9n}{B_1}\sum_{i\in\S}\E\left[ \left\|f_i(\x_t;\xi_t)-f_i(\x_{t-1};\xi_t) \right\|^2\right],
\end{aligned}	
\end{equation*}
where the second inequality is due to $\E\left[f_i(\x_t;\xi_t)-f_i(\x_{t}) \right]=0$ and $\beta_t\leq\frac{2n}{B_1}$.

\end{proof}

\end{document}